\documentclass{sage}

\ifx\definitionMacros\undefined\else \fi
 \ifx\optionkeymacros\undefined\else \fi

\catcode`\Œ=\active\defŒ{{\aa}}       
\catcode`\º=\active\defº{\int}        
\catcode`\=\active\def{\c c}        
\catcode`\¶=\active\def¶{\partial}    
\catcode`\Ä=\active\defÄ{\oint}       
\catcode`\Æ=\active\defÆ{\triangle}   
\catcode`\Â=\active\defÂ{\neg}        
\catcode`\µ=\active\defµ{\mu}         
\catcode`\¿=\active\def¿{{\o}}        
\catcode`\¹=\active\def¹{\pi}         
\catcode`\Ï=\active\defÏ{{\oe}}       
\catcode`\§=\active\def§{{\ss}}       
\catcode`\ =\active\def {\dagger}     
\catcode`\Ã=\active\defÃ{\sqrt}       
\catcode`\·=\active\def·{\Sigma}      
\catcode`\Å=\active\defÅ{\approx}     
\catcode`\½=\active\def½{\Omega}      
\catcode`\£=\active\def£{{\it\$}}     
\catcode`\°=\active\def°{\infty}      
\catcode`\¤=\active\def¤{{\S}}        
\catcode`\¦=\active\def¦{{\P}}        
\catcode`\¥=\active\def¥{\bullet}     
\catcode`\»=\active\def»{\leavevmode\raise.585ex\hbox{\b a}}      
\catcode`\¼=\active\def¼{\leavevmode\raise.6ex\hbox{\b o}}        
\catcode`\­=\active\def­{\not=}       
\catcode`\²=\active\def²{\leq}        
\catcode`\³=\active\def³{\geq}        
\catcode`\Ö=\active\defÖ{\div}        
\catcode`\É=\active\defÉ{{\dots}}     
\catcode`\¾=\active\def¾{{\ae}}       
\catcode`\Ç=\active\defÇ{\ll}         
\catcode`\Ò=\active\defÒ{``}          
\catcode`\Á=\active\defÁ{!`}          
\catcode`\¢=\active\def¢{\rlap/c}     
\catcode`\Ô=\active\defÔ{`}           
\catcode`\Õ=\active\defÕ{'}           

\catcode`\=\active\def{{\AA}}       
\catcode`\'=\active\def'{\c C}        
\catcode`\¯=\active\def¯{{\O}}        
\catcode`\¸=\active\def¸{\Pi}         
\catcode`\Î=\active\defÎ{{\OE}}       
\catcode`\®=\active\def®{{\AE}}       
\catcode`\×=\active\def×{\diamond}    
\catcode`\¡=\active\def¡{\accent'27}  
\catcode`\Ó=\active\defÓ{''}          
\catcode`\±=\active\def±{\pm}         
\catcode`\È=\active\defÈ{\gg}         
\catcode`\À=\active\defÀ{?`}          
\catcode`\Ð=\active\defÐ{--}          
\catcode`\Ñ=\active\defÑ{---}         

\catcode`\Š=\active\defŠ{\"a}        
\catcode`\'=\active\def'{\"e}        
\catcode`\•=\active\def•{\"{\i}}     
\catcode`\š=\active\defš{\"o}        
\catcode`\Ÿ=\active\defŸ{\"u}        
\catcode`\Ø=\active\defØ{\"y}        
\catcode`\€=\active\def€{\"A}        
\catcode`\…=\active\def…{\"O}        
\catcode`\†=\active\def†{\"U}        
\catcode`\‡=\active\def‡{\'a}        
\catcode`\Ž=\active\defŽ{\'e}        
\catcode`\'=\active\def'{\'{\i}}     
\catcode`\—=\active\def—{\'o}        
\catcode`\œ=\active\defœ{\'u}        
\catcode`\ƒ=\active\defƒ{\'E}        
\catcode`\ˆ=\active\defˆ{\`a}        
\catcode`\=\active\def{\`e}        
\catcode`\"=\active\def"{\`{\i}}     
\catcode`\˜=\active\def˜{\`o}        
\catcode`\=\active\def{\`u}        
\catcode`\Ë=\active\defË{\`A}        
\catcode`\‹=\active\def‹{\~a}        
\catcode`\–=\active\def–{\~n}        
\catcode`\›=\active\def›{\~o}        
\catcode`\Ì=\active\defÌ{\~A}        
\catcode`\"=\active\def"{\~N}        
\catcode`\Í=\active\defÍ{\~O}        
\catcode`\‰=\active\def‰{\^a}        
\catcode`\=\active\def{\^e}        
\catcode`\"=\active\def"{\^{\i}}     
\catcode`\™=\active\def™{\^o}        
\catcode`\ž=\active\defž{\^u}        

\let\optionkeymacros\null

 \def\registered{{\ooalign                                          
   {\hfil\kern+.05em\raise.12ex                                     
 \hbox{\tiny R}\hfil\crcr{\footnotesize\mathhexbox20D}}}}           
 \newcommand{\TrueMath}    [1]{\mbox{$#1$}}                  
 \def\half{\TrueMath{\leavevmode\kern.1em \raise.5ex
                     \hbox{\the\scriptfont0 1}
                     \kern-.1em / \kern-.15em\lower.25ex
                     \hbox{\the\scriptfont0 2}
           }         }
 \newcommand{\onR}[1]{\in\mathbb{R}^{#1}}
 \def\cool#1{{\fontfamily{cmss}\selectfont #1}}
 \def\cools#1{{\emph{\fontfamily{cmss}\selectfont #1}}}
 \def\dots{\TrueMath{\ldots}}
\def\set12{\newfont{\size12}{cmbx12}}

\renewcommand{\frac}[2]{\TrueMath{\TrueMath{#1}\over\TrueMath{#2}}}
\newcommand{\define}   {\TrueMath{\,\buildrel \triangle \over = \,}}
\newcommand{\sign}      {\hbox{\rm sign}}
\def\lie#1#2{\TrueMath{{\cal L}_{\bf #1}(#2)}}               
\def\lies#1#2#3{\TrueMath{{\cal L}_{\bf #1}^{#3}(#2)}}       
\def\simless
{\mathop{\raise0.50ex\hbox{$<$}\kern-0.75em\lower0.50ex\hbox{$\sim$}}\nolimits}
\def\simgreat
{\mathop{\raise0.50ex\hbox{$>$}\kern-0.75em\lower0.50ex\hbox{$\sim$}}\nolimits}
\usepackage{mathtools}
\usepackage{amsmath}
\usepackage{amsthm}
\usepackage{amsfonts}
\usepackage{amsbsy}
\usepackage{amssymb}
\usepackage{latexsym}

\usepackage{wrapfig}
\usepackage{natbib,url}

\usepackage{multirow}
\usepackage{subfigure}
\usepackage{xfrac}
\usepackage{mathrsfs}
\usepackage{cases}
\usepackage{empheq}

\usepackage{algorithmic}
\usepackage{algorithm}
\usepackage{subeqnarray}

\usepackage{hyperref}
\usepackage{soul}
\usepackage{xcolor}

\usepackage{tikz}
\usepackage{blindtext}

\usepackage{comment}
\usepackage{minitoc}

\usepackage[font=scriptsize]{caption}

\newcommand*\reddot{{\protect \includegraphics[width=0.8em]{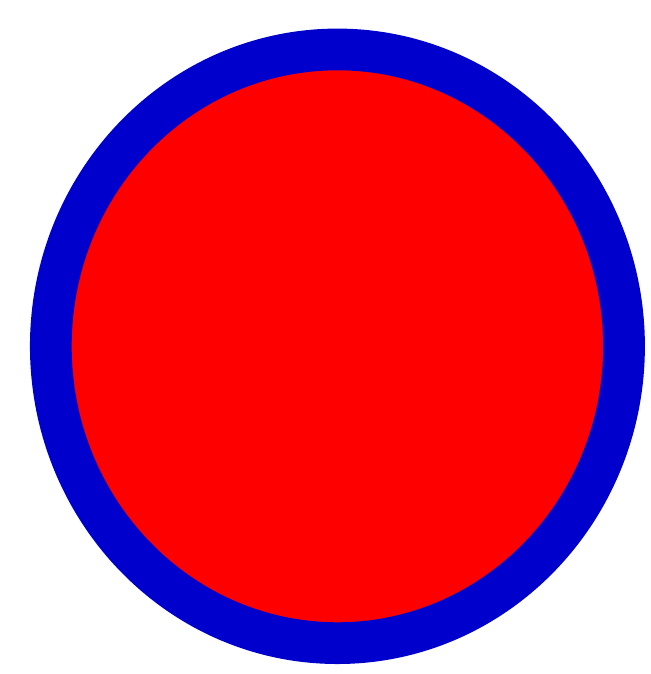}}}
\newcommand*\greendot{{\protect \includegraphics[width=0.8em]{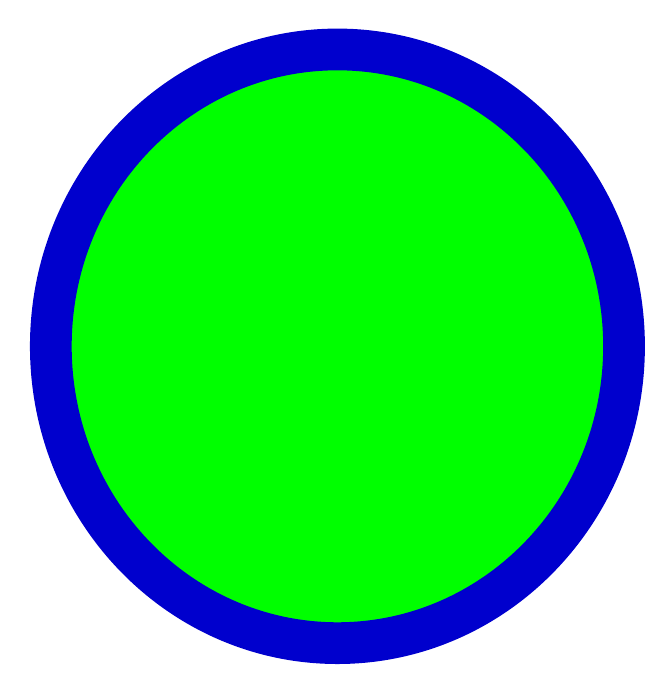}}}
\newcommand*\blueellipsoid{{\protect \includegraphics[width=1em]{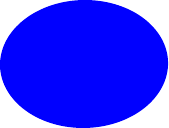}}}
\newcommand*\greenellipsoid{{\protect \includegraphics[width=1.4em]{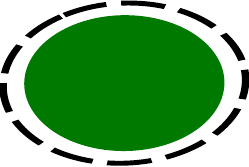}}}
\newcommand*\lightgreenellipsoid{{\protect \includegraphics[width=1em]{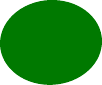}}}

\usepackage{color}

\newlength\myindent
\newtheorem*{theorem*}{Theorem}
\newtheorem*{proposition*}{Proposition}
\newtheorem{proposition}{Proposition}

\newtheorem{definition}{Definition}

\newtheorem*{example*}{Example}
\newtheorem*{property*}{Property}

{\color{blue}}

 \usepackage{tikz}

 \usepackage{marvosym}
 \usepackage{amssymb}
 \usepackage{pifont}
 \usepackage{xspace}

 \def\*tr{^{\,*T}}

 \renewcommand{\onR}[1]{\in\mathbb{R}^{#1}}

 \DeclareSymbolFont{rsfs}{U}{rsfs}{m}{n}
 \DeclareSymbolFontAlphabet{\mathrsfs}{rsfs}
 \def\Cont#1{\mathrsfs{C}^{#1}}

 \def\cool#1{{\fontfamily{cmss}\selectfont #1}}
 \def\cools#1{{\emph{\fontfamily{cmss}\selectfont #1}}}

 \newcounter{HCRLalgorno}

\newcommand\moveFigureUp[2]{%
   \raisebox{#1}{#2}}

\makeatletter
 \newcommand{\customlabel}[2]{%
   \protected@write \@auxout {}{\string \newlabel {#1}{{#2}{}}}}
\makeatother

\journal{The International Journal of Robotics Research}
\volume{XX}
\issue{XX}
\copyrightline{$\copyright$ Copyright}
\firstpage{1}
\lastpage{39}
\doi{doi number}
\articletype{Article Type}
\pubyear{2015}

\begin{document}


\title{Robust Optimal Planning and Control of Non-Periodic Bipedal Locomotion with A Centroidal Momentum Model}


\author{Ye Zhao$^1$, Benito R. Fernandez$^2$, and Luis Sentis$^{1}$\thanks{Corresponding author;
                 E-mail addresses: yezhao@utexas.edu,
                                 benito@austin.utexas.edu, 
                                 lsentis@austin.utexas.edu.}
}

\address{$^1$Human Centered Robotics Laboratory, The University of Texas at Austin, TX, USA.\\
$^2$Neuro-Engineering Research and Development Laboratory, The University of Texas at Austin, TX, USA.
}
 
\maketitle

\begin{abstract}
This study presents a theoretical method for planning and controlling agile bipedal locomotion based on robustly tracking a set of non-periodic keyframe states. Based on centroidal momentum dynamics, we formulate a hybrid phase-space planning and control method which includes the following key components: (i) a step transition solver that enables dynamically tracking non-periodic keyframe states over various types of terrains, (ii) a robust hybrid automaton to effectively formulate planning and control algorithms, (iii) a steering direction model to control the robot's heading, (iv) a phase-space metric to measure distance to the planned locomotion manifolds, and (v) a hybrid control method based on the previous distance metric to produce robust dynamic locomotion under external disturbances. Compared to other locomotion methodologies, we have a large focus on non-periodic gait generation and robustness metrics to deal with disturbances. Such focus enables the proposed control method to robustly track non-periodic keyframe states over various challenging terrains and under external disturbances as illustrated through several simulations.
\end{abstract}

\keywords{Phase-Space Locomotion Planning, Non-Periodic Keyframe Mapping, Robust Hybrid Automaton, Optimal Control.}

\section{Introduction}
\label{sec:intro}

Humanoid and legged robots may soon nimbly and robustly maneuver over highly rough terrains and unstructured environments. 
This study formulates a new method for the generation of trajectories and an optimal controller to achieve locomotion in those types of environments using a phase-space formalism.  Using prismatic inverted pendulum dynamics and given a set of desired keyframe states, we present a phase-space planner that can precisely negotiate the challenging terrains.  The resulting trajectories are formulated as phase-space manifolds.  Borrowing from sliding mode control theory, we use the newly defined manifolds and a Riemannian distance metric to measure deviations due to external disturbances or model uncertainties.  A control strategy based on dynamic programming is proposed that steers the locomotion process towards the planned trajectories. Finally, we devise a robust hybrid automaton to effectively formulate control algorithms that involve both continuous and discrete input processes for disturbance recovery. We validate this planning methodology via various simulations  including: dynamically walking over a random rough terrain, walking under external disturbances, walking while changing the robot's heading and dynamically leaping over a disjointed terrain.
\subsection{Dynamic Legged Locomotion}
Dynamic legged locomotion has been a center of attention for the past few decades [\cite{grizzle2014models, pratt2001virtual, hubicki2016atrias, park2017high, hutter2014quadrupedal, erez2007bipedal, wu2010terrain, zhao2014human}]. The work in [\cite{raibert1986legged}] pioneered robust hopping locomotion of point-foot monoped and bipedal robots using simple dynamical models but with limited applicability to semi-periodic hopping motions. His focus is on dynamically stabilizing legged robots. Instead, our focus is on precisely tracking keyframe states, i.e. a discrete set of desired robot center-of-mass (CoM) positions and velocities along the locomotion paths. Such capability is geared towards the design of highly non-periodic gaits in unstructured environments or the characterization of dynamic gait structure in a generic sense. [\cite{pratt2001virtual}] achieved point-foot biped walking using a virtual model control method but with limited applicability to mechanically supported robots. Unsupported point-foot biped locomotion in moderately rough terrains has been recently achieved by [\cite{grizzle2014models}] and [\cite{ramezani2014performance}] using Poincar\'e stability methods. However, Poincar\'e maps cannot be leveraged to achieving non-periodic gaits for highly irregular terrains. The work [\mbox{\cite{yang2009framework}}] devised switching controllers for aperiodic walking via re-defining the notation of walking stability. In contrast, our work focuses on non-periodic gaits for unsupported robots in highly irregular and disjointed terrains.

Our method compares well with similar reduced-order approaches. [\cite{wu20133}] proposes a time-based deadbeat controller for highly robust hoping behaviors based on the SLIP model over uncertain terrains. This work achieves robustness via feedforward control instead of correcting for past disturbances. In our case we quantify robustness as the distance between the planned and the actual disturbed trajectories. And we use such distance as part of the cost function for control. Using a SLIP model for humanoid robots quickly turning is studied in [\cite{wensing20143d}] based on a steering optimization. However, this work does not address rough terrains and robustness quantification is still missing. The authors in [\cite{piovan2015reachability}] propose a SLIP model for energy-varying planning on rough terrains and devised reachability metrics via numerical analysis. A numerical algorithm is proposed to achieve desired apex states, which is analogous to the principle of our keyframe-based planning. However, their method focuses on 2D locomotion patterns.

The Capture Point method [\cite{pratt2006capture}] provides one of the most practical frameworks for locomotion. Sharing similar core ideas, the divergent component of motion \mbox{[\cite{takenaka2009real}]} and the extrapolated center-of-mass \mbox{[\cite{hof2008extrapolated}]} were independently proposed. Extensions of the Capture Point method [\cite{englsberger2015three, morisawa2012balance}], allow locomotion over rough terrains. Recently, the work in [\cite{ramos2015generalizations}] generalizes the Capture Point method by proposing a ``Nonlinear Inverted Pendulum'' model, but it is limited to the two-dimensional case, and angular momentum control is ignored. Motion planning techniques based on interpolation through kinematic configurations have been explored, among other works, by [\cite{hauser2014fast}] and [\cite{pham2013kinodynamic}]. Those techniques are making great progress towards mobility and locomotion in various kinds of environments. The main difference from these studies is that our controller provides a robust optimal recovery strategy and ensures stability to achieve under-actuated dynamic walking. 

Another close work on agile locomotion is [\mbox{\cite{mordatch2010robust}}] which proposes a physics-based locomotion controller and devises an online motion planner to generate various types of robust gaits over rough terrains. Recently, progress on this line of work enables the generation of non-periodic locomotion trajectories [\cite{mordatch2012discovery}]. One key missing aspect of these works is quantifying robustness and analyzing feedback stability. Additionally, these works does not address locomotion of point-foot robots.

\subsection{Optimal Control and Planning}
Optimal control of legged locomotion over rough terrains are explored in [\cite{kuindersma2016optimization, dai2012optimizing, feng2015optimization, Byl:09, carpentier2016versatile}]. The work in [\cite{manchester2011stable}] proposed a control technique to stabilize non-periodic motions of under-actuated robots with a focus on walking over uneven terrain. The control is achieved by constructing a lower-dimensional system of coordinates transverse to the target cycle and then computing a receding-horizon optimal controller to exponentially stabilize the linearized dynamics. Recently, progress on this line of work enables the generation of non-periodic locomotion trajectories [\cite{manchester2014real}]. In contrast with these works, we focus on robustness by providing a distance metric for recovery and an optimal control approach. In [\cite{saglam2014robust}], a controller switching strategy for walking on irregular terrains is proposed. They optimize policies for switching between a set of known controllers. Their method is further extended to incorporate noise on the terrain and through a value iteration process they achieve a certain degree of robustness through switching. Instead, our paper is focused on creating new optimal controllers from scratch for general types of terrains. Additionally, their work is focused on 2D locomotion whereas we focus on 3D.

Model predictive control is explored in [\cite{tassa2012synthesis, audren2014model, nguyen2017walking, faraji2014robust, van2017real, brasseur2015robust}] for complex humanoid behaviors. [\cite{stephens2010push}] uses model predictive control (MPC) for push recovery by planning future steps. The authors in [\cite{caron2016multi}] propose a preview control method of 3D center-of-mass accelerations for multi-contact rough terrain locomotion. To make the problem tractable, polyhedral bounds are used to decouple quadratic inequalities into a set of linear constraints. [\cite{wieber2006trajectory}] presents a linear MPC scheme for zero moment point control with perturbations. However, many MPC methods rely on linearizing system dynamics at each time step, and only local optimality for each step is guaranteed. Our robust control strategy uses a dynamic programming approach to generate a policy table offline and then execute it in an online pattern. This can be treated as an explicit MPC approach with a few walking step horizon.

\subsection{Robustness and Recovery Strategies}
Numerous studies have focused on recovery strategies upon disturbances [\cite{hofmann2006robust, posa2017balancing, li2015fall, zhao2013biologically}]. Various recovery methods have been proposed based on ankle, hip, knee, and stepping strategies [\cite{kuo1992human, stephens2010push}]. In [\cite{hyon2007disturbance}], a stepping controller based on ground contact forces is implemented in a humanoid robot. The study in [\cite{komura2005feedback}] controls hip angular momentum to achieve planar bipedal locomotion. In our study, we simultaneously control the rate of change of the torso angular momentum, the center-of-mass apex height and the foot placements to achieve unsupported rough terrain walking. 

In [\cite{hobbelen2007disturbance}], a gait sensitivity norm is presented to measure disturbance rejection during dynamic walking. In [\cite{hamed2016exponentially}], sensitivity analysis with respect to ground height variations is performed to model robustness of orbits. These techniques are limited to cyclic walking gaits. The work in [\cite{arslan2012reactive}] unifies planning and control to provide robustness. However, the technique is only applied to planar hopping robots.

Exact knowledge of the terrain profile is normally impractical due to the inaccurate sensing processes and ubiquitous noise. Many works in locomotion assume perfect terrain sensing [\cite{liu2015trajectory, feng2015optimization}]. The work of [\cite{Byl:09}] used mean first-passage time to quantify the robustness to unknown terrains, whose height follows a modeled probabilistic distribution.  Recently, [\cite{dai2012optimizing, griffin2016nonholonomic}] proposed robust optimization approaches with augmented cost function penalizing state and control deviations arising from unknown terrain heights. [\cite{park2013finite}] devised finite-state-machine-based controllers for unexpected terrain height variations and implemented them in a planar robot. Although our study in this paper does not explicitly model terrain uncertainties, our proposed robustness metric and recovery strategies could be applicable to deal with unknown terrains. For instance, it is plausible to analyze the effect of terrain uncertainties in terms of the disturbance categories defined in Section~\ref{sec:optimization}. As a result, the robust control strategies developed in Section~\ref{sec:simulations} could be applied for recovery.

In [\cite{frazzoli2001robust}], a robust hybrid automaton is introduced to achieve time-optimal motion planning of a helicopter in an environment with obstacles. The same group studies robustness to model uncertainties [\cite{schouwenaars2003robust}] but ignores external disturbances. More recently, [\cite{majumdar2013robust}] accounts for external disturbances like cross-wind, by computing funnels via Lyapunov functions and switching between these funnels for maneuvering unmanned air vehicles in the presence of obstacles and disturbances. We apply some of these concepts to point-foot locomotion. Our dynamic system is hybrid, i.e., possessing a different set of dynamic equations for each contact stage.  As a result, we propose a hybrid control algorithm that switches states when the physical system changes the number of contacts.  We use the hybrid automaton as a tool for planning and control of bipedal locomotion.  We in fact extend their use of hybrid automaton to accommodate for hybrid systems. Additionally we re-generate phase-space trajectories on demand while the previous works rely on pre-generated primitives. 
\subsection{Contributions and Organization}
In light of the discussions above, our contributions are the following: (i) we formulate a hybrid automaton to characterize non-periodic locomotion dynamics, (ii) using the automaton, we synthesize motion plans in the phase-space to maneuver over irregular terrains while tracking a set of desired keyframes composed of CoM apex states and heading direction angles, (iii) a phase-space manifold is created with a Riemannian distance metric to measure nominal trajectory deviations and design an in-step controller, and (iv) we derive an optimal control method to recover from disturbances and uncertainties, and propose a theorem for its attractiveness. Overall, the key difference compared with previous works is our focus on trajectory generation and robust control of non-periodic and hybrid gaits. We are less centered on dynamic balance or moving from an initial to a final location but instead on tracking desired keyframes. A preliminary version of this work is presented in our conference paper [\cite{zhao2016robustplanning}].

This paper is outlined as follows. Section~\ref{sec:ProblemDefinition} introduces the control formalism and presents preliminary notations. In Section~\ref{sec:model}, we present the proposed centroidal-momentum-based locomotion model. An algorithm is devised to produce nominal phase-space trajectories. Section~\ref{sec:3d-moplan} introduces key planner components including a robust hybrid automaton, a step transition solver, and a steering direction model.  In Section~\ref{sec:manifold}, we devise analytical solutions for phase-space tangent and cotangent manifolds. Additionally, we classify disturbance patterns, guards and recovery strategies in the phase-space. Section~\ref{sec:optimization} formulates a two-stage control procedure for disturbance rejection. We propose a theorem for the existence and estimation of a recoverability bundle. Simulation results of locomotion over rough terrains and under disturbances are shown in Section~\ref{sec:simulations}. In Section~\ref{sec:Discussion}, we discuss the results, make conclusions, and motivate future work. The Appendix presents mathematical notations and proofs.

\section{Problem Definition}
\label{sec:ProblemDefinition}
In this section, we present basic control formalism and manifold analysis that will allow us to characterize, plan and control non-periodic locomotion processes in later sections.
\subsection{System Equations}\label{sec:SystemEquations}
Legged robots can be characterized as Multi-Input/Multi-Output (MIMO) systems.  Let us assume that a bipedal robot can be characterized by $n_j$ joint degrees of freedom (DOF), $\boldsymbol{q}=[q_1,q_2,\dots,q_{n_j}]^T\onR{n_j}$. Letting $\boldsymbol{x}(t)=[\boldsymbol{q}^T(t),\dot{\boldsymbol{q}}^T(t))]^T\in\mathbb{R}^{n}$, be the state-space vector ($n=2n_j$), $\boldsymbol{u}(t)\in\mathbb{R}^{m}$, represents the control input vector (generalized torques and forces), and defining $\boldsymbol{f}(\boldsymbol{x}(t))$, $\boldsymbol{g}(\boldsymbol{x}(t))$, and $\boldsymbol{h}(\boldsymbol{x}(t))$ in the obvious manner, the mechanical model is expressed in state variable form as
\begin{subeqnarray}\label{eqs:plant} 
\dot{\boldsymbol{x}}(t) &=&  \boldsymbol{f} (\boldsymbol{x}(t)) 
                    +   \boldsymbol{g} (\boldsymbol{x}(t)) \boldsymbol{u}(t)
                    +   \boldsymbol{J}_d (\boldsymbol{x}(t)) \boldsymbol{d}(t),  \label{eq:plantStates}\\
   \boldsymbol{y}(t)   &=&  \boldsymbol{h} (\boldsymbol{x}(t)),                  \label{eq:plantOutputs}
\end{subeqnarray}
where $\boldsymbol{d}(t)$ represent the generalized external disturbance forces, and $\boldsymbol{J}_d(\boldsymbol{x}(t))$ is the disturbance distribution matrix.  
The output vector $\boldsymbol{y}(t) = [y_1,y_2,\dots,y_p]^T\in \mathbb{R}^{p}$ is generated by $\boldsymbol{h}(\boldsymbol{x}(t))$, that may represent positions and/or velocities in the task space. Without loss of generality, let us consider systems in the normal form, where $\boldsymbol{h}(\cdot)$ is at least $\Cont{r}$, where $r$ is the relative order of the output. The disturbances and modeling errors satisfy the matching conditions [\cite{Fernandez:PhD}]. 
\subsection{System Normalization for Phase-Space Planner Design}
\label{sec:Normalization}
General robotic systems are not in normal form, but we can transform them by finding what relative order of the output derivatives are explicitly controllable.  Each of the outputs $y_i$ in Eq.~(\ref{eq:plantOutputs}) has a relative order $r_i$, defined by the smallest derivative order where the control appears,
 \begin{subeqnarray}
     y_i^{[k]}   &=& \dfrac{d^ky_i}{dt^k} = \lies{\boldsymbol{f}}{h_i(\boldsymbol{x})}{k} 
                +  \lie{\boldsymbol{g}}{\lies{\boldsymbol{f}}{h_i(\boldsymbol{x})}{k-1}}\boldsymbol{u},\\
     \lie{\boldsymbol{g}}{\lies{\boldsymbol{f}}{h_i(\boldsymbol{x})}{k-1}}   &  =  & 0 \quad ~{\rm for} \quad 0\le k<r_i
     \label{eq:RO},\\
     y_i^{[r_i]} &=& \lies{\boldsymbol{f}}{h_i(\boldsymbol{x})}{r_i} 
                +  \lie{\boldsymbol{g}}{\lies{\boldsymbol{f}}{h_i(\boldsymbol{x})}{r_i-1}}\boldsymbol{u},\\
     \lie{\boldsymbol{g}}{\lies{\boldsymbol{f}}{h_i(\boldsymbol{x})}{r_i-1}} & \ne & 0 \quad ~
          \quad \forall \boldsymbol{x} \in \mathbb{S}_i \subset \mathbb{R}^n,
     \label{eq:RO_2}
  \end{subeqnarray}
where $\lies{\boldsymbol{f}}{h_i(\boldsymbol{x})}{0}=h_i(\boldsymbol{x})$, $\lie{\boldsymbol{f}}{\boldsymbol{h}}$ and $\lie{\boldsymbol{g}}{\boldsymbol{h}}$ are the directional Lie derivatives of function $\boldsymbol{h}(\boldsymbol{x})$ in the directions of $\boldsymbol{f}(\boldsymbol{x})$ and $\boldsymbol{g}(\boldsymbol{x})$ respectively [\cite{Isidori:Book}], and $\mathbb{S}_i$ is the output-controllable subspace, where the Lie derivative in Eq.~(\ref{eq:RO_2}) does not vanish,
\begin{equation}
\mathbb{S}_i~=~\Bigl\{~\boldsymbol{x} \onR{n}~~\left|~~\lie{\boldsymbol{g}}{\lies{\boldsymbol{f}}{h_i(\boldsymbol{x})}{r_i-1}}~\ne~0\right.\Bigr\}.
\label{uncontrollableSpace}
\end{equation}
The relative order tells us that the $r_i^{\rm th}$-derivative of output $y_i$ can be explicitly controlled. The region where $\mathbb{S}_i$ vanishes, entails the system loses relative order and hence the $r_i^{\rm th}$-derivative is no longer controllable (at least explicitly). For a controllable system, $r_i \le n$. Following the normalization  procedure, we get the output controllable subspace,
\begin{subeqnarray}\label{eq:xi}
  \xi_{i,1} &=& y_i = h_i(\boldsymbol{x}) = \lies{\boldsymbol{f}}{h_i(\boldsymbol{x})}{0}, \\[-2.5mm]
              &\dots& \nonumber\\[-2.5mm]
  \xi_{i,j} &=& y_i^{[j-1]} = \dot {\xi}_{i,j-1} = \lies{\boldsymbol{f}}{h_i(\boldsymbol{x})}{j-1} 
                              \quad ~{\rm for} \quad 1<j<r_i,\\[-2.5mm]
              &\dots& \nonumber\\[-2.5mm]
            &~& y_i^{[r_i]}   = \dot {\xi}_{i,r_i}   = \lies{\boldsymbol{f}}{h_i(\boldsymbol{x})}{r_i} 
                              + \lie{\boldsymbol{g}}{\lies{\boldsymbol{f}}{h_i(\boldsymbol{x})}{r_i-1}}\boldsymbol{u}(t).
\end{subeqnarray}
The output space variables $\boldsymbol{\xi}_i = [\xi_{i,1}, \xi_{i,2}, \dots, \xi_{i,r_i-1}]^T\onR{r_i}$ represent the phase-space for the $i$-th output. For instance, the output phase-space for locomotion control could be chosen to be the robot's center-of-mass. We can concatenate all $\boldsymbol{\xi}_i$, $\forall i=1,2,\dots,m$ into a single phase-space vector $\boldsymbol{\xi} = [\boldsymbol{\xi}_1^T, \boldsymbol{\xi}_2^T, \dots, \boldsymbol{\xi}_m^T]^T\onR{r}$, where $r=\sum{r_i}$.  
For phase-space motion, we define a phase-space manifold $\mathcal{M}_i$ for each task-space output $y_i$ in terms of its phase-space vector $\boldsymbol{\xi}_i$,
\begin{equation}
 \mathcal{M}_i~=~\Bigl\{~\boldsymbol{\xi}_i \in \mathbb{R}^{r_i}\subset\mathbb{R}^n~~\left|~~\sigma_i
             ~\define~\sigma_i(\boldsymbol{\xi}_i)~=~ 0\right.\Bigr\}, \label{eq:surface}
\end{equation}
where $\sigma_i$ is referred to as the $i^{\rm th}$ element of \cools{deviation vector}, which measures the deviation distance from the manifold $\mathcal{M}_i$ using a Riemannian metric. More details are shown in Appendix~\ref{sec:PSMBasics}. In order to be able to control this deviation, the order of the manifold is one less than the relative order of the $i^{\rm th}$-output, i.e., $r_i-1$. For most legged robots (not considering actuator dynamics), the relative order is $r=2$.
\section{Prismatic Inverted Pendulum Dynamics on a Parametric Surface}
\label{sec:model}
\begin{figure*}
 \centering
 \includegraphics[width=0.95\linewidth]{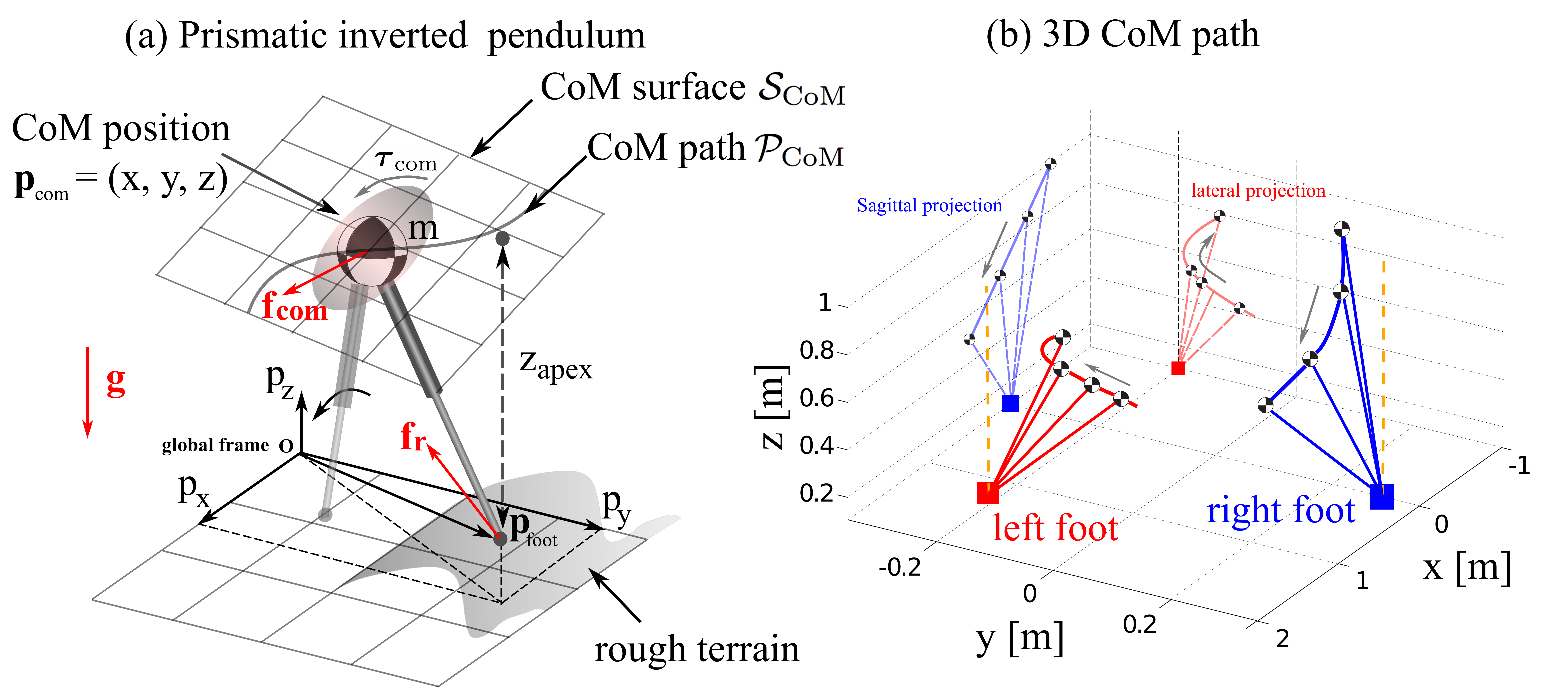}
 \caption{\captionsize 3D prismatic inverted pendulum model. (a) We define a prismatic inverted pendulum model with all of its mass located at its base while equipping it with a flywheel to generate moments. We restrict the movement of the center-of-mass to 3D planes (surfaces) $\mathcal{S}_{\rm CoM}$. Three red arrows represent the CoM inertial force $\boldsymbol{f}_{\rm com}$, the ground reaction force $\boldsymbol{f}_r$ and the gravity force $m \boldsymbol{g}$, respectively. These forces satisfy $\boldsymbol{f}_{\rm com} = \boldsymbol{f}_r - m \boldsymbol{g}$. (b) shows motions of pendulum dynamics restricted to a 3D plane. Note that, our study assumes time-varying leg length. The apex height $z_{\rm apex}$ is $1$ m.}
 \label{fig:Inverted_Pendulum}
\end{figure*}
The rigid body dynamics of point-foot bipedal robots during single contact resemble those of a simple inverted pendulum model (see Fig.~\ref{fig:Inverted_Pendulum}), as observed by studies in dynamic human walking [\cite{kuo2002energetics, matthis2013humans}].  In our case, our model consists of a prismatic massless joint with all the mass concentrated on the hip position [\cite{kajita2003biped, koolen2012capturability}], defined as the 3D CoM position $\boldsymbol{p}_{\rm com}=(x, y, z)^T$, and a flywheel spinning around it, with orientation angles $\boldsymbol{R} = (\phi, \theta, \psi)^T$. Various human walking [\cite{kuo1992human}] and balancing [\cite{winter1995human}] studies emphasize that controlling the centroidal angular momentum can improve CoM tracking, balancing and recovering from disturbances. This rule of thumb has been recently adopted in dynamic robot locomotion [\mbox{\cite{pratt2006capture, komura2005feedback, yun2011momentum}}]. The objective of locomotion is to move the robot's CoM along a certain path from point A to B over a terrain. As such, we first specify a 3D surface $\mathcal{S}_{\rm CoM}$, where the CoM path will evolve, which in general, may have the following implicit form,
\begin{align}\label{eq:genericSurface}
 \mathcal{S}_{\rm CoM}~=~\Bigl\{\boldsymbol{p}_{\rm com} \onR{3}
            ~~\left|~~\psi_{\rm CoM}(\boldsymbol{p}_{\rm com})~=~ 0\right.\Bigr\}.
\end{align}
This surface can be specified in various ways, such as via piecewise arc geometries [\cite{mordatch2010robust, srinivasan2006computer}] and spline functions [\cite{morisawa2005pattern, englsberger2015biologically}]. Once the controller is designed, the CoM will follow a concrete trajectory $\mathcal{P}_{\rm CoM}$ (as shown in Fig.~\ref{fig:Inverted_Pendulum}), which we specify via piecewise splines described by a progression variable $\zeta \in [\zeta_{j-1}, \zeta_j]$, for the $j^{\rm th}$ path manifold, i.e.
\begin{eqnarray}\label{eq:splineManifold}
 \mathcal{P}_{\rm CoM} &= \bigcup_{j}^{} \mathcal{P}_{{\rm CoM}_j} 
   \subseteq \mathcal{S}_{\rm CoM},\hspace{0.2in} 
   \mathcal{P}_{{\rm CoM}_j} &= 
   \Bigl\{~\boldsymbol{p}_{{\rm com}_j} \in \mathbb{R}^3
         ~~\left|~~\boldsymbol{p}_{{\rm com}_j} = \sum_{k = 0}^{n_p} \boldsymbol{a}_{jk} \zeta^k\right.\Bigr\},
\end{eqnarray}
where $n_p$ is the degree of the spline.  The progression variable $\zeta$ is therefore the arc length along the CoM path acting as the Riemannian metric for distance. Each $\boldsymbol{a}_{jk} \in \mathbb{R}^3$ is the coefficient vector of $k^{\rm th}$ order. To guarantee spline smoothness, $\boldsymbol{p}_{\rm com}$ requires the connection points, i.e. the knots at progression instant $\zeta_j$, to be $\Cont{n_p-1}$ continuous, 
\begin{align}
\boldsymbol{p}^{[l]}_{{\rm com}_j}(\zeta_j) 
  = \dfrac{d^l\boldsymbol{p}_{{\rm com}_j}}{d\zeta^l}(\zeta_j) 
  = \boldsymbol{p}^{[l]}_{{\rm com}_{j+1}}(\zeta_j), 
  \quad \forall \; 0 \leq l \leq n_p-1.
\end{align}
The purpose of introducing the CoM manifold $\mathcal{S}_{\rm CoM}$ is to constrain CoM motions on surfaces that are designed to conform to generic terrains while allowing free motion within this surface. Following a concrete CoM path is achieved by selecting proper control inputs as we will see further down. The CoM path manifold $\mathcal{P}_{\rm CoM}$ (embedded in $\mathcal{S}_{\rm CoM}$), can be represented in the phase-space $\boldsymbol{\xi}$.  We call this representation as the \emph{phase-space manifold} and define it as,
\begin{align}\label{eq:Mcom}
 \mathcal{M}_{\rm CoM} = \bigcup_{j}^{} \mathcal{M}_{{\rm CoM}_j}, 
   \qquad 
    \mathcal{M}_{{\rm CoM}_j} = \Bigl\{~\boldsymbol{\xi} \in \mathbb{R}^6
            ~~\left|~~\sigma_j(\boldsymbol{\xi})~=~ 0\right.\Bigr\},
\end{align}
which is the main manifold used in our planning and control methods.  The function $\sigma_j(\boldsymbol{\xi})$ is an implicit function in the phase-space measuring the distance to the manifold.

The centroidal momentum dynamics can be characterized via formulating the dynamic balance of moments around the system's centroidal point.
\begin{align}\label{eq:linearmomentum}
\boldsymbol{\dot{l}} = m \boldsymbol{\ddot{p}}_{\rm com} = \sum_i^{N_c} \boldsymbol{f}_{r_i} - m \boldsymbol{g}, \quad \boldsymbol{\dot{k}} = \boldsymbol{\tau}_{\rm com} = \sum_i^{N_c} (\boldsymbol{p}_{{\rm foot}_i} - \boldsymbol{p}_{\rm com}) \times \boldsymbol{f}_{r_i} +  \boldsymbol{\tau}_i,
\end{align}
where $\boldsymbol{l} \in \mathbb{R}^3$ and $\boldsymbol{k} \in \mathbb{R}^3$ represent the centroidal linear and angular momenta, respectively. $\boldsymbol{f}_{r_i} \in \mathbb{R}^3$ is the $i^{\rm th}$ ground reaction force, $m$ is the total mass of the robot, $\boldsymbol{g} =(0, 0, g)^T$ corresponds to the gravity field, $\boldsymbol{f}_{\rm com} = m \boldsymbol{\ddot{p}}_{\rm com} = m (\ddot{x}, \ddot{y}, \ddot{z})^T$ is the vector of center-of-mass inertial forces. The first equation above represents the rate of change of linear momentum being equal to the total action of linear contact forces minus gravitational forces. $\boldsymbol{\tau}_{\rm com} = (\tau_x, \tau_y, \tau_z)^T$ is the vector of angular moments of the modeled flywheel attached to the inverted pendulum. $\boldsymbol{p}_{{\rm foot}_i} = (p_{i,x}, p_{i,y}, p_{i,z})^T$ is the position of the $i^{\rm th}$ foot contact contact. $\boldsymbol{\tau}_{r_i} \in \mathbb{R}^3$ is the $i^{\rm th}$ contact torque vector. The second equation above represents the rate of change of angular momentum being equal to the sum of the torques generated by total action of contact wrenches projected to the CoM. 
In our case, $\boldsymbol{\tau}_{r_i} = \boldsymbol{0}$ due to having point-foot contacts.
\subsection{Single Contact Dynamics}
\label{subsec:singlecontact}
For our single contact scenario, the sum of moments, with respect to the global reference frame (see Fig.~\ref{fig:Inverted_Pendulum}) is 
\begin{equation}\label{eq:balance-2}
-\boldsymbol{p}_{\rm foot} \times \boldsymbol{f}_{r} + \boldsymbol{p}_{\rm com} \times \Big( \boldsymbol{f}_{\rm com} + m \, \boldsymbol{g} \Big) +
\boldsymbol{\tau}_{\rm com} = 0,
\end{equation}
The system's linear force equilibrium can be formulated as $\boldsymbol{f}_{r} = \boldsymbol{f}_{\rm com} + m \, \boldsymbol{g}$, allowing us to simplify Eq.~(\ref{eq:balance-2}) to
\begin{equation}\label{eq:balance-3}
\Big(\boldsymbol{p}_{\rm com} - \boldsymbol{p}_{\rm foot} \Big) \times (\boldsymbol{f}_{\rm com} + m \, \boldsymbol{g}) = -\boldsymbol{\tau}_{\rm com}.
\end{equation}
For our purposes, we consider only the class of prismatic inverted pendulums whose center-of-mass is restricted to a path surface $\mathcal{S}_{\rm CoM}$ as indicated in Eq.~(\ref{eq:genericSurface}). In our previous work [\cite{zhao2012three}], we had assumed that the CoM height is invariant to lateral coordinate changes. To remove this restriction, we model a ``true'' 3D plane (i.e., both sagittal and lateral variations of the CoM height are permitted).
%
\noindent A detailed definition of this 3D plane will be presented in Eq. (\ref{eq:linear2DSurface}). This type of model with varying height is called the Prismatic Inverted Pendulum Model (PIPM) [\cite{zhao2012three}].

Previously we had observed that the CoM behavior during human walking approximately follows the slope of terrains [\cite{zhao2012three, zhao2016humanwalking}]. Based on this observation, we design piecewise CoM planes approximating terrain slopes and adjust the CoM planes according to the acceleration or deceleration phases.

A variety of CoM trajectory design methods have been proposed over the years. The Capture Point method in [\cite{koolen2012capturability}] assumes a constant CoM height. Closely related to us, [\cite{kajita2003biped}] constrains the CoM motion to a 3D plane. However, our focus is on robust hybrid control. Designing CoM trajectories with a varying CoM height are described in [\cite{englsberger2015three, koolen2016balance}]. The work described in [\cite{ramos2015generalizations}] proposed a Nonlinear Inverted Pendulum model and the CoM path is extended to a parabola, but it focuses on planar locomotion.

Considering as our output state the CoM positions $\boldsymbol{p}_{\rm com}$, the state space $\boldsymbol{\xi}= (\boldsymbol{p}_{\rm com}^T, \boldsymbol{\dot{p}}_{\rm com}^T)^T = (x,y,z,\dot{x},\dot{y},\dot{z})^T \in \Xi \subseteq \mathbb{R}^6$ is the phase-space vector, where $\Xi$ is the set of admissible CoM positions and velocities. Then from Eq. (\ref{eq:balance-3}) it can be shown that the prismatic inverted pendulum model for a $q^{\rm th}$ walking step, is simplified to the following control system
\begin{align}\label{eq:accel}
\dot{\boldsymbol{\xi}} = \boldsymbol{\mathcal{F}}(q, \boldsymbol{\xi}, \boldsymbol{u}) = 
 \begin{pmatrix}
  \dot{x}\\
  \dot{y}\\
  \dot{z}\\
  \omega_q^2 (x - x_{{\rm foot}_q}) - \frac{ \omega_q^2}{mg}( \tau_y + b_q \tau_z)\\[2mm]
  \omega_q^2 (y - y_{{\rm foot}_q}) - \frac{ \omega_q^2}{mg}(\tau_x + a_q \tau_z)\\[2mm]
 a_q\omega_q^2 (x - x_{{\rm foot}_q}) + b_q\omega_q^2 (y - y_{{\rm foot}_q}) - \frac{\omega_q^2}{mg}(a_q \tau_y + b_q \tau_x + 2 a_q b_q \tau_z)
 \end{pmatrix},
\end{align}
where the phase-space asymptotic slope is defined as
\begin{align}\label{eq:CoMaccelRate}
\omega_q = \sqrt{\dfrac{g}{z_{{\rm apex}_q}}}, \;{\rm with}\; z_{{\rm apex}_q} = (a_q \cdot x_{{\rm foot}_q} + b_q \cdot y_{{\rm foot}_q} + c_q - z_{{\rm foot}_q}),
\end{align}
where $g$ is the gravity constant. $a_q$ and $b_q$ are the slope coefficients while $c_q$ is the constant bias for the linear CoM path surfaces that we consider, i.e.
\begin{align}\label{eq:linear2DSurface}
\mathcal{S}_{{\rm CoM}_q} = \left\{(x,y,z) \onR{3} \quad\Big|\quad 
                            \psi_{{\rm CoM}_q}(x,y,z) = z - a_q x - b_q y  - c_q = 0\right\}.
\end{align}
\begin{figure}[t]
 \centering
 \includegraphics[width=0.85\linewidth]{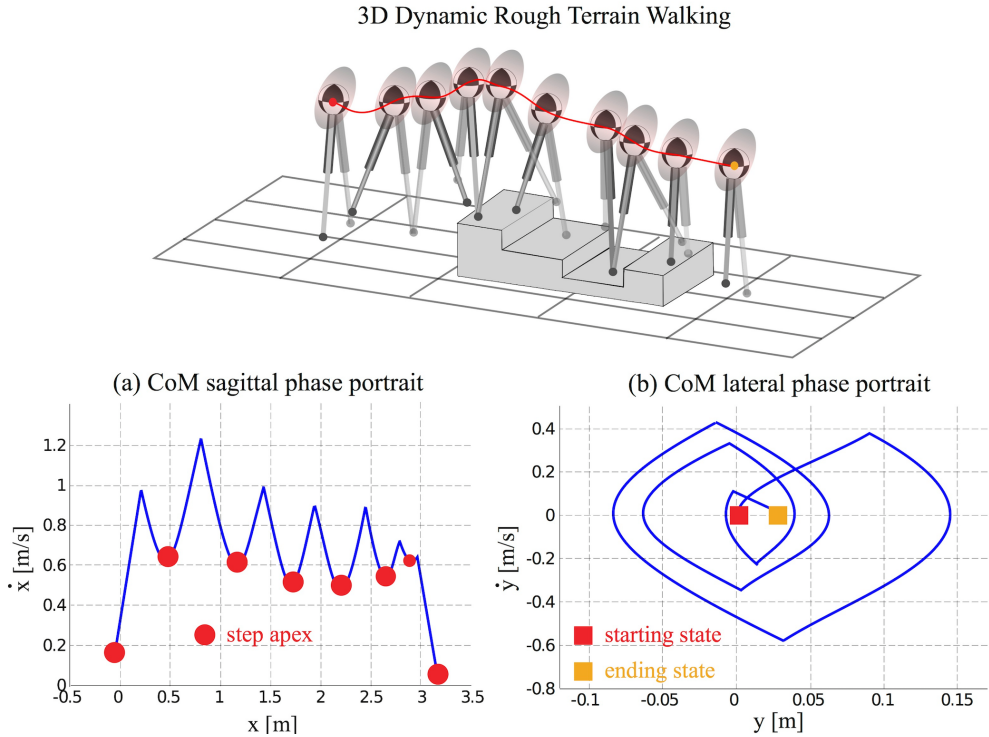}
 \caption{\captionsize 3D phase-space planning. Given step apex conditions, single contact dynamics generate the valley profiles shown in (a). (b) depicts a similar strategy in the lateral plane. However, since foot transitions have already been determined, what is left is to determine foot lateral positions. This is done so the lateral CoM behavior shown in (b) follows a semi-periodic trajectory that is bounded within a closed region.}
 \label{fig:3Dsinglecontact}
\end{figure}
\noindent Detailed derivations of Eq.~(\ref{eq:accel}) are provided in Appendix~\ref{sec:derivationEq}. $z_{{\rm apex}_q}$ is the height of the CoM at the apex of its sagittal path ``$x$ direction'' as shown in Fig.~\ref{fig:Inverted_Pendulum} such that it corresponds to the vertical distance between the CoM and the location of the foot contact at the instant when the CoM is on the top of the foot location. $\boldsymbol{\mathcal{F}}$ represents a vector field of inverted pendulum dynamics, which is assumed to be infinitely continuous and differentiable (i.e., $\mathcal{C}^{\infty}$) in the domain $\mathcal{D}(\boldsymbol{\xi})$ and globally Lipschitz in $\Xi$, given fixed control inputs.
In general, there will be a hybrid control policy $\boldsymbol{u}=\boldsymbol{\pi}(q,\boldsymbol{\xi})$ defined by the control variables $\boldsymbol{u} = (\omega_q, \boldsymbol{\tau}_{{\rm com}_q}, \boldsymbol{p}_{{\rm foot}_q})^T \in \mathcal{U}$,  where $\mathcal{U}$ is a set of admissible control values. The sets $\Xi$ and $\mathcal{U}$ are assumed to be compact. Our design of foot placement algorithms further into the paper will guarantee the tracking of keyframe states within a specified tolerance.

\begin{definition}[Sagittal and Lateral Apex]\label{def:sagapex}
The sagittal apex occurs when the projection of the CoM is equal to the location of the foot contact in the system's sagittal axis. The lateral apex is defined as the CoM lateral position when the sagittal apex occurs.
\end{definition}
The concept of apex state has been widely used in locomotion using the SLIP model to represent the state at the highest CoM position during the flight phase. In our case, we define keyframes as apex states during walking and use them as inputs to generate non-periodic trajectories, suitable for variable terrain heights. More details about contact switching strategy will be introduced in Section~\ref{subsec:StepTrans}.

From a physical perspective, the continuous control input $\omega_q$ in Eq.~(\ref{eq:accel}) is equivalent to modulating the leg force magnitude, since it can directly change the CoM accelerations by modulating the apex height $z_{\rm apex}$ as shown in Eq. (\ref{eq:CoMaccelRate}).
Using piecewise-linear CoM planes can cause sharp changes on the phase trajectories when the center-of-mass switches among multiple steps. To mitigate this problem, we will employ multi-contact strategies to smooth CoM trajectories. The multi-contact phase controls the CoM behavior when two feet are on the ground. 
Another point worthy to note is that although the CoM plane is piecewise-linear, the CoM path constrained within this plane is a 3D curve instead of piecewise-linear lines. Also, the sagittal and lateral phase space trajectories end up being continuous and smooth between contact phases.

\subsection{Multi-Contact Dynamics}
\label{subsec:dualContact}
We introduce a multi-contact model and briefly present how to modulate the internal tension force such that the friction cone constraints are satisfied. Differently from the 2D planar scenario described in [\cite{sentis2011humanoids}], this study focuses on the 3D walking. Based on the virtual linkage model [\cite{sentis2010compliant}], the multi-contact dynamics can be approximated by a multi-contact/grasp matrix as,
\begin{equation}\label{eq:CoMintension}
\begin{pmatrix}
\boldsymbol{f}_{\rm com}+m \boldsymbol{g}\\
\boldsymbol{\tau}_{\rm com}\\
f_{\rm int}
\end{pmatrix}
=[\boldsymbol{G}]_{7\times6}
\begin{pmatrix}
\boldsymbol{f}_{r_{\rm left}}\\
\boldsymbol{f}_{r_{\rm right}}\\
\end{pmatrix},
\end{equation}
where $f_{\rm int}$ represents the internal force along the line of dual feet contact points. $[\boldsymbol{G}]_{7\times6}$ is the multi-contact/grasp matrix defined as 

\begin{equation}\label{eq:GraspMatrix}
[\boldsymbol{G}]_{7\times6} = 
\begin{pmatrix}
[\boldsymbol{W}_{\rm com}]_{6\times6}\\[2mm]
[\boldsymbol{W}_{\rm int}]_{1\times6}
\end{pmatrix},
\end{equation}
By inverting Eq.~(\ref{eq:CoMintension}), we can solve the ground reaction forces for given center-of-mass inertial forces and moments, and a desired internal force trajectory
\begin{equation}\label{eq:reactionForce}
\begin{pmatrix}
\boldsymbol{f}_{r_{\rm left}}\\
\boldsymbol{f}_{r_{\rm right}}\\
\end{pmatrix}
= [\boldsymbol{G}]^+_{7\times6}
\begin{pmatrix}
\boldsymbol{f}_{\rm com}+m \boldsymbol{g}\\
\boldsymbol{\tau}_{\rm com}\\
f_{\rm int}
\end{pmatrix} = \boldsymbol{G}_f (\boldsymbol{f}_{\rm com}+m \boldsymbol{g}) + \boldsymbol{G}_{\tau} \boldsymbol{\tau}_{\rm com} + \boldsymbol{G}_{\rm int} f_{\rm int}.
\end{equation}
Matrices $\boldsymbol{W}_{\rm com}, \boldsymbol{W}_{\rm int}, \boldsymbol{G}_f, \boldsymbol{G}_{\tau}$ and $\boldsymbol{G}_{\rm int}$ are outlined in [\cite{sentis2011humanoids}]. Different from the method of simultaneously controlling CoM and internal force behaviors described in [\cite{sentis2010compliant}], this study implements the 
following procedure: (i) we first design a multi-contact phase trajectory between single contact phases that satisfies CoM position, velocity, and acceleration boundary conditions. The duration of the multi-contact phase and boundary velocities can be chosen by the designer. A similar multi-contact transition strategy, named "Continuous Double Support" (CDS) trajectory generator, was proposed in [\cite{englsberger2014trajectory}] to achieve smooth "Enhanced Centroidal Moment Pivot" (eCMP) and leg force profiles. We had ourselves previously used this strategy in [\cite{zhao2012three}]. (ii) Using Eq.~(\ref{eq:reactionForce}) and the CoM inertial wrench trajectory, we solve for the internal forces such that they satisfy friction constraints.
\subsection{Nominal Phase-Space Trajectory Generation}
We will first focus on the generation of trajectories in the sagittal plane of the robot's walking reference. Sagittal dynamics are represented \-- ignoring for simplicity, the discrete variable $q$, \-- in the first and fourth row of the system of Eq. (\ref{eq:accel}), i.e.
\begin{equation}\label{eq:dynx}
\boldsymbol{\dot x} = \boldsymbol{\mathcal{F}_x}(\boldsymbol{x}, \boldsymbol{u_x})=
\begin{pmatrix}
\dot x\\
\omega^2 (x - x_{\rm foot}) - \frac{ \omega^2}{mg}(\tau_y + b_q \tau_z)
\end{pmatrix}.
\end{equation}
This system would be fully controllable if its control inputs $\boldsymbol{u_x} = (\omega, \tau_y, \tau_z, x_{\rm foot})^T$ were unconstrained. However, their limited range urges us to first consider the motion trajectories under nominal values (i.e. open loop). As we previously motivated in Eq. (\ref{eq:linear2DSurface}), the path manifold $\mathcal{S}_{\rm CoM}$ is defined a priori to conform to the terrains via simple heuristic methods previously described in \mbox{[\cite{zhao2012three, sentis2011humanoids}]}. From Eq. (\ref{eq:CoMaccelRate}), once the path manifold is defined and for known contact locations, the set of phase-space asymptotic slopes $\omega$ is also known from Eq. (\ref{eq:CoMaccelRate}). For simplicity, the nominal flywheel moments are designed to be null, i.e. $\tau_y = 0, \tau_z = 0$. Under these considerations, the following algorithm produces nominal phase-space trajectories of the robot's center-of-mass in the sagittal direction of reference:\\

\noindent\textbf{Algorithm 1.} Nominal Phase-Space Trajectory Generation.
\begin{itemize}
\item[] \textbf{Input:}
\item[] \textbf{(i)}: $\mathcal{S}_{\rm CoM} \leftarrow \{\mathcal{S}_{{\rm CoM}_q} : [\zeta_{q-1}, \zeta_q] \rightarrow  \mathbb{R}^3, \; \forall q = 1, \ldots, N\}$
\item[] \textbf{(ii)}: $x_{{\rm foot}} \leftarrow \{x_{{\rm foot}_1}, x_{{\rm foot}_2}, \ldots, x_{{\rm foot}_{N}}\}$
\item[] \textbf{(iii)}: $\dot{x}_{{\rm apex}} \leftarrow \{\dot{x}_{{\rm apex}_1}, \dot{x}_{{\rm apex}_2}, \ldots, \dot{x}_{{\rm apex}_N}\}$
\item[] \textbf{(iv)}: $(\tau_y(t), \tau_z(t)) \leftarrow {\bf 0}$

\item[] \textbf{Operation:} 
\item[] \textbf{(i)}: $\omega \coloneqq \{\omega_1, \omega_2, \ldots, \omega_N\}$ is assigned via Eqs.~(\ref{eq:CoMaccelRate}) and~(\ref{eq:linear2DSurface})
\item[] \textbf{(ii)}: $(x(t), \dot x(t), \ddot x(t)) \leftarrow {\rm PIPM} (\omega, \tau_y(t), \tau_z(t), x_{\rm foot})$ via Eq.~(\ref{eq:dynx}) and the analytical solution proposed in Eq.~(\ref{eq:simplifiedPSM})

\textbf{Output:}
\item[] Phase-space trajectories $\mathcal{M}_{\rm CoM} := \bigcup_q^{} \mathcal{M}_{{\rm CoM}_q}$
\end{itemize}
The reader should refer to Fig.~\ref{fig:OverallPlanningStructure} to see the end-to-end planning and control process of the proposed locomotion methodology. It is specially important to understand that the desired CoM surfaces, nominal foot positions, keyframe states, and zero flywheel torques are provided a priori by the designer. A similar algorithm can be designed to generate trajectories in the lateral CoM direction via Eq.~(\ref{eq:accel}). Here, $\dot x_{\rm apex}$ represents the desired apex velocity. PIPM represents the prismatic inverted pendulum model defined in Eq. (\ref{eq:dynx}), used to derive CoM accelerations. Trajectories for multiple steps of a locomotion sequence on rough terrain are simulated using this process in Fig.~\ref{fig:3Dsinglecontact}.
\section{Hybrid Phase-Space Motion Planning}
\label{sec:3d-moplan}
In this section we propose a robust hybrid automaton [\cite{branicky1998unified, frazzoli2001robust, lygeros2008hybrid}] with the following key features: (i) an invariant bundle and a recoverability bundle to characterize control robustness, i.e., the bundle of attractiveness, and (ii) a non-periodic step transition strategy based on the previously described phase-space trajectories.  The hybrid automaton governs the planner's behavior across multiple walking steps and as such constitutes the theoretical core of our proposed phase-space locomotion planning method.

We continue our focus on sagittal plane dynamics first, then extend the planner to all directions. For practical purposes we will use the symbol $\boldsymbol{x} = \{x, \dot x\}$ to describe the sagittal state space associated with CoM dynamics. Note that this symbol represents now the output dynamics outlined in Eq. (\ref{eq:xi}) instead of the robot plant of Eq. (\ref{eqs:plant}). Eq. (\ref{eq:Mcom}) can thus be re-considered in the output space as $\mathcal{M}_{{\rm CoM}_q} = \left\{\boldsymbol{x} \in \mathcal{X} ~\big|~ \sigma_q(\boldsymbol{x}) = 0 \right\}$
 where $\sigma_q$ represents the deviation from the manifold $\mathcal{M}_{{\rm CoM}_q}$.
\begin{definition}[Invariant Bundle]\label{def:invariantBundle}
A set $\mathcal{B}_q(\epsilon)$ is an invariant bundle if, given $\boldsymbol{x}_{\zeta_0} \in \mathcal{B}_q(\epsilon)$, with $\zeta_0  \in \mathbb{R}_{\geq 0}$, and an increment $\epsilon>0$, $\boldsymbol{x}_{\zeta}$ stays within an $\epsilon$-bounded region of $\mathcal{M}_{{\rm CoM}_q}$,
\begin{align}\label{eq:path-manifold}
  \mathcal{B}_q(\epsilon) = \left\{\boldsymbol{x} \in \mathcal{X} 
   \quad\Big|\quad \left|\sigma_q(\boldsymbol{x})\right| \le \epsilon \right\},
 \end{align}
where $\zeta_0$ and $\zeta$ are initial and current phase progression variables, respectively and $\boldsymbol{x}_{\zeta_0}$ is an initial condition. 
\end{definition}
\noindent This type of bundle characterizes ``robust subspaces'' (i.e., ``tubes'') around nominal phase-space trajectories which guarantee that, if the state initializes within this space, it will remain on it.

\begin{definition}[Finite-Phase Recoverability Bundle]\label{def:recoverBundle}
 The invariant bundle $\mathcal{B}_q(\epsilon)$ around a phase-space manifold $\mathcal{M}_{\rm CoM_q}$ has a finite-phase recoverability bundle, $\mathcal{R}_q(\epsilon,\zeta_f)  \subseteq \mathcal{X}$ defined as,
 \begin{align}
   \mathcal{R}_q(\epsilon,\zeta_f) = 
      \left\{ 
        \boldsymbol{x}_{\zeta} \in \mathcal{X}, \quad \zeta_0\le\zeta\le \zeta_f
        \quad\Big|\quad 
        \boldsymbol{x}_{\zeta_f} \in \mathcal{B}_q(\epsilon)
       \right\}.
 \end{align}
\end{definition}

\noindent Note that this bundle assumes the existence of a control policy for recoverability. We will later use these metrics to characterize robustness of our controllers. Visualization of the invariant and recoverability bundles are shown in Fig. \ref{fig:Detailed-2Half-Steps2}.
\begin{figure}
\begin{center}
 \moveFigureUp{5mm}{\includegraphics[width=0.49\textwidth]{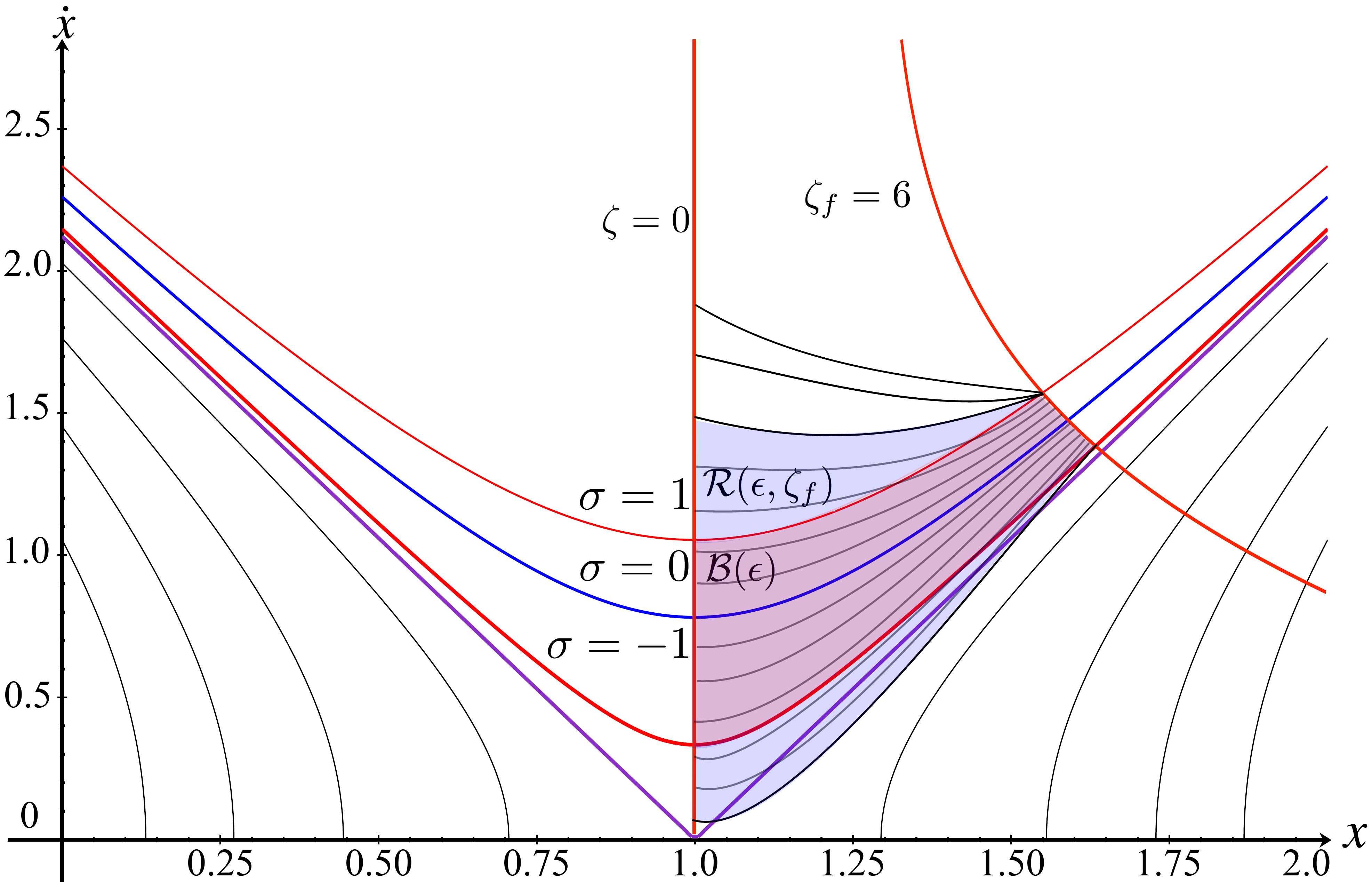}}
 \includegraphics[width=0.48\linewidth]{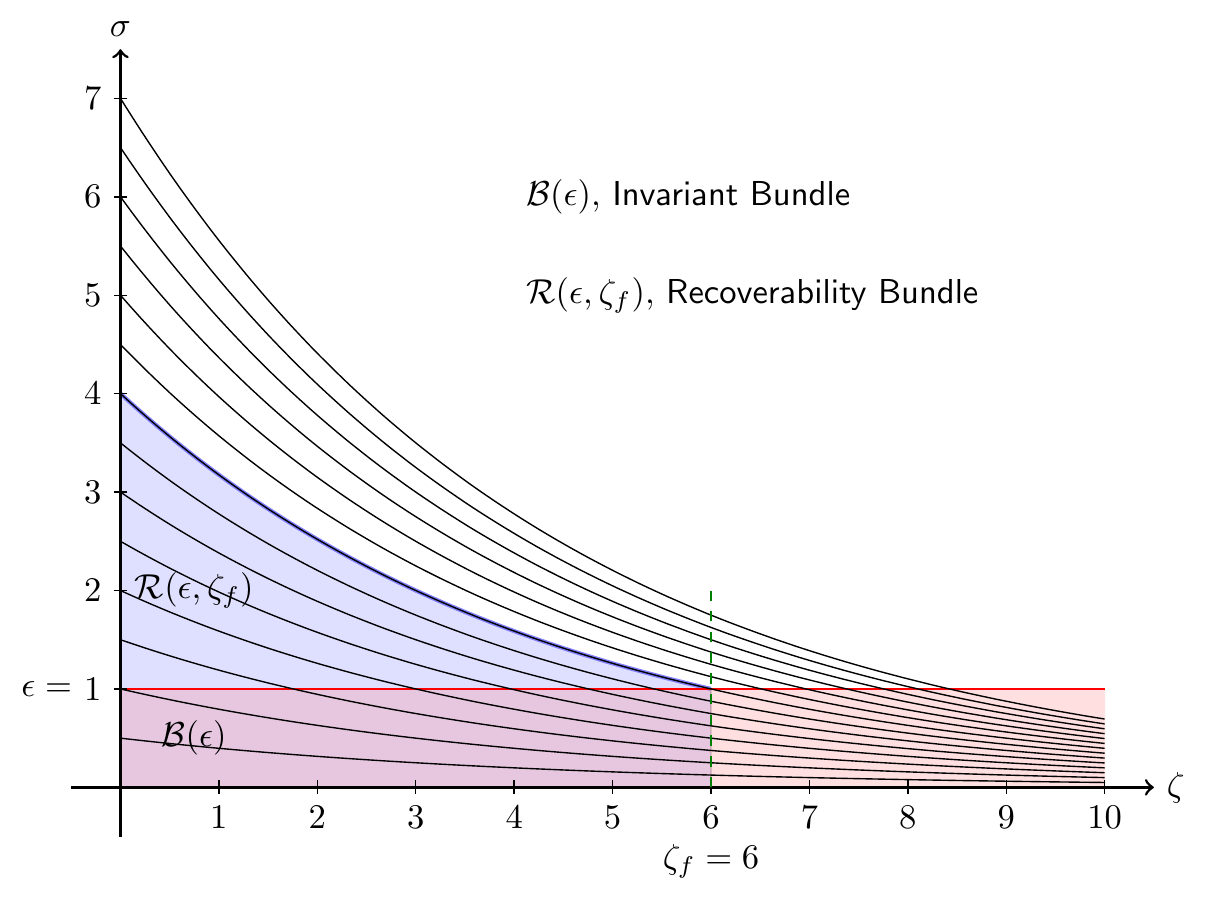}
 \caption{Mapping between Cartesian phase-space and $\zeta-\sigma$ space. The two subfigures show the invariant bundle $\mathcal{B}(\epsilon)$ (shown in red) and the recoverability bundle $\mathcal{R}(\epsilon,\zeta_f)$ (shown in blue) in different spaces. The left subfigure shows Cartesian phase-space while the right one shows $\zeta-\sigma$ space ($\sigma$ denotes the phase-space manifold as defined in Eq.~(\ref{eq:simplifiedPSM}) of Section~\ref{sec:manifold}). The figure on the right only shows positive bundles of $\sigma$ while the negative ones are symmetric about the $\zeta$ axis.  Since this is a Euclidean space, the manifold for a constant $\sigma$ is a horizontal line and constant values of $\zeta$ are vertical lines. If the condition when we expect the transition to occur is at $\zeta=\zeta_f$, the recoverability bundle shows the range of perturbations that can be tolerated at different $\zeta$ -- the system recovers to the invariant bundle before $\zeta_f$.}
 \label{fig:Detailed-2Half-Steps2}
\end{center}
\end{figure}
\subsection{Hybrid Locomotion Automaton}
\label{subsec:RHLA}
Legged locomotion is a naturally hybrid control system, with both continuous and discrete dynamics. We define discrete states $\mathcal{Q} = \{q_{l}, q_{r}, q_{s}\}$ representing the contact of the left foot $q_{l}$, the right foot $q_{r}$ or dual feet $q_{s}$ (stance) as shown in Fig.~\ref{fig:Automaton}.  In each mode, the continuous dynamics are represented by Eq.~(\ref{eq:dynx}) and over a domain $\mathcal{D}(q)$, except for the dual contact phase, $q_s$, where we use the multi-contact dynamic procedure defined in Subsection~\ref{subsec:dualContact}.  We characterize the hybrid system as a directed graph $(\mathcal{Q},\mathcal{E})$ (see Fig.~\ref{fig:Automaton}), with nodes represented by $q\in\mathcal{Q}$ and \cools{edges} represented by $\mathcal{E}(q,q+1)$, that characterize the transitions between nodes. The transitions between states can be grouped into eight classes depending on whether a vector field or variable changes discontinuously and what the trigger mechanism is.  Table~\ref{table:SwitchingGuard} shows the transition classification. 
\begin{table}[ht]
 \caption{Transition Classifications.  System vector field is $\boldsymbol{\mathcal{F}_x}$ as shown in Eq.~(\ref{eq:dynx}).}
  \begin{center}\vspace{-5mm}
   \begin{tabular}{c||c|c|c} \hline\hline
    {\bf Type}        & {\bf Transition}     & {\bf Switching} &          {\bf Jump}             \\     \hline \hline
    Autonomous  & $\Delta_a^{[\tau]}$ 
                & $\quad\boldsymbol{\mathcal{F}_x}^{+}(\cdot\;,\; \cdot\;, \boldsymbol{x}^{+}, \;\cdot\;,\; \cdot) \leftarrow \Delta_a^{[\delta_s]}(\boldsymbol{x}^{-})$
                & $\quad\boldsymbol{x}^{+} \leftarrow \Delta_a^{[\delta_j]}(\boldsymbol{x}^{-})$\\ 
    \hline 
    Controlled  & $\Delta_c^{[\tau]}$ 
                & $\quad\boldsymbol{\mathcal{F}_x}^{+}(\cdot\;,\; \cdot\;,\; \cdot\;,\; \boldsymbol{u_x}^{+},\; \cdot) \leftarrow \Delta_c^{[\delta_s]}(\boldsymbol{u_x}^{-})$
                & $\quad\boldsymbol{u_x}^{+} \leftarrow \Delta_c^{[\delta_j]}(\boldsymbol{u_x}^{-})$\\ 
    \hline 
    ``Timed''     & $\Delta_t^{[\tau]}$ 
                & $\quad\boldsymbol{\mathcal{F}_x}^{+}(\zeta, \;\cdot\;,\; \cdot\;,\; \cdot\;, \;\cdot\;, \;\cdot) \leftarrow \Delta_t^{[\delta_s]}(\zeta)$
                & $\quad\boldsymbol{x}^{+} \leftarrow \Delta_t^{[\delta_j]}(\zeta)$\\ 
    \hline 
    ``Disturbed''    & $\Delta_d^{[\tau]}$ 
                & $\quad\boldsymbol{\mathcal{F}_x}^{+}(\cdot\;,\; \cdot\;,\; \cdot\;, \;\cdot\;, w_d) \leftarrow \Delta_d^{[\delta_s]}(w_d)$
                & $\quad\boldsymbol{x}^{+} \leftarrow \Delta_d^{[\delta_j]}(w_d)$\\ 
    \hline 
   \end{tabular}
  \end{center}
  \label{table:SwitchingGuard}
\end{table}

\noindent The hybrid automaton state is given by: $\boldsymbol{s}=(\zeta, q, \boldsymbol{x}^T)^T$.  $\tau\in\{\delta_s,\delta_j\}$ represents the ``switching'' or ``jump'' transition types, respectively. $\mu \in \{a, c, t, d\}$ represents the ``autonomous", ``controlled", ``timed" and ``disturbed" transitions, respectively. The transition map $\Delta_{\mu}^{[\tau]}(\cdot)$ is described in further details in Appnedix~\ref{sec:RHAutomaton}. Other details for these types of transitions can be found in [\cite{branicky1998unified}]. The condition that triggers the type of \cool{event} (switching or jump) is determined by a \cools{guard} $\mathcal{G}(q, q+1)$ for the particular edge $\mathcal{E}(q,q+1)$. Given these preliminaries, let us formulate a robust hybrid automaton to mathematically support our locomotion planner. 
\begin{figure}[t]
 \centering
   \includegraphics[width=.8\linewidth]{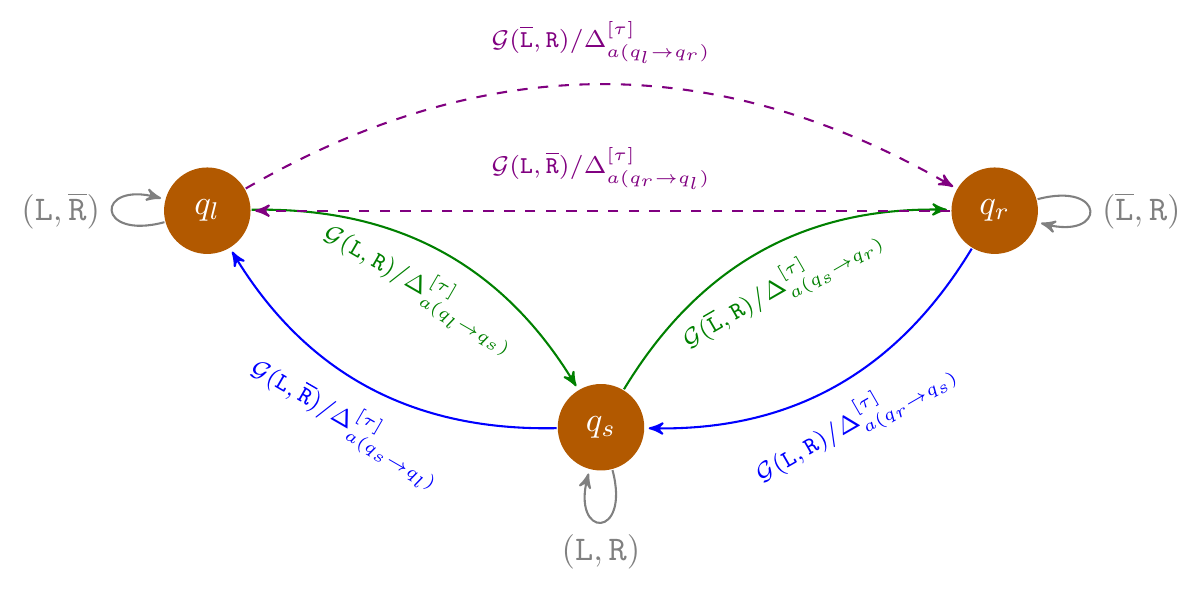}
 \caption{\captionsize This figure shows the hybrid locomotion automaton for a biped walking process. This automaton has three generic discrete modes $\mathcal{Q}=\{q_l,q_s,q_r\}$, that represent when the robot is in left leg contact ($q_l$), in right leg contact ($q_r$), and in dual stance contact ($q_s$), respectively.  Shown in the edges are the guard $\mathcal{G}(q, q+1)$ and the transition map $\Delta_{a(q\rightarrow q+1)}^{[\tau]}$. This locomotion automaton has non-periodic mode transitions.}
\label{fig:Automaton}
\end{figure}
\begin{definition}\label{def:PSRHA}
A phase-space robust hybrid automaton is a dynamical system, described by a $n$-tuple 
\begin{equation}\label{eq:PSRH}
{\rm PSRHA} \coloneqq (\zeta, \mathcal{Q},\mathcal{X}, \mathcal{U}, \mathcal{W}, \mathcal{F}, \mathcal{I},\mathcal{D}, \mathcal{R}, \mathcal{B}, \mathcal{E}, \mathcal{G}, \mathcal{T}, \Delta),
\end{equation}
\end{definition}
\noindent where $\zeta$ is the previously defined phase-space progression variable, $\mathcal{Q}$ is the set of discrete states, $\mathcal{X}$ is the set of continuous states, $\mathcal{U}$ is the set of control inputs, $\mathcal{W}$ is the set of disturbances, $\mathcal{F}$ is the vector field, $\mathcal{I}$ is the initial condition, $\mathcal{D}$ is the domain, $\mathcal{R}$ is the collection of recoverability bundles, $\mathcal{B}$ is the collection of invariant bundles, $\mathcal{E}\coloneqq \mathcal{Q}\times\mathcal{Q}$ is the edge, $\mathcal{G}: \mathcal{Q}\times\mathcal{Q}\rightarrow2^\mathcal{X}$ is the \cools{guard}, $\mathcal{T}$ is the transition termination set, and $\Delta$ is the transition map. More detailed definitions of these symbols are provided in Appendix~\ref{sec:RHAutomaton}, including arguments and subscripts. This automaton will be used to represent non-periodic trajectories since our planning process focuses on walking over irregular and disjointed terrain. A directed diagram of this non-periodic automaton is shown in Fig.~\ref{fig:Automaton}.

To the best of the authors' knowledge, this is the first formulation of a robust hybrid automaton used for dynamic locomotion. 
In Section~\ref{sec:optimization}, more details will be provided for how this automaton governs the hierarchical optimization sequence. To demonstrate the usefulness of this hybrid automaton, we provide an example of a planning sequence as follows.

For example, consider a phase-space trajectory that contains two consecutive walking steps $\mathcal{Q} = \{q, q+1\}$ (e.g., left and right feet). Given an initial condition $(\zeta_0, q, \boldsymbol{x}_{q}(\zeta_0)) \in \mathcal{I}$, the hybrid system will evolve following the dynamics of Eq.~(\ref{eq:accel}) as long as the continuous state $\boldsymbol{x}_{q}$ remains in $\mathcal{D}(q)$ (e.g., one foot in the ground the other one is swinging).  If at some point $\boldsymbol{x}_{q}$ reaches the guard $\mathcal{G}(q, q+1)$ (e.g., the right foot touches the ground) of some edge $\mathcal{E}(q, q+1)$, the discrete state will switch from $q$ to $q+1$.  At the same time the continuous dynamics will reset to some value via $\Delta^{[\tau]}_{a(q\rightarrow q+1)}$.  After this transition, the whole procedure repeats.
\subsection{Step Transition Strategy}
\label{subsec:StepTrans}
Step transitions can be analyzed as an idealize instantaneous contact change (as in Fig.~\ref{fig:PS-stepTransition}(a)) or being more realistic by having a multi-contact phase (as in Fig.~\ref{fig:PS-stepTransition}(b)). Below, we first create a strategy for the instantaneous contact switch, and then extend it to the multi-contact case in Appendix~\ref{sec:multicontact}.
\begin{definition}[A Phase-Space Walking Step]\label{def:simple-walking-step} A walking step, $q^{\rm th}$, is a phase-space trajectory in domain $\mathcal{D}(q)$, having two guards $\mathcal{G}(q-1, q)$ and $\mathcal{G}(q, q+1)$.
\end{definition}
\noindent To characterize the non-periodic mapping associated with walking in rough terrains, we define a keyframe map between keyframe states.

\begin{definition}[Keyframe Map of Non-Periodic Gaits]\label{def:snpg} We define the keyframe map of non-periodic gaits as a return map $\Phi$ that takes the robot's center-of-mass from one desired keyframe $(\dot{x}_{{\rm apex}_q}, x_{{\rm foot}_q}, \theta_{q})$ to the next one, and via the control input $\boldsymbol{u_x}$ i.e.
\begin{equation}\label{eq:mapping}
(\dot{x}_{{\rm apex}_{q+1}}, x_{{\rm foot}_{q+1}}, \theta_{q+1}) = \Phi(\dot{x}_{{\rm apex}_q}, x_{{\rm foot}_q}, \theta_{q}, \boldsymbol{u_x}).
\end{equation}
where $\theta_{q}$ represents the heading of the $q^{\rm th}$ walking step.
\end{definition}
\noindent We will use the above map for the walking model of Section~\ref{subsec:steeringDirec} which includes steering abilities. The above map addresses the nature of ``non-periodic" gaits by enabling arbitrary keyframe specifications. Users can design ``non-periodic" keyframes that change the speed and steer the robot through its walk. For this study, we use heuristics to design keyframes. More recently, we have proposed to use a keyframe decision maker based on linear temporal logic [\cite{zhao2016high}].

Our motion planner employs CoM apex states instead of touchdown states as keyframes due to our focus on non-periodic CoM dynamics. CoM apex states represent practical salient states for agile walking and help to design walking directions and velocities in a versatile fashion.
\begin{figure}[t]
 \centering
\includegraphics[width=\textwidth]{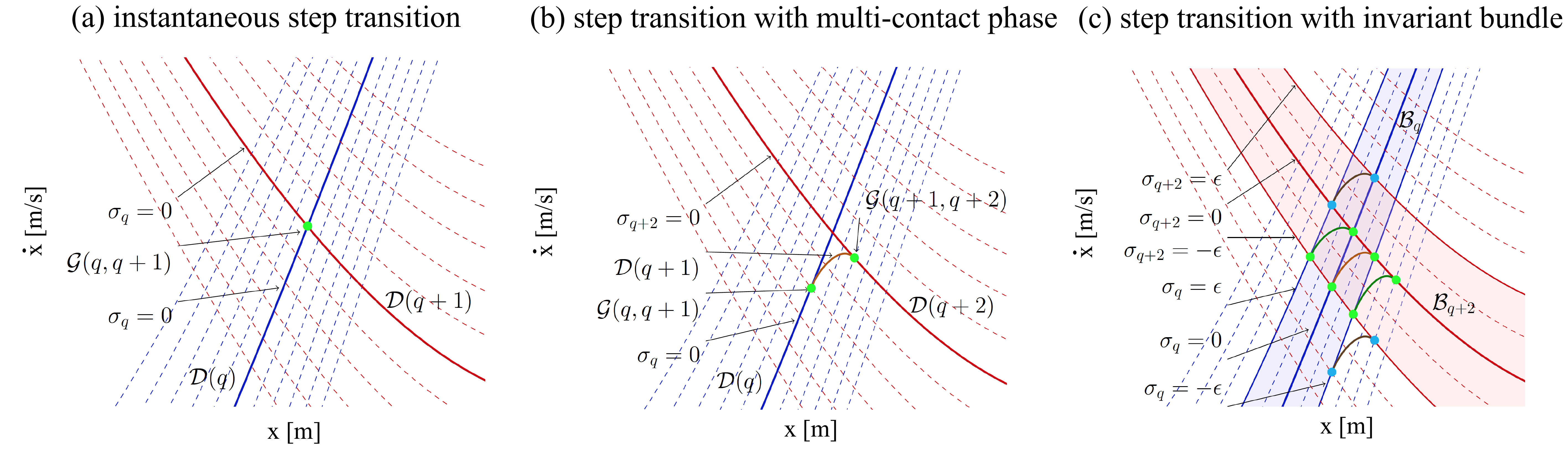}
 \caption{\captionsize Step transitions. This figure illustrates three types of step transitions in the sagittal phase-space, associated with $\sigma$-isolines: instantaneous step transition in (a), step transition with multi-contact phase in (b) and step transition with invariant bundle in (c). (a) switches between two single contacts instantaneously while (b) has a multi-contact phase. (c) shows several guard alternatives for multi-contact transitions, from the current single-contact manifold value $\sigma_q$ to the next single-contact step bundle $\sigma_{q+2}$.  In particular the invariant bundle bounds $\sigma_{q} = \pm\epsilon$ are shown.  The transition bundle in green reattaches to the original manifold $\sigma_{q+2}=0$, while the transition bundle (in brown) maintains its $\sigma$ value, i.e., $\sigma_{q+2}=\sigma_q$.}
 \label{fig:PS-stepTransition}
\end{figure}
\begin{definition}[Phase Progression Transition Value] \label{def:progVariable}
A phase progression transition value $\zeta_{\rm trans} : \mathcal{Q} \times \mathcal{X} \rightarrow \mathbb{R}_{\geq 0}$ is the value of the phase progression variable when the state $\boldsymbol{x}_q$ intersects a guard $\mathcal{G}$, i.e.,
\begin{align}
\zeta_{\rm trans} \coloneqq \inf \{\zeta > 0 \quad {\large|} \quad 
                  \boldsymbol{x}_{q} \in \mathcal{G}\}.
\end{align}
\end{definition}
\noindent We propose an algorithm to find transitions between adjacent steps, which occur at $\zeta_{\rm trans}$. Given known step locations and apex conditions, phase-space trajectories can be derived using Algorithm 1. One of the characteristics of pendulum dynamics portrayed in the phase-space is displaying infinite slopes when crossing the zero-velocity axis [\cite{zhao2013phase}]. To deal effectively with this difficulty, we use NURBS (non-uniform rational B-splines)\footnote{Different from polynomials, non-rational splines or B\'ezier curves, NURBS can be used to precisely represent conics and circular arcs by adding weights to control points.} for fitting the data from our numerical integration process (see Fig. \ref{fig:PSMSurf} for an illustration of adjacent step manifolds). Subsequently, finding step transitions, $\zeta_{\rm trans}$, consists on finding the root difference between adjacent curves. Such a process is straightforward using NURBSs. The pipeline for finding step intersections is shown below.
\begin{center}
   \includegraphics[width=0.8\textwidth]{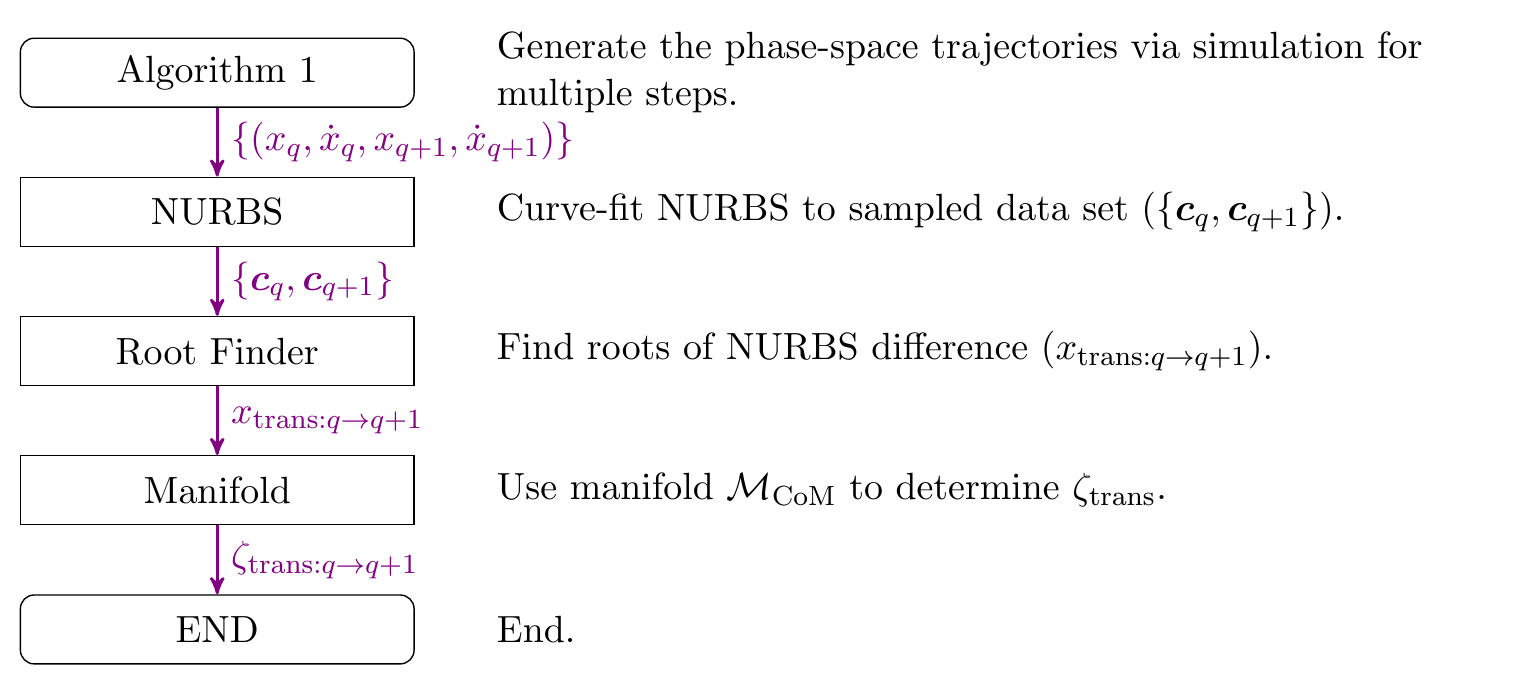}
 \end{center}
For clarity, sagittal apices are the states relating the robot's CoM velocities to their positions when crossing the sagittal contact positions. On the other hand, the instants at which contact transitions occur, derived from the above algorithm, are halfway between apices.

In hybrid dynamics, impact dynamics have often been considered as a discrete jump to address the sudden joint velocity changes on the robot joints [\cite{yang2009framework, grizzle2014models}]. However, our model considers negligible impact dynamics for planning since the planning algorithm focuses exclusively on the CoM behavior and since we assume light limbs. This issue was recently addressed by our group in [\cite{WBOSC15}]: 1) most of the robot's mass is concentrated on the upper body and the legs are considered lightweight, 2) we assume there are practical elastic elements that reduce impacts such as foot bumpers or series elastic elements on some actuators, 3) the actuators are frictionless and have a low reflected inertia, 4) the robot's upper body is practically decoupled from the foot impact points due to the kinematic chains of the limbs, and (5)  the knee of the landing leg is away from a singular configuration, i.e. straight knee.

\subsection{Lateral Foot Placement Algorithm}
\label{subsec:searchlateral}
To complete the 3D walking planner, we formulate a searching strategy for lateral foot placement that complies with the timing of sagittal step transitions. The main objective of the lateral dynamics is to return the robot's center-of-mass to a walking center through a semi-periodic cycle. If lateral foot placements are not adequately picked, the lateral behavior will drift away or even become unstable. According to Eq.~(\ref{eq:accel}), lateral center-of-mass dynamics are equal to
\begin{equation}\label{eq:dyny}
\boldsymbol{\dot y} = \boldsymbol{\mathcal{F}_y}(\boldsymbol{y}, \boldsymbol{u_y})=
\begin{pmatrix}
\dot y\\
\omega^2 (y - y_{\rm foot}) - \frac{ \omega^2}{mg}(\tau_y + a_q \tau_z)
\end{pmatrix},
\end{equation}
which can be simulated adapting Algorithm 1 to the lateral dynamics (see Fig.~\ref{fig:Numerical_Integration} for simulations of lateral dynamics).
\setcounter{algorithm}{1}
\begin{algorithm}
\begin{algorithmic}[1]
\setstretch{1.3}
\STATE Initialize iteration index $n \leftarrow 1$, maximum iterations $n_{\rm max}$, tolerance $\dot{y}_{\rm tol}$ and initial state $y_{\rm init}, \hat{y}_{\rm foot}(1)$, $\ddot{y}_{\rm apex}(1)$
\STATE  $\dot{y}_{\rm apex}(1) \leftarrow$ integration of inverted pendulum model given in Eq.~(\ref{eq:dyny}) with $\hat{y}_{\rm foot}(1)$
\WHILE {$n < n_{\rm max}$ and $|\dot{y}_{\rm apex}(n)| >\dot{y}_{\rm tol}$} 
\STATE $\quad \hat{y}_{\rm foot}(n+1) = \hat{y}_{\rm foot}(n) - \dot{y}_{\rm apex}(n)/\ddot{y}_{\rm apex}(n)$ by Newton-Raphson method
\STATE $\quad \dot{y}_{\rm apex}(n+1) \leftarrow$ integration of the inverted pendulum dynamics in Eq.~(\ref{eq:dyny}) with $\hat{y}_{\rm foot}(n+1)$
\STATE $\quad \ddot{y}_{\rm apex}(n+1) = (\dot{y}_{\rm apex}(n+1) - \dot{y}_{\rm apex}(n))/(\hat{y}_{\rm foot}(n+1) - \hat{y}_{\rm foot}(n))$
\STATE $\quad n \leftarrow n +1$
\ENDWHILE
\end{algorithmic}
\caption{Newton-Raphson Search for Lateral Foot Placement}
\label{al:newtonraphson}
\end{algorithm}
To generate bounded lateral trajectories, we choose the simple criterion of achieving zero apex lateral velocity $\dot{y}_{\rm apex} = 0$ at the instant when the CoM lateral apex position $y_{\rm apex}$ is located between the two feet. Here $y_{\rm apex}$ and $\dot{y}_{\rm apex}$ are the CoM lateral position and velocity when the center-of-mass crosses the sagittal apex as defined in Def. \ref{def:sagapex}. Algorithm~\ref{al:newtonraphson} achieves this objective by using the Newton-Raphson method, which approximates the roots of real-valued functions. In our case, this function is chosen to be the apex lateral velocity $\dot{y}_{\rm apex}$ with the lateral foot placement $\hat{y}_{\rm foot}$ as its independent variable (as shown in line 4 of Algorithm~\ref{al:newtonraphson}). $\hat{y}_{\rm foot}(n)$ represents the estimated lateral foot placements for the $n^{\rm th}$ search iteration. Feasible ranges $\hat{y}_{\rm foot, {\rm min}} \leq \hat{y}_{\rm foot} \leq \hat{y}_{\rm foot, {\rm max}}$ and a maximum number of iterations $n < n_{\rm max}$ are also enforced. Examples of usage of this algorithm are shown in Fig.~\ref{fig:Numerical_Integration}.
\begin{figure*}
 \centering
   \includegraphics[width=.85\linewidth]{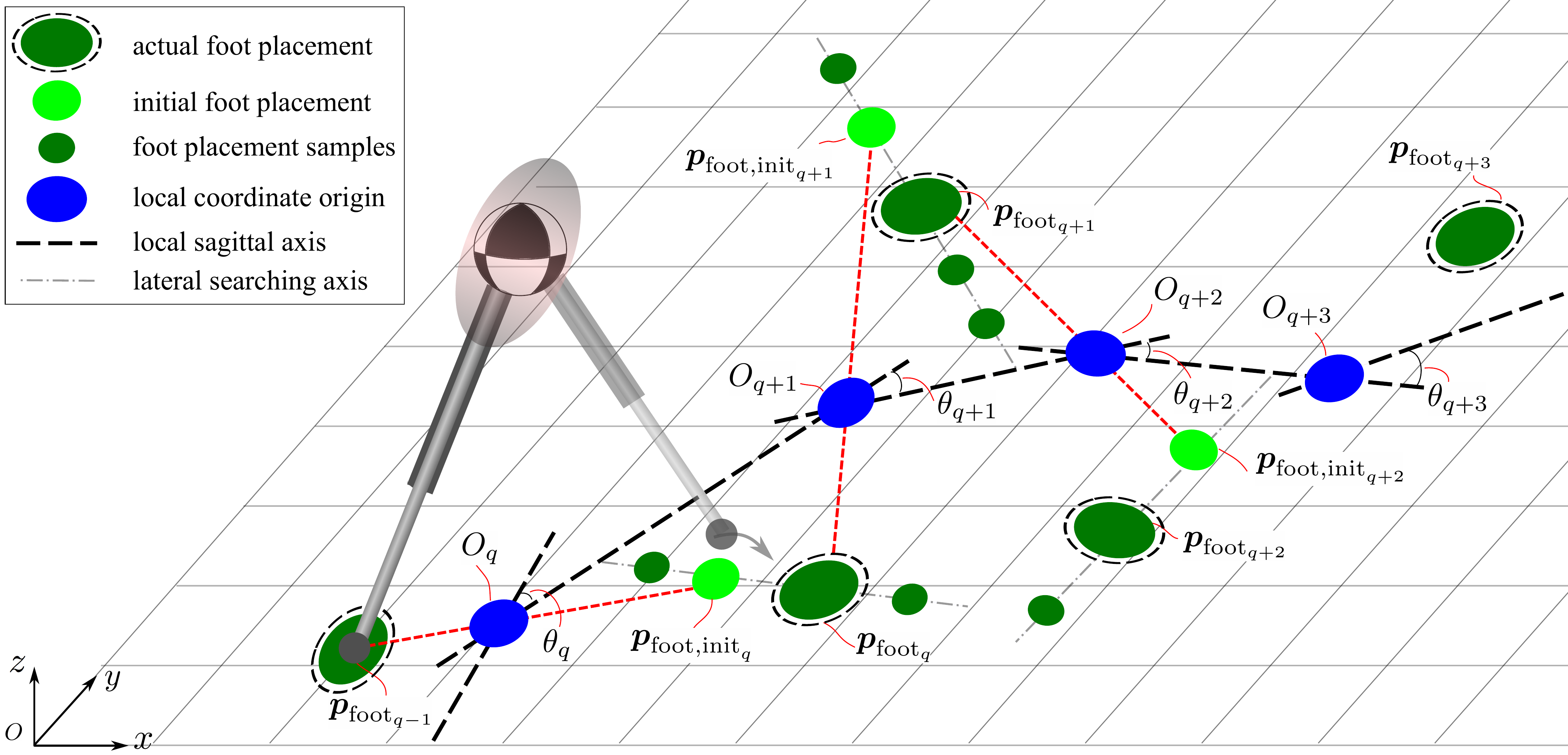}
 \caption{\captionsize Strategy for Steerable Walking. We define a local coordinate frame with origin represented by $\blueellipsoid$ and local sagittal axes represented by the black dash lines. The lateral foot placement searching algorithm described in Algorithm~\ref{al:newtonraphson} is applied using the newly defined local frames. $\greenellipsoid$ represents the final foot locations found via this procedure.
}
 \label{fig:SteerDirect}
\end{figure*}
\begin{figure*}[!t]
 \centering
   \includegraphics[width=.95\linewidth]{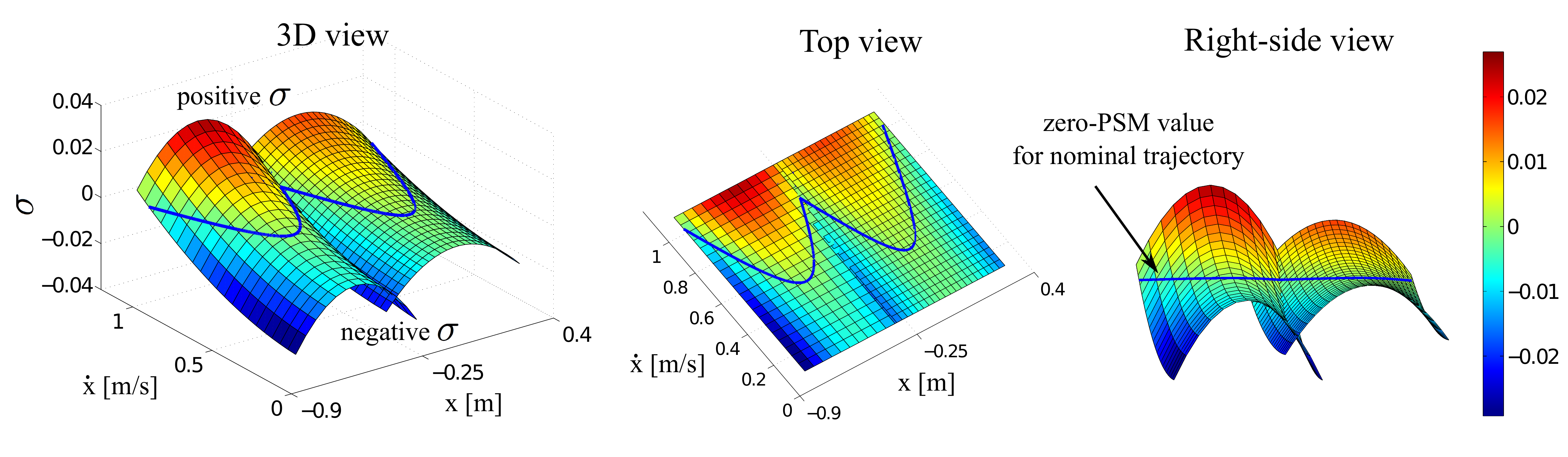}
 \caption{\captionsize Phase-space manifold isolines. This three-dimensional space demonstrates our phase-space manifold isolines defined in Eq.~(\ref{eq:simplifiedPSM}) by the color map. The horizontal plane represents the sagittal phase-space while the vertical third dimension represents the non-zero $\sigma$ value in Eq.~(\ref{eq:simplifiedPSM}). As we can see, the blue nominal trajectory has a zero $\sigma$ value. The phase-space region above the nominal trajectory has positive $\sigma$ values while the lower region has negative $\sigma$ values.
}
 \label{fig:PSMSurf}
\end{figure*}

\subsection{Steerable Walking}
\label{subsec:steeringDirec}

%
 To plan practical walking behaviors, we introduce a steerable walking model based on local coordinates as shown in Fig.~\ref{fig:SteerDirect}. The pipeline for the steerable walking process is as follows: (i) define a local sagittal axis (black dash line) projected to level ground which specifies the heading angle $\theta_{q}$ for the $q^{\rm th}$ step; (ii) define the local origin $O_{q}$ (represented by $\blueellipsoid$) as the intersection of the local sagittal axis and a dash line (red dash line) connecting the previous foot placement $\boldsymbol{p}_{{\rm foot}_{q-1}}$ (represented by $\greenellipsoid$) and an initial guess of the foot placement $\boldsymbol{p}_{{\rm foot, init}_q}$ (represented by $\lightgreenellipsoid$); (iii) search the lateral foot placement with respect to the local frame (note that the lateral search line, shown as a gray color dash-dot line, is orthogonal to the local sagittal axis); (iv) once determined the foot placement $\boldsymbol{p}_{{\rm foot}_q}$, we generate CoM and foot trajectories for the walking step; (v) after this step, we provide the new desired heading angle $\theta_{q+1}$ and re-start the planning process to step (i). Once all trajectories in local frame are obtained, we transform them to global frame via the recorded heading angles. A circular walking example is provided in Section~\ref{subsec:circularwalking} using the above planning strategy.

\section{Phase-Space Manifold Analysis}
\label{sec:manifold}
In this section, we propose an analytical phase-space manifold (PSM) and use it as a metric of deviations from planned locomotion trajectories. 
\begin{proposition}[Phase-Space Tangent Manifold]\label{theorem:PSM}
Given the prismatic inverted pendulum dynamics of Eq. (\ref{eq:dynx}) with initial conditions $(x_0, \dot{x}_0)$ and known foot placement $x_{\rm foot}$, the phase-space tangent manifold is
\begin{align}\label{eq:surface1}
\sigma =  \; (x_0 - x_{\rm foot}) ^2\big(2\dot{x}^2_0 - \dot{x}^2 + \omega^2 (x - x_0) (x + x_0 - 2x_{\rm foot})\big)
- \dot{x}^2_0 (x - x_{\rm foot})^2 + \dot{x}^2_0 (\dot{x}^2 - \dot{x}^2_0)/\omega^2,
\end{align}
where $\sigma = 0$ represents the nominal phase-space manifold of locomotion. Additionally, $\sigma$ represents the Riemannian distance to the estimated locomotion trajectory.
\end{proposition}
\begin{proof}
See Appendix~\ref{sec:PSMDerivation}.
\end{proof}
\noindent If we use the apex conditions as initial values, i.e. $(x_0, \dot{x}_0) = (x_{\rm foot}, \dot{x}_{\rm apex})$, the tangent manifold can be easily simplified to
\begin{align}\label{eq:simplifiedPSM}
\sigma(x, \dot{x}, \dot{x}_{\rm apex}, x_{\rm foot}) = \dfrac{\dot{x}^2_{\rm apex}}{\omega^2} \big(\dot{x}^2 - \dot{x}^2_{\rm apex} 
                                            - \omega^2(x - x_{\rm foot})^2\big).
\end{align}
Since the tangent manifold is considered as a trajectory deviation metric in the phase-space, we will use it in the next section as a feedback control parameter to ensure robustness. Fig.~\ref{fig:PSMSurf} provides an illustration of the value of $\sigma$ as a function of the state. The same type of analysis can be done for lateral trajectory deviations using the lateral pendulum dynamics of Eq. (\ref{eq:dyny}).
\begin{figure}[t]
  \begin{center}
    \includegraphics[width=0.49\linewidth]{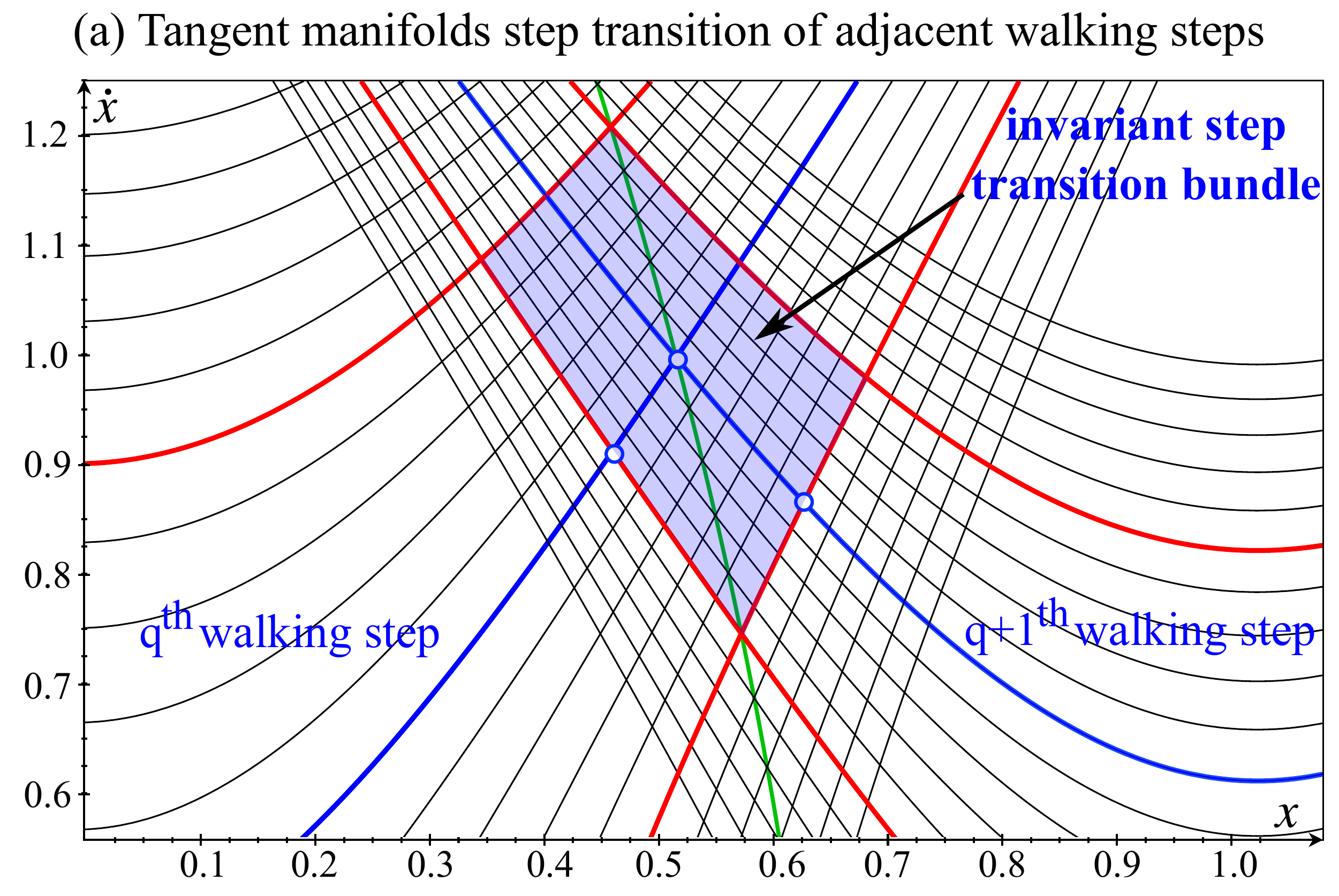}
    \includegraphics[width=0.45\linewidth]{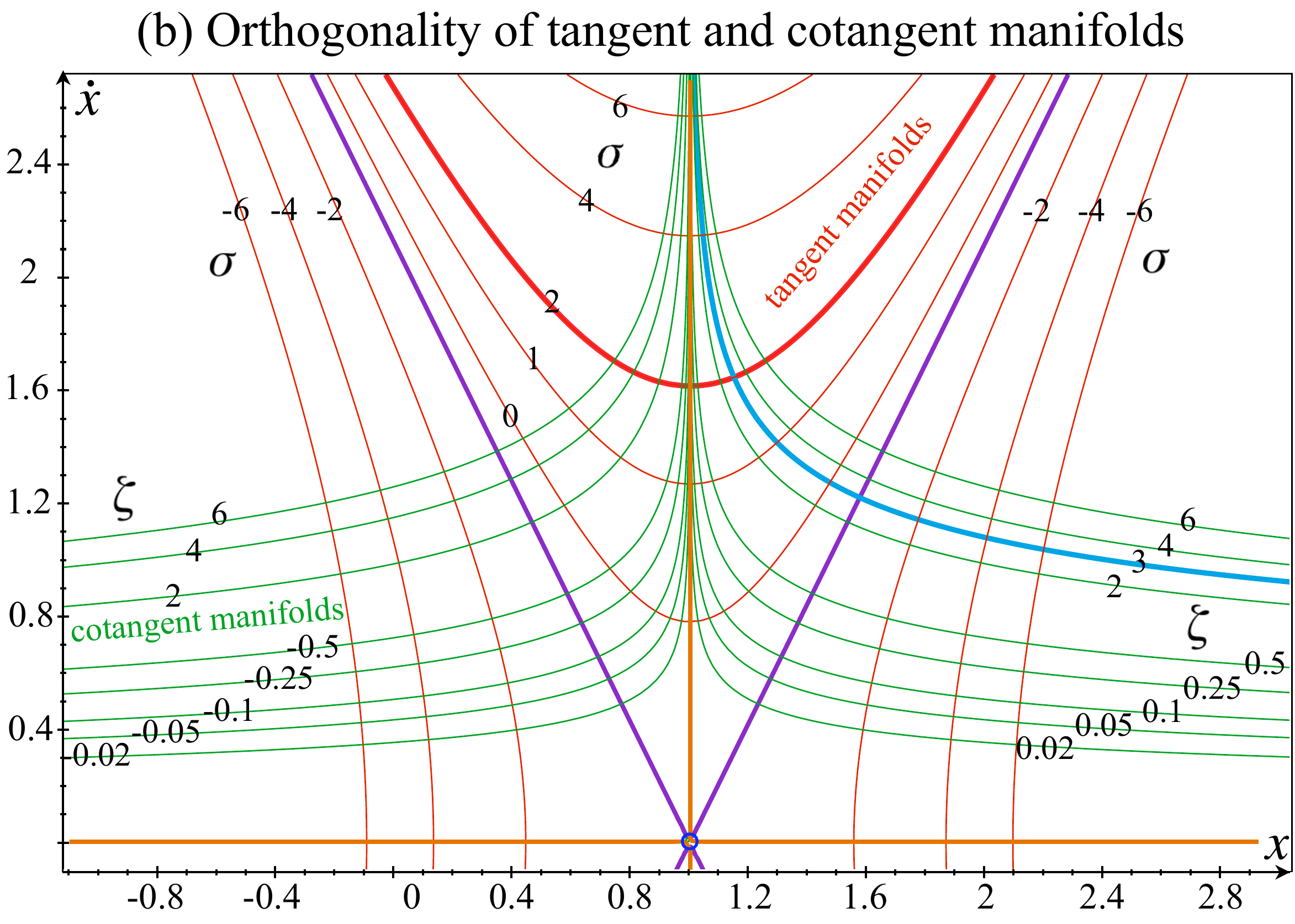}
 \caption{Phase-space tangent and cotangent manifolds. The left subfigure shows the tangent manifolds for two consecutive walking steps.  Nominal manifolds are shown as thick blue trajectories. The subfigure on the right shows orthogonal tangent and cotangent manifolds (i.e., iso-parametric curves) in phase-space. The tangent manifold $\sigma$ is defined in Eq. (\ref{eq:simplifiedPSM}) while the cotangent manifold $\zeta$ is defined in Eq. (\ref{eq:zeta_manifold}). The numbers in subfigure (b) represent the manifold deviations from the nominal one. For instance, the cotangent manifold $\zeta = 4$ has a tangent Riemannian distance 4 to the nominal cotangent manifold $\sigma = 0$.}
   \label{fig:Detailed-2Half-Steps1}
  \end{center}
 \end{figure}
\begin{proposition}[Phase-Space Cotangent Manifold]\label{prop:PSCoM}
Given the prismatic inverted pendulum dynamics of Eq.~(\ref{eq:dynx}), the cotangent manifold is equal to
\begin{align}\label{eq:zeta_manifold}
\zeta  = \zeta_0(\dfrac{\dot{x}}{\dot{x}_0})^{\omega^2} \dfrac{x - x_{\rm foot}}{x_0 - x_{\rm foot}},
\end{align}
and represents the arc length along the tangent manifold of Eq.~(\ref{eq:simplifiedPSM}). $\zeta_0$ is a nonnegative scaling factor, which represents the initial condition of a cotangent manifold. We choose it as the phase progression value when a contact switch occurs.
\end{proposition}
\begin{proof}
See Appendix~\ref{sec:OrthogonalManifoldDerivation}.
\end{proof}
\noindent Illustration of the tangent and cotangent manifolds are given in Fig. \ref{fig:Detailed-2Half-Steps1}. In subfigure~\ref{fig:Detailed-2Half-Steps1}(a), their intersection corresponds to the phase progression transition value $\zeta_{\rm trans}$ and the guard $\mathcal{G}_{q \rightarrow q+1}$.  Shown in red are the boundary manifolds $\sigma=\pm \epsilon$ of the invariant bundle $\mathcal{B}_q(\epsilon)$.  For the current $q^{\rm th}$ step, we can use the $\sigma= -\epsilon$ manifold of the next-step invariant bundle $\mathcal{B}_{q+1}(\epsilon)$ as the guard, namely, $\mathcal{G}_{q\rightarrow q+1}=\{(x, \dot{x}) \;\big|\; \sigma_{q+1}=-\epsilon\}$. In subfigure (b), the red tangent manifolds are shown as curves of constant $\sigma$ as defined in Eq.~(\ref{eq:simplifiedPSM}). Thick lines in purple are the asymptotes of tangent manifolds\footnote{These two asymptotes are equivalent to the eigenvector lines in Capture Point [\cite{pratt2006capture}]. The one in the second quadrant is stable while the one in the first quadrant is not.} where the thick red line illustrates a specific manifold $\sigma=2$. The asymptotes intercept the saddle point $(x_{\rm foot}, 0)$, where $x_{\rm foot}$ is the sagittal foot position.  The green cotangent manifold are curves of constant $\zeta$ that are orthogonal to the tangent manifolds.  Horizontal and vertical lines in orange are the asymptotes of cotangent manifolds, and the thick cyan line represents a specific manifold $\zeta=3$.  The vertical asymptote represents a manifold $\zeta=0$.
\begin{figure*}[t]
 \centering
   \includegraphics[width=\linewidth]
   {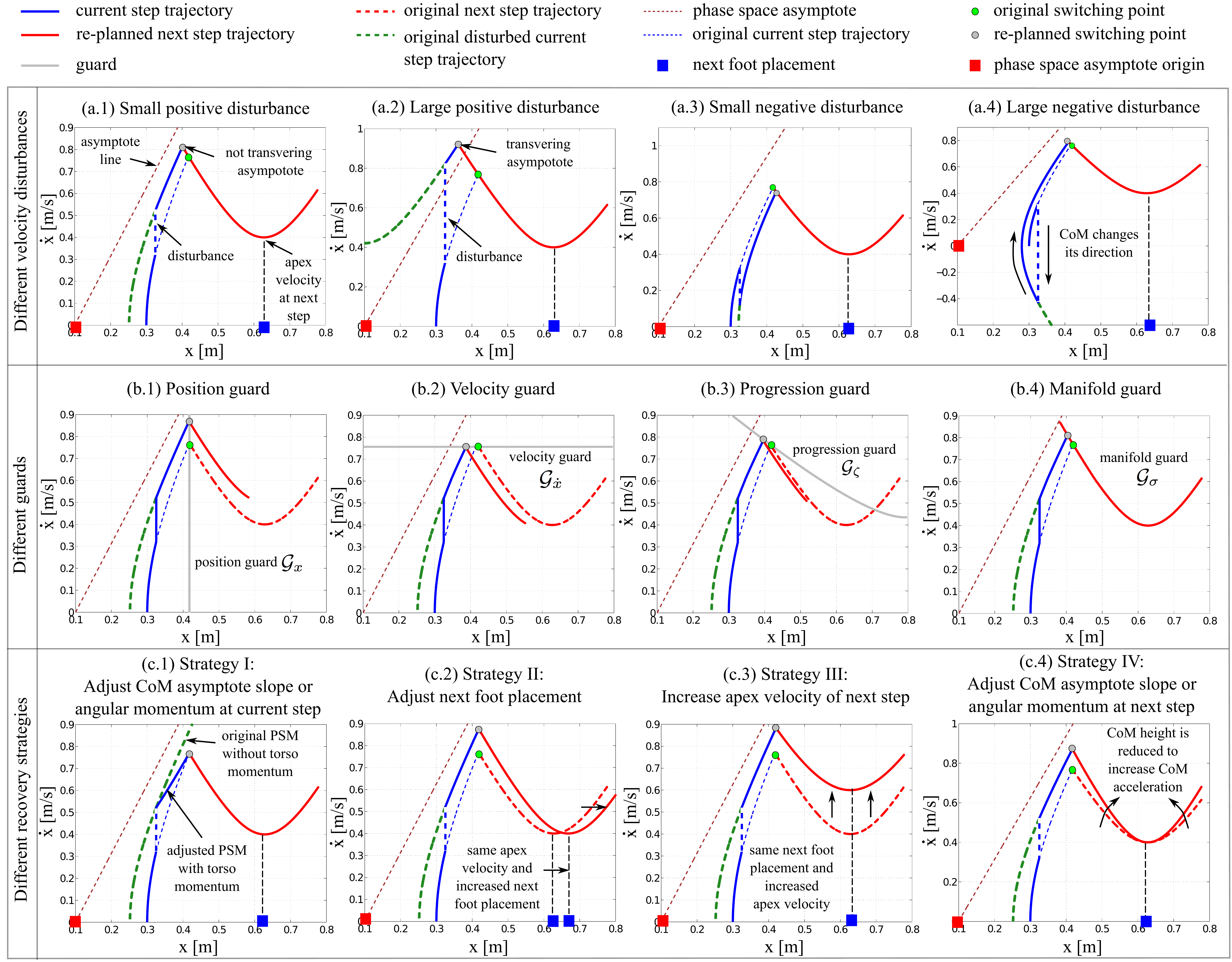}
 \caption{\captionsize Disturbance pattern, guard and recovery strategy classification. Four different velocity disturbance cases are shown in subfigures (a.1)-(a.4). The second row shows four proposed guards for the next step transition while the last row provides four recovery strategies.}
 \label{fig:VelocityDisturbSet1}
\end{figure*}

In robust control theory [\cite{zhou1996robust}], close-loop input-output behavior can be optimized using system norms. In this spirit, we define a new norm that characterizes sensitivity to disturbances of our non-periodic gaits, as
\begin{definition}[Phase-Space Sensitivity Norm]\label{def:sensitivityNorm}
Given a disturbance $d$, we define a phase-space sensitivity norm as
\begin{align}
\kappa\Big(\sigma(x_{\zeta_d}, \dot{x}_{\zeta_d})\Big) = \Big(\dfrac{1}{\zeta_{\rm trans} - \zeta_d} \int_{\zeta_d}^{\zeta_{\rm trans}}\sigma(x_\zeta, \dot{x}_\zeta)^2 d \zeta\Big)^{1/2},
\end{align}
where $\zeta_d$ corresponds to the phase value when a disturbance occurs and $\zeta_{\rm trans}$ is the phase transition for a given step defined in Def.~\ref{def:progVariable}. 
\end{definition}
\noindent In contrast to other sensitivity norms [\cite{hobbelen2007disturbance, hamed2016exponentially}], our gait norm evaluates disturbance sensitivity for non-periodic gaits. It does so by explicitly accounting for disturbance magnitude and for the instant where disturbances occur. And it does not rely on approximate linearization nor Taylor series expansion as periodic gait norms require. We will use this norm in the control section for dynamic programming. 
The disturbances that we consider are assumed to be impulses that change the CoM velocity instantaneously, regardless of the entity that generates them. They could be of diverse types: (i) instantaneous changes to the CoM behavior; (ii) continuous perturbations [\cite{englsberger2015three}]; (iii) terrain height disturbances [\cite{piovan2016approximation}], and (iv) friction-like drag forces. In the case of continuous force disturbances, the method proposed in [\cite{hyon2007full}] can be used to estimate the effect of unwanted external forces. In any case, our proposed disturbance characteristics and recovery strategies could address this diversity.

We consider various types of disturbances and outline potential recovery strategies. Disturbances can be categorized in the phase space based on four characteristics: (i) the disturbance direction, (ii) the disturbance magnitude, (iii) the terminal asymptote-region, and (iv) the change of the motion direction.  Fig.~\ref{fig:VelocityDisturbSet1} (a.1)-(a.4) illustrates those four scenarios, respectively. (a.2) has a larger positive disturbance than (a.1) such that velocity after the disturbed trajectory crosses the asymptote of the inverted pendulum model. On the other hand, (a.3) has a smaller negative disturbance such that velocity after disturbance keeps the same direction while (a.4) does not. In general, a disturbance can be characterized by its direction and magnitude. However, our study provides support for designing recovery strategies using the proposed phase space control strategies. In such case, we need to understand whether the disturbances cross terminal regions. This is the reason why we incorporate additional disturbance categories.

More disturbance scenarios could be defined, depending on specific occurrence states and characteristic patterns. We discuss various types of guard strategies to recover by changing step transitions \--- see Fig.~\ref{fig:VelocityDisturbSet1} (b.1)-(b.4). The guards shown are: position guard $\mathcal{G}_x$ (vertical line), velocity guard $\mathcal{G}_{\dot{x}}$ (horizontal line), progression guard $\mathcal{G}_\zeta$ ($\zeta$-isoline), and manifold guard $\mathcal{G}_\sigma$ ($\sigma$-isoline).  We find each guard such that they have the same transition point for the nominal phase-space manifold (PSM). Although this guard recovering strategy causes the motion to adjust, it might not be sufficiently corrective. If that is the case, we consider designing more recovery strategies by appropriately using control inputs. In the last four subfigures of Fig.~\ref{fig:VelocityDisturbSet1}, four recovery strategies are illustrated. These strategies are inspired by observations of human walking behaviors [\cite{hofmann2006robust, kuo1992human, abdallah2005biomechanically}] and further motivated via our experiences gained during extensive simulations. In the next section we explore control policies to deal with disturbances.

\section{Hybrid Control Strategy under Disturbances}
\label{sec:optimization}

This section formulates a two-stage control procedure to recover from disturbances.  When a disturbance occurs, the robot's CoM deviates from the planned phase-space manifolds obtained via Algorithm 1.  We use dynamic programming to find an optimal policy of the continuous control variables for recovery, and, when necessary, we re-plan feet placements from their initial locations based on the guards defined in Fig.~\ref{fig:VelocityDisturbSet1}. Our proposed controller, relies on the distance metric of Eq.~(\ref{eq:simplifiedPSM}) to steer the robot current's trajectory to the planned manifolds.  
\subsection{Dynamic Programming-Based Optimal Control}
\label{subsec:DP}
We propose a dynamic programming-based controller for the continuous control of the sagittal locomotion behavior. A similar controller can be formulated for the vertical CoM behaviors, given the PIPM dynamics of Eq.~(\ref{eq:accel}). To robustly track the planned CoM manifolds, we minimize a finite-phase quadratic cost function and solve for the continuous control parameters, i.e.
\begin{equation}\label{eq:optimization-1}
\begin{aligned}
&\underset{\boldsymbol{u}_{\boldsymbol{x}}^c}{\text{min}} \;\; \mathcal{V}_N(q, \;\boldsymbol{x}_N) + \sum_{n=0}^{N-1} \eta^n \mathcal{L}_n(q, \boldsymbol{x}_n, \boldsymbol{u}^c_{\boldsymbol{x}})& \\
&\textrm{subject to}:\; \boldsymbol{\dot x} = \boldsymbol{\mathcal{F}_x}(\boldsymbol{x}, \boldsymbol{u}^c_{\boldsymbol{x}}, d), \\
&\hspace{0.7in}\omega^{\rm min} \leq \omega \leq \omega^{\rm max}, \\
&\hspace{0.7in} \tau^{\rm min}_y \leq \tau_y \leq \tau^{\rm max}_y,
\end{aligned}
\end{equation}
where $\boldsymbol{u}_{x}^c = \{\omega, \tau_y\}$ corresponds to the continuous variables of the hybrid control input $\boldsymbol{u_x}$ of Eq. (\ref{eq:dynx})\footnote{For simplicity, the yaw torque $\tau_z$ is assumed to be zero in the disturbance case and thus it is not included in $\boldsymbol{u}_{x}^c$.}, $\omega$ and $\tau_y$ are scalars in this case, $0 \leq \eta \leq 1$ is a discount factor, $N$ is the number of discretized stages until the next step transition $\zeta_{\rm trans}$, the terminal cost is $\mathcal{V}_N(q, \boldsymbol{x}_N) = \alpha (\dot{x}(\zeta_{\rm trans}) - \dot{x}(\zeta_{\rm trans})^{\rm des})^2$. Here, $\dot{x}(\zeta_{\rm trans})$ is the terminal velocity associated with the disturbance at the instant of the next step transition, and $\dot{x}(\zeta_{\rm trans})^{\rm des}$ is the desired transition velocity at that instant. The first equality constraint $\boldsymbol{\mathcal{F}_x}(\cdot)$ is defined by the PIPM dynamics of Eq. (\ref{eq:dynx}) with an extra input disturbance $d$. Additionally, $\mathcal{L}_n$ is the one step cost-to-go function at the $n^{\rm th}$ stage defined as
\begin{align} \label{eq:L_s}
\mathcal{L}_n(q, \boldsymbol{x}_n, \boldsymbol{u}^c_{\boldsymbol{x}}) = & \int_{\zeta_{q, n}}^{\zeta_{q, n+1}} \big[\beta \sigma ^2 + \Gamma_1 \tau_y^2 + \Gamma_2 (\omega - \omega^{\rm ref})^2 \big] d\zeta,
\end{align}
where $\sigma$ is the tangent manifold of Eq.~(\ref{eq:simplifiedPSM}) used as a feedback control parameter, $\zeta_{q, n}$ and $\zeta_{q, n+1}$ are the starting and ending phase progression variables for the $n^{\rm th}$ stage of the $q^{\rm th}$ walking step, $\alpha$, $\beta$, $\Gamma_1$ and $\Gamma_2$ are weights, and $\omega^{\rm ref}$ is the reference phase-space asymptote slope given in Algorithm 1. Eq.~(\ref{eq:optimization-1}) is solved in a backward propagation pattern. More details of dynamic programming are provided in Appendix~\ref{sec:DP}. This optimal control process is applied only when a disturbance occurs. In disturbance-free scenario, no control adjustments are required, as the system naturally follows its CoM dynamics. We do not currently consider flywheel position limits as our focus has been on outlining a proof-of-concept control approach. For real implementations in future extensions of this work, we will need to account for flywheel dynamics and constraint on the positions.

To avoid chattering effects\footnote{This chattering is caused by the digital controllers with finite sampling rate. In theory, an infinite switching frequency will be required. However, the control input in practice is constant within a sampling interval, and thus the real switching frequency cannot exceed the sampling frequency. This limitation leads to the chattering.} in the neighborhood of the planned manifold, a $\epsilon$-boundary layer is defined and used to saturate the controls, i.e.
\begin{subequations}
\label{eq:slidingcontrol}
\begin{empheq}[left={\boldsymbol{u}^{c'}_{\boldsymbol{x}} = }\empheqlbrace]{align}\label{eq:slidingcontrol_a}
  & \boldsymbol{u}^c_{\boldsymbol{x}} & |\sigma| > \epsilon \\\label{eq:slidingcontrol_b}
  & \dfrac{|\sigma|}{\epsilon}\boldsymbol{u}^{c, \epsilon}_{\boldsymbol{x}} + \dfrac{\epsilon - |\sigma|}{\epsilon}\boldsymbol{u_x}^{c, {\rm ref}} & |\sigma| \leq \epsilon
\end{empheq}
\end{subequations}
where $\epsilon$ corresponds to the boundary value of an invariant bundle $\mathcal{B}(\epsilon)$ as defined in Def.~\ref{def:invariantBundle}, $\boldsymbol{u}_{\boldsymbol{x}}^{c, \epsilon} = \{\omega^\epsilon, \tau_y^\epsilon\}$ are control inputs at the instant when the trajectory enters the invariant bundle $\mathcal{B}(\epsilon)$, $\boldsymbol{u}_{\boldsymbol{x}}^{c, {\rm ref}} = \{\omega^{\rm ref}, \tau_y^{\rm ref}\}$ are nominal control inputs defined by Algorithm 1. A proof of smoothness of the above control saturation function is discussed in [\cite{utkin2013sliding}]. As Eq.~(\ref{eq:slidingcontrol}) shows, when $|\sigma| \leq \epsilon$, the control effort $\boldsymbol{u}^{c'}_{\boldsymbol{x}}$ is scaled between $\boldsymbol{u}_{\boldsymbol{x}}^{c, \epsilon}$ and $\boldsymbol{u}_{\boldsymbol{x}}^{c, {\rm ref}}$. This control law is composed of an ``inner'' and an ``outer'' controller. The ``outer'' controller steers states into $\mathcal{B}(\epsilon)$ while the ``inner'' controller maintains states within $\mathcal{B}(\epsilon)$. Note that, this controller performs better than asymptotic stability since the invariant bundle $\mathcal{B}(\epsilon)$ is reached in finite time. Recovery trajectories are shown in Fig.~\ref{fig:DP_robust_bound} for two scenarios in the presence of random disturbances.
\begin{figure}[t]
\centering
\includegraphics[width=0.95\linewidth]{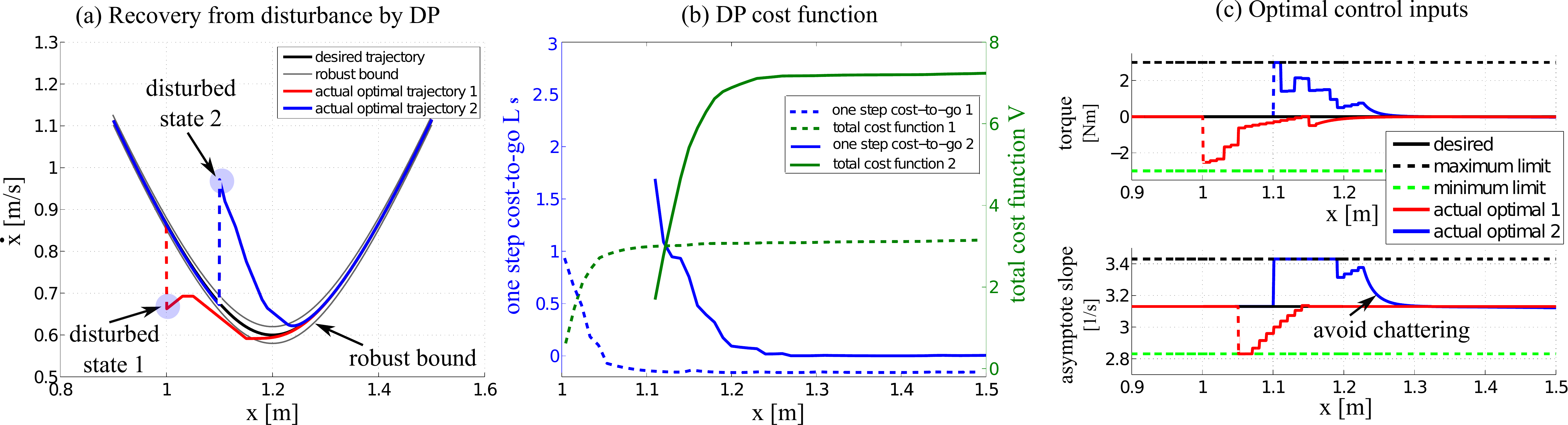}
\caption{\captionsize Chattering-free recoveries from disturbance by the proposed optimal recovery continuous control law. Subfigure (a) shows two random disturbances, where disturbed state 1 has a negative impulse while the disturbed state 2 has a positive impulse. Control variables are piecewise constant within one stage as shown in subfigure (c). In these simulations, torque control range is $[-3, 3]$ Nm and phase-space asymptote slope range is $[2.83, 3.43]$ 1$/$s. Other parameters are shown in Table~\ref{table:DP_parameter}.}
\label{fig:DP_robust_bound}
\end{figure}

\begin{table}[hbt!]
\caption{Dynamic Programming Parameters} 
\begin{center}\vspace{-5mm}
\begin{tabular}{c|c||c|c}
\hline\hline
Parameter & Range & Parameter & Range\\ \hline\hline
nominal pitch torque $\tau^{\rm ref}_y$ & 0 Nm& nominal asymptote slope $\omega^{\rm ref}$ &3.13 1$/$s\\ \hline
pitch torque range $\tau_y^{\rm range}$ & [-3, 3] Nm& asymptote slope range $\omega^{\rm range}$ &[2.83, 3.43] 1$/$s\\ \hline
apex height $z_{\rm apex}$ & 1 m & mass $m$ & 1 kg\\\hline
stage range & [0.9, 1.5] m & state range & [0.03, 1.5] m/s \\\hline
stage resolution & 0.01 m & state resolution & 0.01 m/s \\\hline
disturbed initial state $s_{\rm initial}$ & $(1.1$ m, $0.7$ m$/$s$)$ & nominal apex velocity $\dot{x}_{\rm apex}$ & 0.6 m$/$s\\\hline
weighting scalar $\Gamma_{1}$ & 5 &weighting scalar $\Gamma_{2}$ & 5\\\hline 
weighting scalar $\beta$ & $4 \times 10^4$ &weighting scalar $\alpha$ & 100 \\\hline 
\end{tabular}
\end{center}
\label{table:DP_parameter}
\end{table}

Since the control inputs are constrained within a desired range, i.e. $\boldsymbol{u}^c_{\boldsymbol{x}}  \in \boldsymbol{u}^{c, {\rm range}}_{\boldsymbol{x}}$, we re-define the finite-phase control-dependent recoverability bundle.
Given an acceptable deviation $\epsilon_0$ from the manifold, the practical invariant bundle is $\mathcal{B}(\epsilon_0)$.
The control policy in Eq.~(\ref{eq:slidingcontrol}) generates a control-dependent practical recoverability bundle (a.k.a., region of attraction to the ``boundary-layer'') defined as 
 \begin{align}\label{eq:OptimalRecoverabilty}
  \mathcal{R}(\epsilon, \zeta_{\rm trans}) = \Big\{\boldsymbol{x}_{\zeta} \in \mathbb{R}^2, \quad \zeta_0\le\zeta\le \zeta_{\rm trans}\; \big| \; 
    \boldsymbol{x}_{{\zeta}_{\rm trans}} \in \mathcal{B}(\epsilon), \quad \boldsymbol{u}^c_{\boldsymbol{x}}  \in \boldsymbol{u}^{c, {\rm range}}_{\boldsymbol{x}} \Big\}. 
 \end{align}

\begin{theorem*}[Existence of Recoverability Bundle]\label{theo:theorem}
Given a Lyapunov function $V = \sigma^2/2$, a phase progression transition value $\zeta_{\rm trans}$, and the control policy in Eq. (\ref{eq:slidingcontrol}), a recoverability bundle $\mathcal{R}(\epsilon, \zeta_{\rm trans})$ exists and can be bounded by a maximum tube radius $\sigma^{\rm max}_0$.
\end{theorem*}

\begin{proof}
Given $\exists\; \sigma_0 > \epsilon$ such that $\sigma_{\rm trans} \leq \epsilon$, then $\mathcal{R}(\epsilon, \zeta_{\rm trans})$ is composed of the range $(x, \dot{x})_{\zeta}, \; \zeta_0\le\zeta\le \zeta_{\rm trans}$, such that $V_{\rm trans} = \sigma_{\rm trans}^2/2\le \;\epsilon^2/2$. Taking the derivative of $V$ along the pendulum dynamics in Eq.~(\ref{eq:accel}), we have
\begin{align}\nonumber
\dot{V} = \sigma\dot{\sigma} &= \sigma \dot{x}^2_{\rm apex} \big( -2 \dot{x} (x - x_{\rm foot}) + 2 \dot{x} \ddot{x}/\omega^2\big) = \sigma  \dot{x}^2_{\rm apex} \Big( -2 \dot{x} (x - x_{\rm foot}) + 2 \dot{x} \big((x - x_{\rm foot}) - \dfrac{\tau_y}{mg}\big) \Big)\\\label{eq:stability}
&= -\dfrac{2  \dot{x}^2_{\rm apex} \sigma \dot{x} \tau_y}{mg} = -\dfrac{2 \sqrt{2}  \dot{x}^2_{\rm apex} \dot{x} \tau_y  \cdot {\rm sign}(\sigma)}{mg} \sqrt{V} \le 0.
\end{align}
\noindent which can prove the stability (i.e., attractiveness) of $\sigma=0$ under certain assumptions. For instance, consider the case of forward walking $\dot{x} > 0$. Then, as long as $\sigma \cdot \tau_y > 0$, i.e., the pitch torque has the same sign as $\sigma$, the attractiveness is guaranteed.  That is, if $\sigma>0$ (the robot moves forward faster than expected), then we need $\tau_y > 0$ to slow down, and vice-versa. If $\tau_y = 0$, then $\dot{V} = 0$, which implies a zero convergence rate. This means that the CoM state will follow its natural inverted pendulum dynamics without converging. As such, in order to converge to the desired invariant bundle, $\tau_y$ control action is required.

Note that, Eq.~(\mbox{\ref{eq:stability}}) shows the interesting phenomenon that $\dot{V}$ is independent of $\omega$, which cancels out during the derivation. The reason for this is the structure of the phase-space manifold $\sigma$ in Eq.~(\mbox{\ref{eq:simplifiedPSM}}), which is derived from Eq.~(\mbox{\ref{eq:surface1}}) by choosing the initial condition as $(x_0, \dot{x}_0) = (x_{\rm foot}, \dot{x}_{\rm apex})$. This $\omega$-independence makes our attraction analysis tractable.
\begin{figure}[t]
 \centering
   \includegraphics[width=0.9\linewidth]{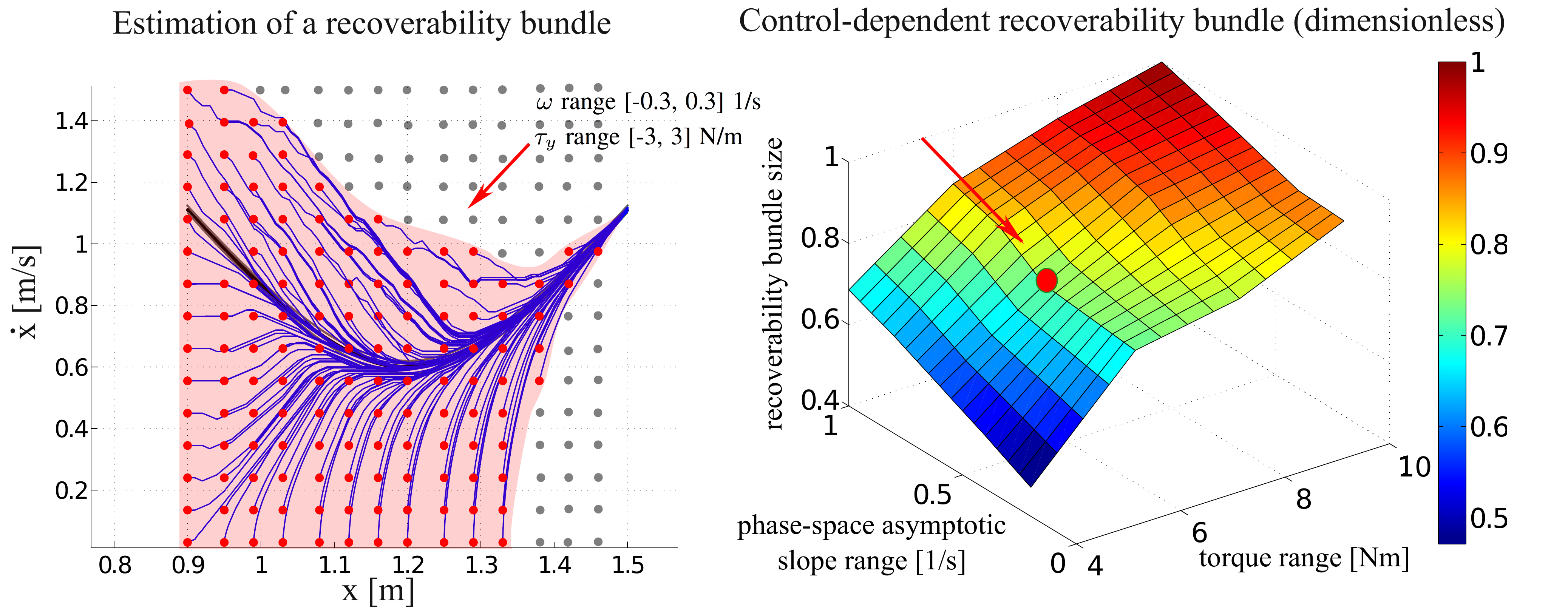}
 \caption{\captionsize Estimation of dimensionless control-dependent recoverable region. In the left figure, disturbed states are sampled in a discretized grid and the shaded region represents the recoverability bundle. As it is shown, a larger recoverable region is achieved at the beginning of the phase (i.e., before the apex state). In the ending phase, the recoverable region shrinks to the invariant bundle. Here the control constraint is: $\omega \in [-0.3, 0.3]$ 1/s and $\tau_y \in [-3, 3]$ N/m. The right figure shows the dependence of the size of the recoverable region with respect to the allowed control ranges. For better visualization, we use the range value to represent the control inputs in the horizontal axes. For instance, if the torque range is $r$, then it implies $\tau_y \in [-r/2, r/2]$.
 }
 \label{fig:RecoverabilitySet}
\end{figure}

To estimate $\mathcal{R}(\epsilon,\zeta_{\rm trans})$, we propose the following two methods: (i) use the optimal control policy proposed in Eq.~(\ref{eq:slidingcontrol}), defining an ``optimal'' recoverability bundle; or (ii) use the maximum control inputs (without any regards to optimality) obtained by selecting the bounds $\boldsymbol{u}^{c, {\rm range}}_{\boldsymbol{x}}$, defining a ``maximum'' recoverability bundle. These two cases can be characterized as:
\paragraph{Case I: DP based Control.}
 If $\tau_y$ is derived using the optimal controller of Eq.~(\ref{eq:optimization-1}), we get $|\tau_y| > |\tau^\epsilon_y|$. Then Eq.~(\ref{eq:stability}) becomes
\begin{align}\label{eq:recoverability-condition-1}
\dot{V} < -\dfrac{2 \sqrt{2}  \dot{x}^2_{\rm apex} \dot{x} |\tau^\epsilon_y|}{mg} \sqrt{V} <   0.
\end{align}
\paragraph{Case II: Supremum Control.}
 If we design $\tau_y= \tau_y^{\rm max}\sign(\dot{x})$ for the forward walking case, i.e., $\dot{x} > 0$, then,
\begin{align}\label{eq:recoverability-condition-2}
\dot{V} = -\dfrac{2 \sqrt{2}  \dot{x}^2_{\rm apex} \dot{x} \tau_y^{\rm max}\sign(\dot{x})}{mg} \sqrt{V} 
            = -\dfrac{2 \sqrt{2}  \dot{x}^2_{\rm apex} \dot{x} \tau_y^{\rm max}}{mg} \sqrt{V} < 0.
\end{align}
\noindent Note that, $\dot{V}$ in Eqs.~(\ref{eq:recoverability-condition-1}) and (\ref{eq:recoverability-condition-2}) have similar structure and can be analyzed considering the following integral equation, derived from basic manipulation of the equality $\dot V = dV / dt$
\begin{equation}\label{eq:integralLyapunov}
\int_{V_0}^{V_{\rm trans}} \dfrac{dV}{\sqrt{V}} \leq - \int_{t_0}^{t_{\rm trans}} \nu \dot{x} \tau_y dt = - \nu \tau_y(x_{\rm trans} - x_0),
\end{equation}
where $\nu = (2 \sqrt{2}  \dot{x}^2_{\rm apex})/(mg)$, $\tau_y = \tau_y^\epsilon$ for Case I while $\tau_y = \tau_y^{\rm max}$ for Case II. Eq.~(\ref{eq:integralLyapunov}) can be solved using common integral rules to yield 
\begin{equation}
\sqrt{V_0} \leq \sqrt{V_{\rm trans}} + \dfrac{1}{2}\nu \cdot (x_{\rm trans} - x_0) \cdot \tau_y.
\end{equation}
Since $V_0 = \sigma_0^2/2, V_{\rm trans} = \sigma_{\rm trans}^2/2 \leq \epsilon^2/2$, we get
\begin{align}
\quad \sigma_0 \leq \epsilon + \dfrac{\sqrt{2}}{2}\nu \cdot (x_{\rm trans} - x_0) \cdot \tau_y = \sigma^{\rm max}_0,
\end{align}
where $\sigma^{\rm max}_0$ defines the maximum tube radius. Therefore we can re-write the recoverability bundle of Eq.~(\ref{eq:OptimalRecoverabilty}) using this new tube radius as:
\begin{equation}\label{eq:practicalRecover}
\mathcal{R}(\epsilon, \zeta_{\rm trans}) = \Big\{\boldsymbol{x}_{\zeta} \in \mathbb{R}^2, \; \zeta_0\le\zeta\le \zeta_{\rm trans}\; \big| \; 
    \epsilon \leq \sigma_0^{\rm max}\Big\}.
\end{equation}
The existence of a recoverability bundle has been proven with a maximum tube radius, $\sigma_0^{\rm max}$.
\end{proof}
\noindent Since Eq.~(\ref{eq:recoverability-condition-1}) has an inequality bound while Eq.~(\ref{eq:recoverability-condition-2}) has an equality bound, DP based control is an optimal but conservative estimation of the true recoverability bundle while supremum control is an accurate but non-optimal estimation for the recoverability bundle. Our study aims at optimal performance, and therefore the control policy generated from dynamic programming will be used to estimate the recoverability bundle. 
\begin{figure}[!t]
\centering
\includegraphics[width=0.95\linewidth]{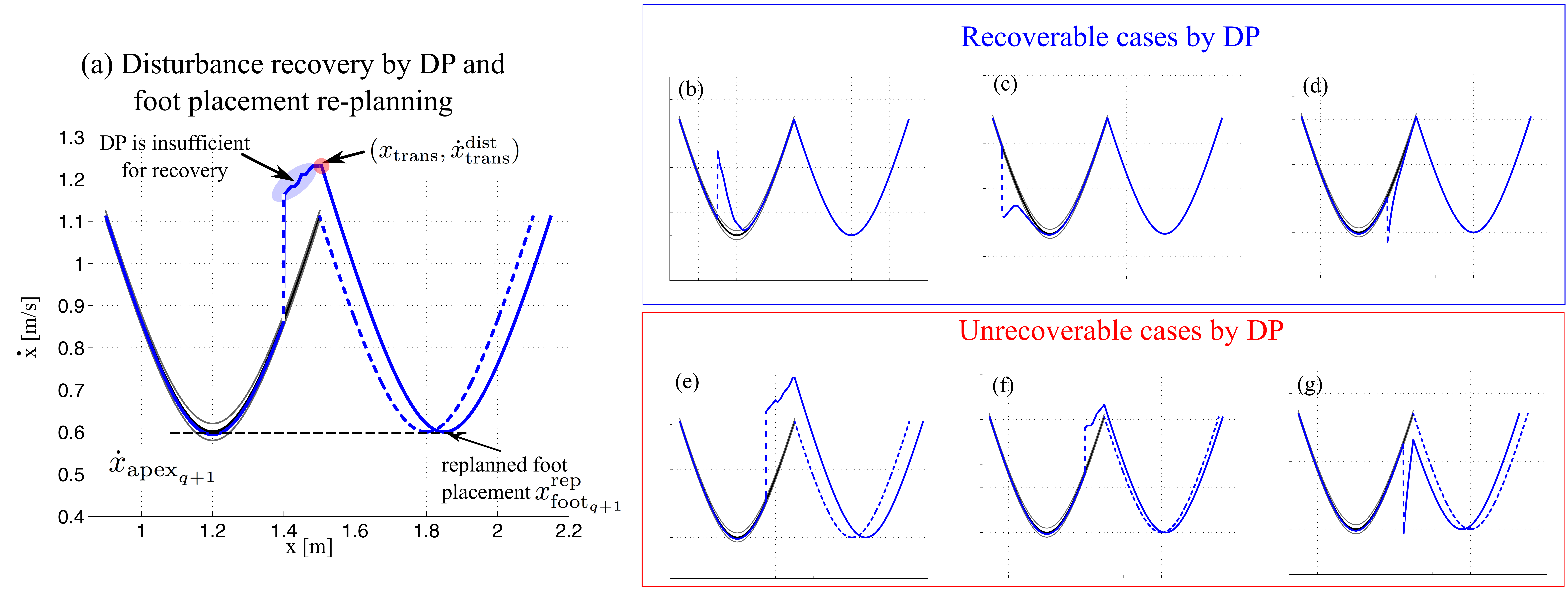}
\caption{\captionsize Recovery from a disturbance by re-planning sagittal foot placement. In this case, the next apex velocity is given \textit{a priori} and maintained despite the disturbance. In subfigures (b)-(d), first-stage continuous DP control is sufficient to achieve the recovery while in the cases of subfigures (e)-(g) it is not. The latter cases occur when either disturbance occurs too close to the transition or is too large. In these cases, a new next foot placement is automatically re-planed based on Eq.~(\ref{eq:replanfoot}).}
\label{fig:High_Level_Robust}
\end{figure}

To estimate $\mathcal{R}(\epsilon, \zeta_{\rm trans})$, we perform a grid sampling from the initial condition $\boldsymbol{x}_{\zeta_0}$, based on the ranges of Table~\ref{table:DP_parameter}. Then we execute the optimization of Eq.~(\ref{eq:optimization-1}) for each sampled $\boldsymbol{x}_{\zeta_0}$ (treated as a realization) and repeat this procedure for all $\boldsymbol{x}_{\zeta_0}$ in the grid. The feasible realizations of recovery trajectories (i.e. the convergence into $\mathcal{B}(\epsilon)$ before $\zeta_{\rm trans}$) constitute the recoverability bundle\footnote{Here, only forward walking is considered. Recovery from disturbances during backward or forward-to-backward walking could be achieved in a similar manner. If we take the backward walking for instance, all that is needed is to plan a proper sequence of apex states and integrate phase-space trajectories in a backward pattern, detect the PSM deviation via Eq.~(\ref{eq:simplifiedPSM}) and look up an offline DP policy table designed for backward walking.
}. 
An example of an estimated recoverability bundle is shown in Fig.~\ref{fig:RecoverabilitySet} (a). 
\subsection{Discrete Foot Placement Control}
\label{subsec:secondoptimization}
When the disturbance is large enough to move the state outside its recoverability bundle, the controller can not recover to the invariant bundle within a single stepping cycle. We propose to use the guard strategies discussed in Section~\ref{sec:manifold} for recovery. As a case study, let us use the position guard strategy and re-plan the foot placement for the next step as was illustrated in Fig.~\ref{fig:VelocityDisturbSet1} (c.2). In that strategy, it is assumed that we keep the previously planned apex velocity $\dot{x}_{{\rm apex}_{q+1}}$ for the next step. Hence, we analytically solve for a new foot placement based on the PSM of Eq.~(\ref{eq:simplifiedPSM}). Let us define the disturbed phase-space transition state as $(x_{\rm trans}, \dot{x}^{\rm dist}_{\rm trans})$. Equating the right hand side of Eq.~(\ref{eq:simplifiedPSM}) to zero, the re-planned sagittal foot placement $x^{\rm rep}_{{\rm foot}_{q+1}}$ is solved as,
\begin{align}\label{eq:replanfoot}
 x^{\rm rep}_{{\rm foot}_{q+1}} = x_{\rm trans} + \dfrac{1}{\omega} (\dot{x}^{{\rm dist} 2}_{\rm trans} - \dot{x}^2_{{\rm apex}_{q+1}})^{1/2}.
\end{align}
In forward walking, the condition $x^{\rm rep}_{{\rm foot}_{q+1}} > x_{\rm trans}$ holds, prompting us to ignore the solution with the negative square root. Note that if $\dot{x}_{{\rm apex}_{q+1}} = 0$, i.e., the robot is coming to a stop, Eq.~(\ref{eq:replanfoot}) becomes $x^{\rm rep}_{{\rm foot}_{q+1}}=x_{\rm trans}+\dot{x}^{\rm rep}_{\rm trans}/\omega$, which is equivalent to the Capture Point dynamics described in [\cite{englsberger2011bipedal}].

To evaluate the performance of this step re-planning method, we consider the six disturbances scenarios of Fig.~\ref{fig:High_Level_Robust}. The top three scenarios are recoverable using the DP-based continuous controller that we presented earlier. In the bottom three scenarios the disturbance occurs too close to the transition or is too large and therefore requires the foot placement re-planner described above to be executed. Once foot placements have been re-planned in the sagittal direction, lateral foot placements are re-planned using Algorithm~\ref{al:newtonraphson}. 

To conclude, the two-stage procedure discussed in this section constitutes the core process of our robust-optimal phase-space planning strategy. The combined locomotion planning procedure is shown in Algorithm~\ref{al:overall-planning} in the Appendix. We use the continuous control strategy first to better track the desired locomotion trajectories. As such, the continuous controller represents a servo process that is always on. When deviations are too large for recovery, we apply the foot placement re-planner.
The computational burden of our control process is low. The reason lies in that the DP-based continuous controller is designed offline to compute a table storing all possible policies for any admissible disturbance. Therefore, once disturbances are detected, the offline table is quickly looked at. The computation time for generating an offline policy table depends on various parameters: state grid resolution, control constraints and boundary phase-space states. For the policy table corresponding to the scenario in Fig. 11, it takes around $6$ hours on a standard laptop with $2.4$ GHz Intel Core $i7$. Once this table is generated offline, the online policy execution takes around $0.3-0.7$ ms for looking up optimal recovery states and control trajectories. In the case of the discrete foot placement re-planner, it is fast to compute due to its algebraic simplicity given in Eq. (\ref{eq:replanfoot}).
%
\subsection{End-to-End Phase Space Planning Procedure}
\label{subsec:summary}
An overall planning and control procedure is shown in Fig.~\ref{fig:OverallPlanningStructure} and Algorithm~\ref{al:overall-planning} in the Appendix.

\begin{figure}[h!]
 \centering
 \includegraphics[width=0.9\linewidth]{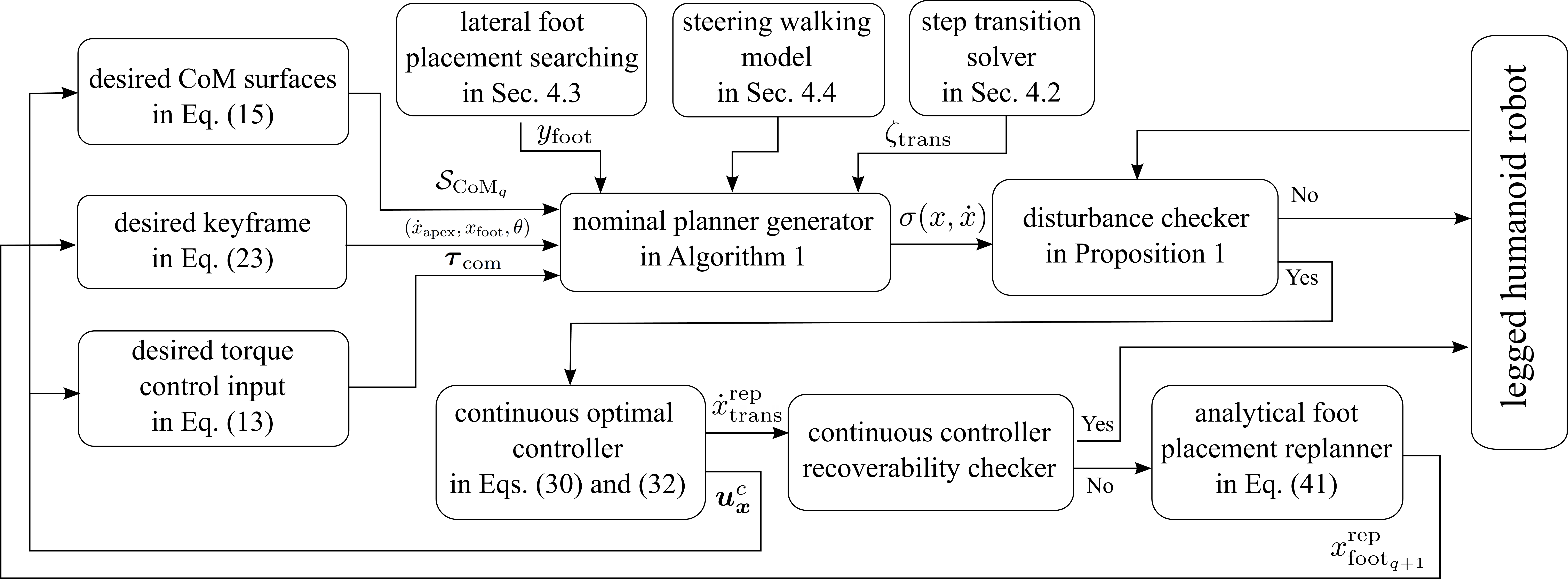}
 \caption{\captionsize End-to-end planning and control process. This diagram describes the pipeline for generating locomotion plans in the phase-space and robust control strategies. The locomotion designer first provides the following pieces of information: (i) desired CoM surfaces, (ii) nominal foot positions, (iii) desired keyframe states, and (iv) flywheel torque limits. Algorithm 1 produces the locomotion trajectories and the distance metric aided by the step transition solver, the lateral foot placement planner, and the steering model. A disturbance checker verifies if the current trajectories are within the invariant bundle. If they are not, the results stored on a table from the DP-based controller are utilized as a control policy and if this is not sufficient, new steps are re-planned in an online fashion.}
 \label{fig:OverallPlanningStructure}
\addtolength{\belowcaptionskip}{-5pt}
\end{figure}

\section{Results}
\label{sec:simulations}

We evaluate four types of locomotion scenarios for the purpose of (i) testing our planner's ability to handle walk on rough terrains and (ii) its robustness to large external disturbances. Our robot model uses six-dimensional free floating states, three degrees-of-freedom (DOFs) per leg and one DOF for torso pitch flexion/extension. Each leg has three actuated joints: hip abduction/adduction, hip flexion/extension and knee flexion/extension. We assume that each actuated joint has enough torque capability to achieve the planned motion. This model has a $0.55$ m torso length, a $0.56$ m hip width, a $0.6$ m thigh length and a $0.55$ m calf length, respectively. Given CoM trajectories and foot locations generated from the planner, we use inverse kinematics to obtain corresponding joint angles. On the other hand, because the feet contact transitions are discrete, we create smooth swinging trajectories to land the feet at the desired locations with the given time stamps. An accompanying video of dynamic walking over various terrain topologies is available here\footnote{\url{https://youtu.be/eSqQS4z7EYA}}. The source code is available online\footnote{\url{https://github.com/YeZhao/phase-space-planner-locomotion}}.
\subsection{Dynamic Walking over Rough Terrain}
We validate the versatility of our phase-space planning and control strategy by performing locomotion over terrains with random but known height variations. Three challenging terrains are tested as shown in Fig.~\ref{fig:DifferentTerrain}: (a) a terrain with convex steps, (b) a terrain with concave steps and (c) a terrain with inclined steps. The height variation $\Delta h_k$ of two consecutive steps is randomly generated based on the uniform distribution, 
\begin{equation}
\Delta h_k  = h_{k+1} - h_{k} \sim \textrm{Uniform} \left\{
     (-\Delta h_{\rm max}, -\Delta h_{\rm min})\cup
     (\Delta h_{\rm min}, \Delta h_{\rm max})\right\},
\end{equation}
where $h_{k}$ represents the height of the $k^{\rm th}$ step, $\Delta h_{\rm min} = 0.1$ m, $\Delta h_{\rm max} = 0.3$ m. A $10^\circ$ tilt angle is used for the slope of the steps. Foot placements are chosen a priori using simple kinematic rules and considering the length of the terrain steps. We design apex velocities according to a heuristic that accounts for terrain heights, and we use an average apex velocity of $0.6$ m/s. We choose CoM surfaces that conform to the terrains. We then apply the trajectory generation and controller pipeline procedures outlined in previous sections including the generation of trajectories based on Algorithm 1 and the search for step transitions based on the procedures of Section 5.2. 
\begin{figure*}[!th]
 \centering
 \vspace{-1in}
 \includegraphics[width=\linewidth]{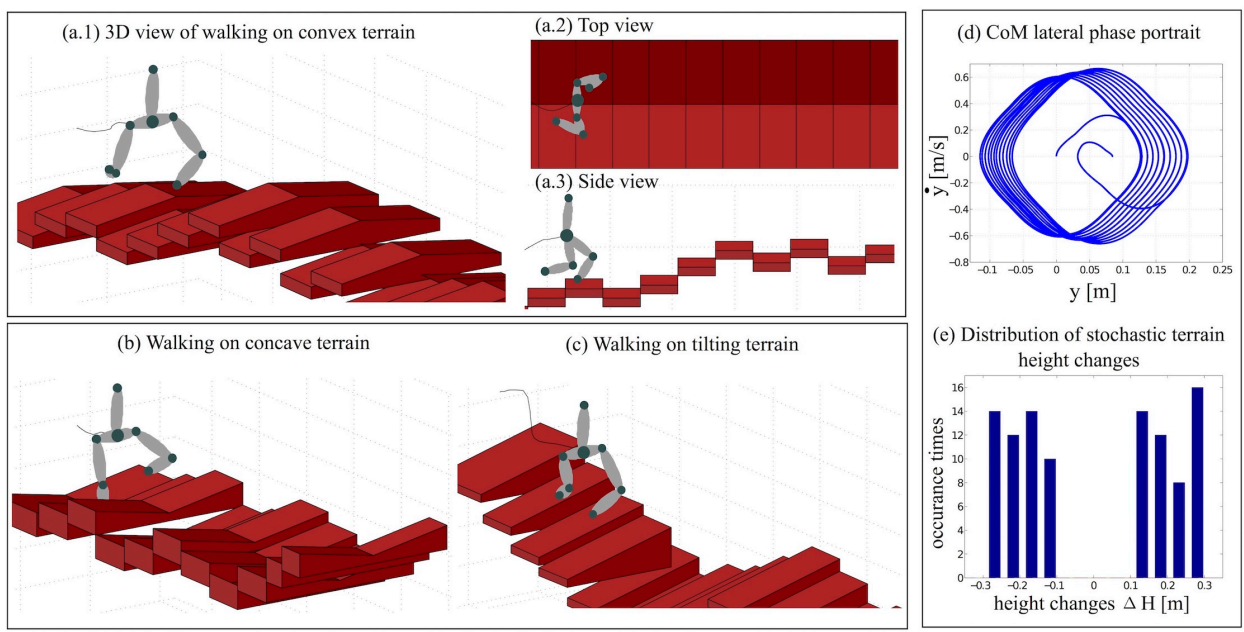}
 \vspace{-1in}
 \caption{\captionsize Traversing various rough terrains. The subfigures on the left block show dynamic locomotion over rough terrains with varying heights. The block on the lower right shows the height variation distribution over 100-steps.}
 \label{fig:DifferentTerrain}
\end{figure*}
Fig.~\ref{fig:DifferentTerrain} (a) shows a snapshot of bipedal walking on the terrain with convex steps.  Fig.~\ref{fig:DifferentTerrain} (b) and (c) show other types of rough terrains also tested in simulation over 100 steps. The lateral CoM phase portrait in Fig.~\ref{fig:DifferentTerrain} (d) shows stable walking over 25 steps.  The bar graph in Fig.~\ref{fig:DifferentTerrain} (e) shows the distribution of the height of the randomly generated terrain. 
\subsection{Dynamic Walking under External Disturbances}


\subsubsection{Recovery from Disturbance on the Sagittal Plane}

We first make the robot walk on a terrain based on the planning algorithms described in the paper. We then apply a pushing force in the sagittal direction, which causes an instantaneous velocity jump as shown in Fig.~\ref{fig:SagittalDisturbance} (a). This disturbance is quite large such that the robot's state cannot recover in one single step to its nominal PSM using the proposed optimal controller. Thus, the foot location re-planning strategy previously described is executed. The dashed line in Fig.~\ref{fig:SagittalDisturbance} (a) represents the original phase-space trajectory while the solid line represents the re-planned trajectory. Also, instead of an instantaneous step transition, a multi-contact transition is used as described in the Appendix~\ref{sec:multicontact}.
\begin{figure*}[!t]
 \centering
   \includegraphics[width=0.75\linewidth]{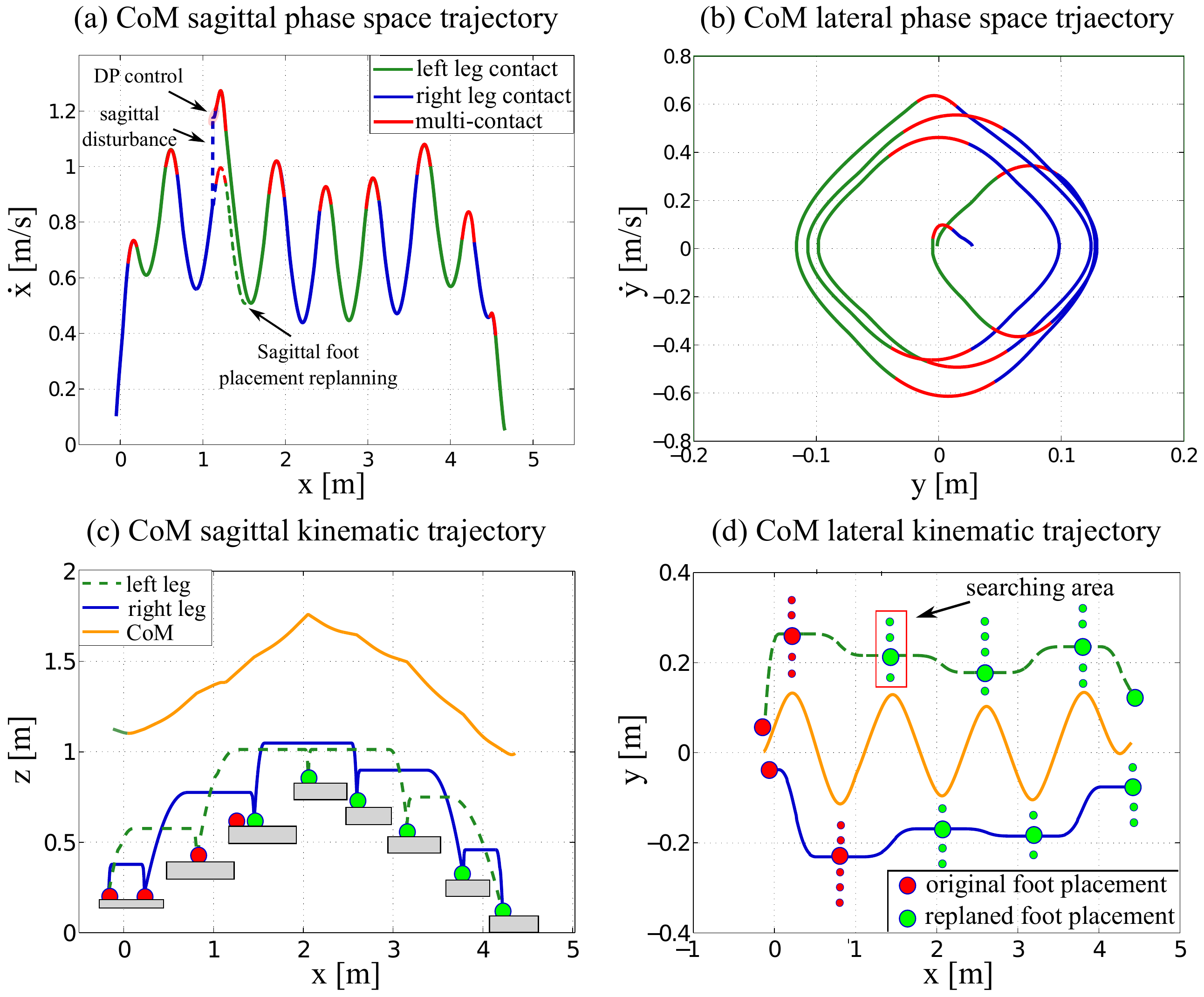}
 \caption{\captionsize Recovery from sagittal plane disturbance. To recover from a sagittal push with a $0.4$ m/s CoM velocity jump, the planner uses both DP continuous control and discrete foot placement re-planning in a sequential manner. $\reddot$ denotes the pre-defined foot placement before the disturbance while $\greendot$ denotes the re-planned foot placement after the disturbance.
} 
 \label{fig:SagittalDisturbance}
\end{figure*}

\begin{figure*}[!h]
 \centering
   \includegraphics[width=0.95\linewidth]{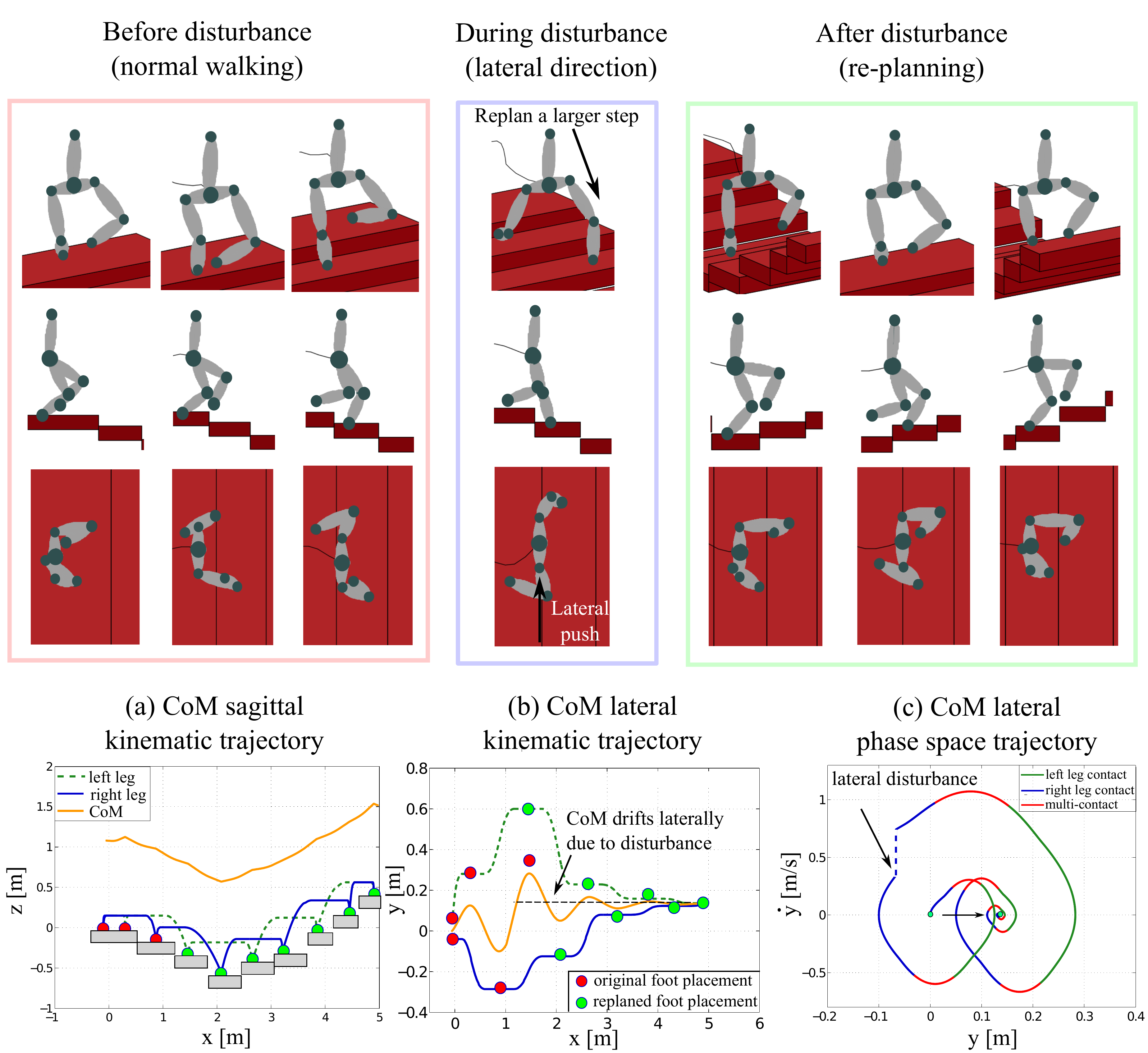}
 \caption{\captionsize Recovery from lateral plane disturbance. When the lateral disturbance occurs, the foot placement is re-planned to avoid falling down. Correspondingly, the CoM trajectory drifts to the left side as shown in (b) while a CoM velocity jump appears in (c). After the disturbance, the CoM trajectory is re-generated based on this new lateral foot placement to achieve stable walking.}
 \label{fig:SequantialSnapshot}
\end{figure*}

\begin{figure*}[!bht]
 \centering
  \vspace{-1in}
   \includegraphics[width=0.9\linewidth]{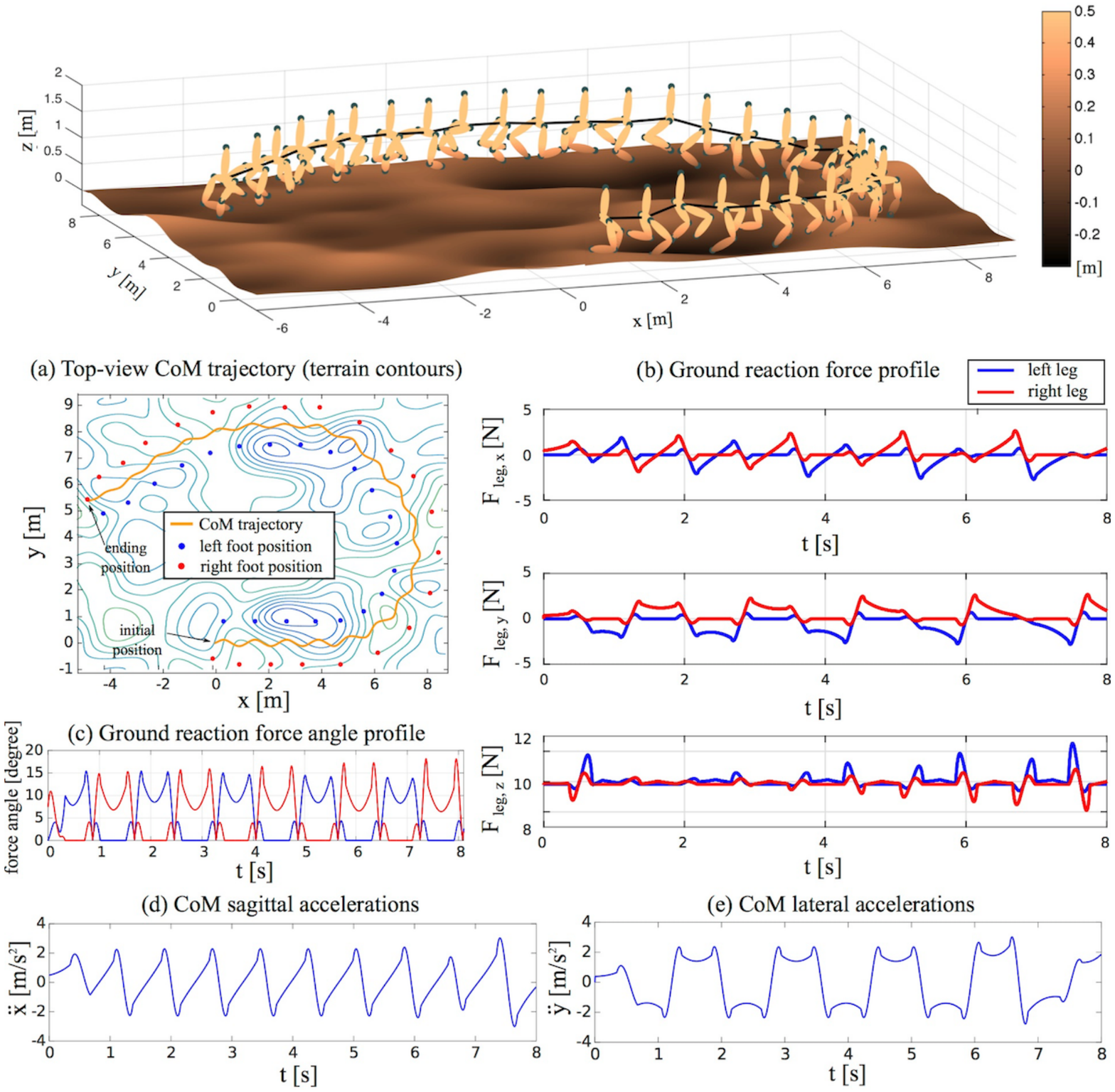}
   \vspace{-1in}
 \caption{\captionsize Circular walking over a random rough terrain. The figure above shows dynamic walking while steering in 3D. The terrain height randomly varies within [-0.24, 0.3] m. Subfigure (a) shows the top view of the CoM trajectory and the foot locations given the terrain contour. (b) shows each leg's ground reaction forces in local coordinate. The reaction forces at step transitions are smooth thanks to the dual-contact control policy. (c) shows that the angle of reaction forces is constrained within the $45^\circ$ friction cone. (d) and (e) show smooth CoM sagittal and lateral accelerations.}
 \label{fig:legforceProfile}
\end{figure*}

\subsubsection{Recovery from Disturbance on the Lateral Plane}

For this simulation, we make the robot once more walk on the rough terrain according to its nominal plan. Then, in its third step, we apply a lateral disturbance as shown in Fig.~\ref{fig:SequantialSnapshot} (b) and (c). To deal with this disturbance, a new lateral foot placement is re-planned according to Algorithm 2. 
\subsection{Circular Walking over a Rough Terrain}
\label{subsec:circularwalking}
Circular walking over a random rough terrain is shown in Fig.~\ref{fig:legforceProfile}. We use this example to validate the steering capability of our planner. The walking direction is defined by the heading angle $\theta$ shown in Def.~\ref{def:snpg}. The planning process is performed in the robot's local coordinate with respect to the heading angle. We then apply a local-to-global transformation. Also, this simulation validates the steering direction model introduced in Section~\ref{subsec:steeringDirec} and smoothness of the leg force profile by using multi-contact dynamics in Section~\ref{subsec:dualContact}.

To validate the applicability of our phase-space planner for more challenging terrain topologies, we test dynamic bouncing over disjoint terrain, and a preliminary study is shown in Appendix~\ref{sec:bouncingOverTerrain} and Fig.~\ref{fig:wedgejumping}.

\section{Discussions and Conclusions}
\label{sec:Discussion}

The main focus of this paper has been on addressing the needs for planning non-periodic bipedal locomotion behaviors. These types of behaviors arise in situations where terrains are non-flat, extremely rough, or even disjointed. The majority of bipedal locomotion methodologies have been historically focused on flat terrain or mildly rough terrain locomotion behaviors. Some of them are making their way into dynamically climbing stairs or inclined terrains. Additionally, Raibert experimented with planar hopping locomotion over rough terrains in the middle-80s. In contrast, our effort is centered around the goals of (i) providing metrics of robustness in rough terrain for robust control of the locomotion behaviors, (ii) generalizing gaits to any types of surfaces, (iii) providing formal tools to study planning, robustness, and reachability of the non-periodic gaits, and (iv) demonstrating the ability of our method to deal with large external disturbances.

In the nominal trajectory generation process of Algorithm 1, we assume a sequence of foot placements given a priori. There exist optimization methods determining discrete foot placements and therefore it has not been a focus of ours to explore this issue. As to the apex state design, this study uses a heuristic related to the terrain heights. More recently, we have proposed an advanced keyframe decision maker based on temporal-logic-based formal methods [\cite{zhao2016high}]. Choosing apex states in our planner strategy is not only a mathematical convenience but also enables designers to plan non-periodic apex velocities which is related to the walking speed. Apex states are a type of salient points more natural to regulate the walking speed.

Our choice of providing a priori CoM surfaces can be traced back to our initial design methodology for this line of work. Initially, we extracted the CoM trajectory from capturing human walking over rough terrains [\cite{zhao2012three, zhao2016humanwalking}]. We observed that the CoM trajectory approximately conforms to the terrain height and slope. This observation prompted us to use the following three-step procedure: 1) design the piecewise-linear CoM plane approximately in parallel with the terrain slope; 2) design heuristics to adjust the CoM plane sagittal and lateral slopes (i.e., tilting angles) according to the walking phases (step acceleration or deceleration phases). 

Our method could use generic CoM surfaces like in [\cite{morisawa2005pattern}], but the dynamics of Eq. (\ref{eq:accel}) would become more complicated. In that case, deriving an analytical phase-space distance metric should be done based on numerical approximation and curve fitting (i.e. NURBS). To avoid this added complexity, we chose to rely on the piecewise-linear CoM surface  model and smooth it out using multi-contact dynamics. We believe that our current method presented in this paper is sufficient to achieve smooth locomotion without using a more complicated metric.

Our planner is based on a simplified inverted pendulum model, which ignores swing leg dynamics. However, this type of dynamics can significantly affect the actual motion tracking performance. In the future, we will explore more sophisticated models that include this type of dynamics. In the dynamic programming approach of Eq. (\ref{eq:optimization-1}), we only constrain the pitch torque while the pitch angle does not have limits. The focus of the manuscript so far has been on the generation of the trajectories and on outlining a robust control approach. However, for real implementation users need to incorporate the dynamics of the flywheel to constraint the torso pitch's range of motion.

Zero lateral velocity at the sagittal apex is a simple heuristic that prevents the center of mass from drifting away from the local frame. It is important to remark that this heuristic is specified in the local frame, and therefore it accounts for the steering angle. As such, when considering the global frame, the lateral velocity at the end of each step is effectively non-zero.

The lateral foot placement is an output of the planner. Each time a new sagittal foot placement is re-planned in an online fashion, the lateral foot placement has also to be re-planned. We view this online re-planning stage, described in Subsections~\ref{subsec:secondoptimization} and \ref{subsec:searchlateral}, as a controller which is a part of the runtime methodology that should be implemented in real experiments.


Overall, future extensions of this work include: (i) Experimental validations, where additional constraints, modeling, pose estimation, and kinematic errors, among other problems will need to be considered. (ii) Proposing a more realistic robot model that incorporates swing leg dynamics. (iii) Devising more sophisticated trajectory optimization methods to design optimal motion trajectories [\cite{hereid20163d, pardo2017hybrid, xi2016selecting}] and even incorporating contact forces into a larger optimization problem [\cite{posa2014direct, mordatch2012discovery, dai2014whole}]. (iv) Proposing a realistic terrain perception model that does not assume perfect terrain information. In that case, we can design more realistic robust controllers according to terrain disturbances. Robust recovery studies in the presence of terrain perturbations in [\cite{piovan2016approximation, dai2012optimizing, griffin2016nonholonomic}] are valuable references.

 \begin{funding}
 This work was supported by the Office of Naval Research, ONR Grant [grant \#N000141512507], NASA Johnson Space Center, NSF/NASA NRI Grant [grant \#NNX12AM03G], and NSF CPS Synergy Grant [grant \#1239136].
 \end{funding}

\begin{appendix_sec}
\appendix
\label{appen}

\section{Index to Multimedia Extensions}

\begin{tabular}{l*{6}{c}r}
Extension  & Type & Description \\
\hline
& & Demonstrations of four locomotion simulations including (1) seven-step  \\
1 & Video & phase-space motion planning; (2) dynamic walking over rough terrains; (3) dynamic \\
& & walking under external disturbances; (4) bouncing maneuver over disjointed terrain.\\
\hline
\end{tabular}
\section{Phase-Space Manifold}
\label{sec:PSMBasics}
The desired behavior of the outputs lie in manifolds $\mathcal{M}_i$ as shown in Eq. (\ref{eq:surface}).  Here we present a brief review of space curves and surfaces that relate to the \cool{phase-space manifold} (\cool{PSM}) and present a Riemannian geometry metric that can be generalized to this family of problems. 
 
The trajectory of the center-of-mass (CoM) of the robot is a space curve in 3D, i.e., $\boldsymbol{p}_{\rm CoM}=(x,y,z)^T\onR{3}$.  Also, for a particular output $y_i$ (i.e., one element of Eq.~(\mbox{\ref{eq:plantStates}}b)), if we consider the case of an output-task with relative order $r_i=3$, the manifold $\mathcal{M}_i\onR{r_i}$ is a space curve $\mathcal{C}_i$ in Euclidean three-dimensional space (see Fig. \ref{fig:spacecurve}).   We assume that the curve is parametrized by an arc-length parameter $\zeta$ that we refer to as the phase progression variable in Def.~\ref{def:progVariable}. Hence the position vector $\boldsymbol{\rho}_i$ of any point on the curve can be defined by specifying the value of $\zeta$,
\begin{equation}  \label{(2)}
 \boldsymbol{\rho}_i(\zeta) = \sum_{k=1}^{r_i}\xi_k(\zeta){\bf E}_k
           = \xi_1(\zeta){\bf E}_1 + \xi_2(\zeta){\bf E}_2 + \xi_3(\zeta){\bf E}_3. 
\end{equation}
\begin{figure*}[t]
 \centering
 \includegraphics[width=0.6\linewidth]{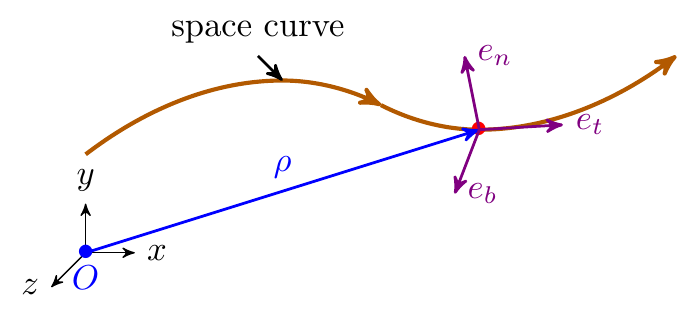}
 \caption{\captionsize A space curve showing the evolution of the Frenet triad.}
 \label{fig:spacecurve}
\end{figure*}
where ${\bf E}_k$ is the unit vector in the $k$-axis of the Euclidean space and $\xi_k(\zeta)$ is the projection coordinate of $\boldsymbol{\rho}$ on ${\bf E}_k$. A unit tangent vector ${\bf e}_t$ to the curve can also be defined,
%
\begin{equation} \label{(3)}
 {\bf e}_t = \dfrac{\partial \boldsymbol{\rho}_i}{\partial \zeta}. 
 \end{equation}
The derivative of this vector defines the curvature $\kappa$ and the unit normal vector ${\bf e}_n$,
%
\begin{equation} \label{(4)} 
 \dfrac{\partial {\bf e}_t}{\partial \zeta} = \kappa {\bf e}_n, 
\quad {\rm where} \quad
    \kappa = \left|\!\left| \dfrac{\partial {\bf e}_t}{\partial \zeta}\right|\!\right|. 
\end{equation}
In the case of a space surface (where $\boldsymbol{\rho}_i$ belongs to a manifold $\mathcal{M}_i$ in the output phase-space), instead of a vector ${\bf e}_t$, we have a tangent  manifold, denoted by $T_{\mathcal{M}_i}$.  The tangent space at any point can be mapped to the vector $\boldsymbol{\chi}_i\in \mathbb{R}^{r_i-1}$ that spans $T_{\mathcal{M}_i}$.  Without loss of generality, the actual motion is a specific line in space curve $\mathcal{M}_i$.  The tangent vector in the manifold $\mathcal{M}_i$ is $\boldsymbol{e}_\zeta=({\bf e}_i)_t$ while the cotangent vector in the manifold is ${\bf e}_\sigma=({\bf e}_i)_n$. ${\bf e}_\sigma$ denotes the normal deviation distance from the surface $\sigma_i$. For $r_i=3$, the binormal vector ${\bf e}_b$ is orthogonal to ${\bf e}_t$ and ${\bf e}_n$. These three vectors are called the Frenet space. These three Frenet frame vectors are proportional to the first three derivatives of the curve $\boldsymbol{\rho}$, as a benefit of taking the arc length $\zeta$ as the parameter.

In disturbance-free cases, the system will remain in the manifold if it starts on it.  It can be considered as the zero dynamics of the surface deviation $\sigma_i$.  When disturbance occurs, the state may escape the manifold and the controller should bring it back for recovery. To define a metric on the manifold itself and normal to it, we use Riemannian Geometry. In general, we treat each manifold $\mathcal{M}_i$ in Eq.~(\ref{eq:surface}) of the task-space $i^{\rm th}$-coordinate as independent from each other. The actual task manifold $\mathcal{M}$ is the intersection of all $\mathcal{M}_i$ manifolds,
\begin{equation} \label{(7)}
 \mathcal{M} = \bigcap_{i}^{} \mathcal{M}_i. 
\end{equation}
This manifold also has a tangent manifold $T_{\mathcal{M}} \in \mathbb{R}^r$, where $r=\sum_i{r_i}$.  Each manifold $\mathcal{M}_i$, separately have a null-space (cotangent manifold $T^{*}_{\mathcal{M}_i}$) and their intersection is the task cotangent manifold $T^{*}_{\mathcal{M}}$. 

In relation to the manifolds presented in Section~\ref{sec:model}, the relationship between CoM configuration manifolds and CoM phase-space manifolds can be visualized via the set diagram of Fig.~\ref{fig:surfacemanifold}.
\section{Derivation of The PIPM Dynamics in Eq.~(\ref{eq:accel})}
\label{sec:derivationEq}
We first expand Eq.~(\ref{eq:balance-3}) as follows
\begin{align}\label{eq:xzorigin}
(z  - z_{{\rm foot}_q}) \cdot m \ddot{x} &= (x - x_{{\rm foot}_q}) \cdot m(\ddot{z} + g ) - \tau_y,\\\label{eq:yzorigin}
-(z  - z_{{\rm foot}_q})\cdot m\ddot{y} &= ( y - y_{{\rm foot}_q}) \cdot m (\ddot{z} + g) - \tau_x,\\\label{eq:xyorigin}
(x  - x_{{\rm foot}_q})\cdot m\ddot{y} &= ( y - y_{{\rm foot}_q}) \cdot m \ddot{x} - \tau_z.
\end{align}
\begin{figure*}[!t]
 \centering
 \includegraphics[width=0.8\linewidth]{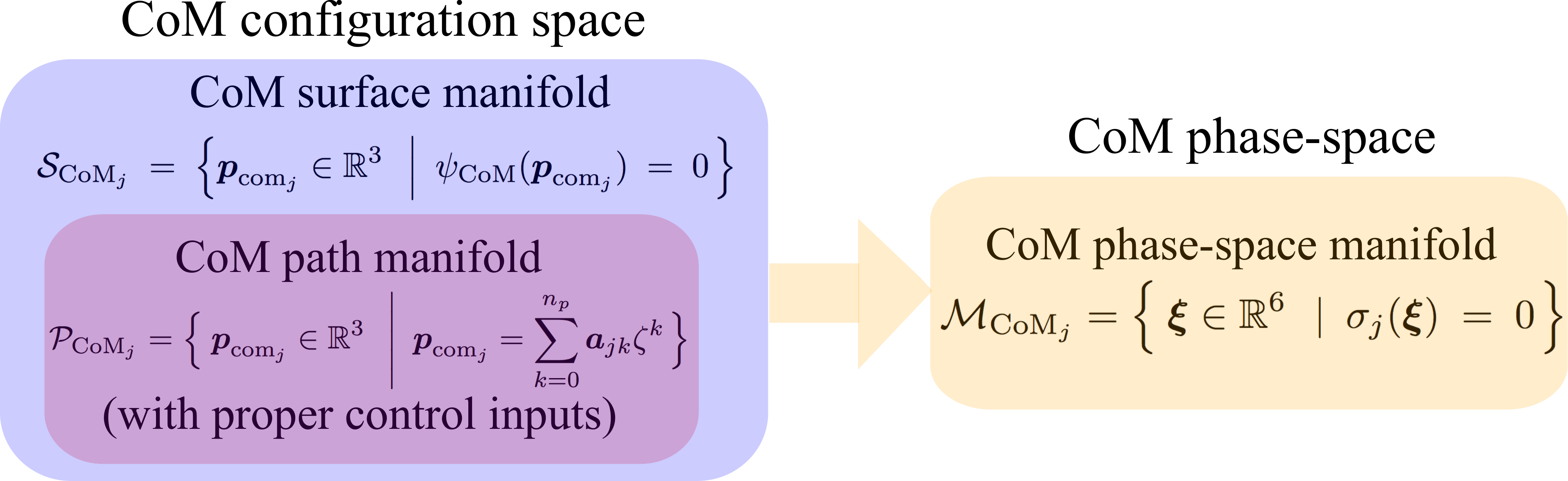}
 \caption{\captionsize Relationship between the CoM configuration space manifolds and the CoM phase-space manifold.}
 \label{fig:surfacemanifold}
\end{figure*}
\begin{figure*}[t]
 \centering
 \includegraphics[width=0.95\linewidth]{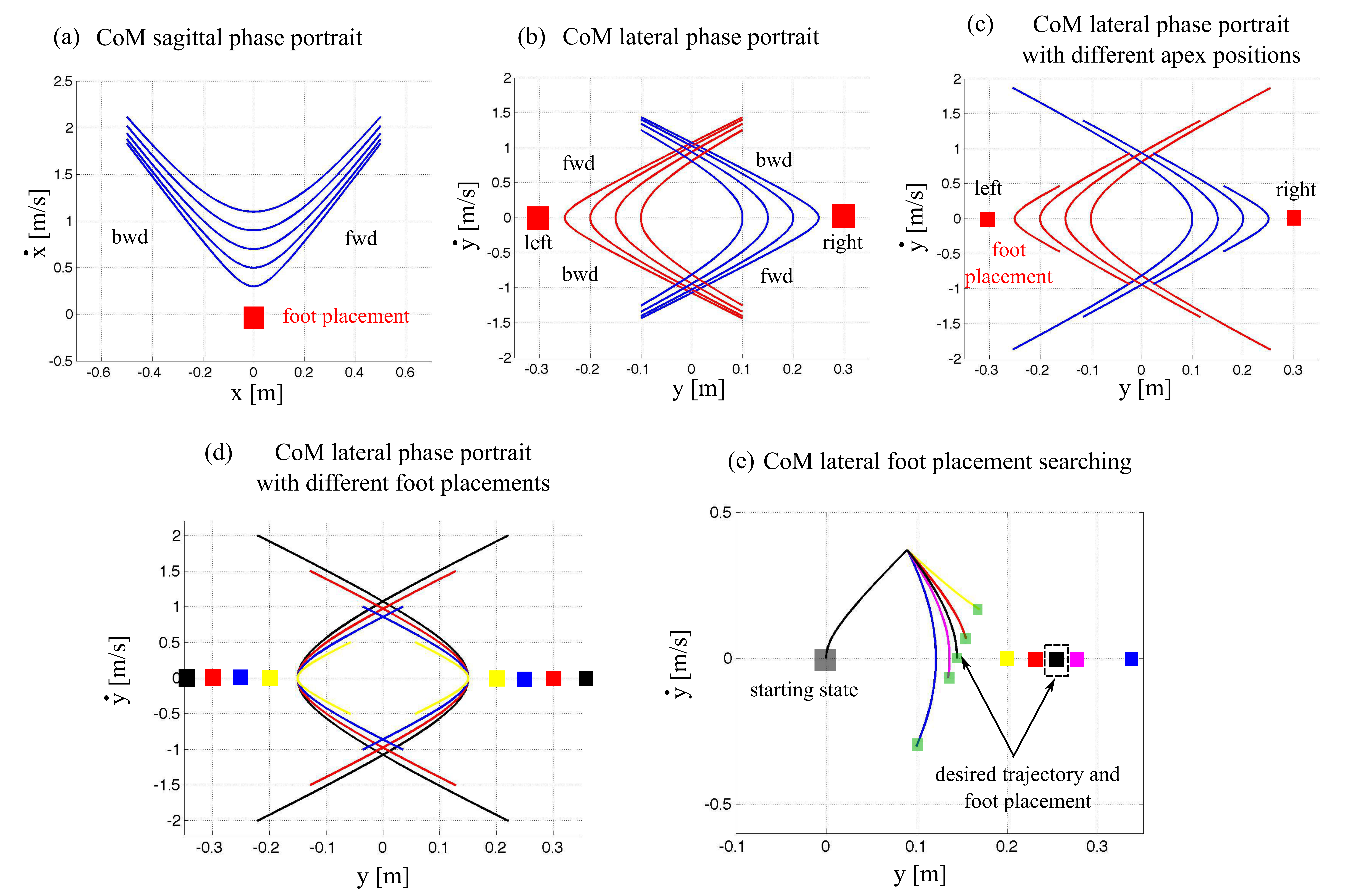}
 \caption{\captionsize Sagittal and lateral phase-space dynamics of one walking step. These phase-space trajectories are generated based on the prismatic
   inverted pendulum dynamics in Eq.~(\ref{eq:accel}) and phase-space manifold in Eq.~(\ref{eq:simplifiedPSM}). The phase portraits (a) and (b) correspond to the sagittal and
   lateral CoM phase behaviors given desired foot contact
   locations (red boxes), a desired CoM surface of motion,
   and initial position and velocity conditions. If the same duration is guaranteed for each trajectory, we can derive two different trajectories shown in (c) and (d) based on different initial conditions.  (c) shows lateral CoM behaviors given a fixed lateral foot placement and varying initial lateral position conditions. (d) corresponds to CoM
   trajectories derived given varying lateral foot placements and a fixed
   initial condition. In (e), we analyze lateral CoM trajectories of two consecutive steps with varying lateral foot placements. As the foot placement moves further apart, the acceleration becomes larger and the CoM position transverses (at the sagittal apex) less in the y direction.}
 \label{fig:Numerical_Integration}
\end{figure*}
Given the 3D surface in Eq.~(\ref{eq:linear2DSurface}), we differentiate it twice and obtain
\begin{align}\label{eq:accel_surface}
\ddot{z} = a_q \ddot{x} + b_q \ddot{y}.
\end{align}
Substituting Eq.~(\ref{eq:accel_surface}) into Eqs. (\ref{eq:xzorigin}) and (\ref{eq:yzorigin}), we have
\begin{align}\label{eq:xdynamics}
(a_q x + b_q y + c_q - z_{{\rm foot}_q})\ddot{x} - ( x - x_{{\rm foot}_q}) (a_q \ddot{x} + b_q \ddot{y}
  + g) + \tau_y/m &= 0,\\\label{eq:ydynamics}
 (a_q x + b_q y + c_q - z_{{\rm foot}_q}) \ddot{y} - ( y - y_{{\rm foot}_q}) (a_q \ddot{x} + b_q \ddot{y}
  + g) + \tau_x/m &= 0.
\end{align}
Combining Eqs. (\ref{eq:xyorigin}) and (\ref{eq:xdynamics}), we obtain
\begin{equation}\label{eq:xdynamicsfinal}
\ddot{x} = \frac{( x - x_{{\rm foot}_q}) g}{a_q x_{{\rm foot}_q} + b_q y_{{\rm foot}_q} + c_q - z_{{\rm foot}_q}} - \frac{\tau_y + b_q \tau_z}{m(a_q x_{{\rm foot}_q} + b_q y_{{\rm foot}_q} + c_q - z_{{\rm foot}_q})}.
\end{equation}
Combining Eqs. (\ref{eq:xyorigin}) and (\ref{eq:ydynamics}), we obtain
\begin{equation}\label{eq:ydynamicsfinal}
\ddot{y} = \frac{( y - y_{{\rm foot}_q}) g}{a_q x_{{\rm foot}_q} + b_q y_{{\rm foot}_q} + c_q - z_{{\rm foot}_q}} - \frac{\tau_x + a_q \tau_z}{m(a_q x_{{\rm foot}_q} + b_q y_{{\rm foot}_q} + c_q - z_{{\rm foot}_q})}.
\end{equation}
By combining Eqs.~(\ref{eq:accel_surface}), (\ref{eq:xdynamicsfinal}) and (\ref{eq:ydynamicsfinal}), $\ddot{z}$ can be derived accordingly. By defining a phase-space asymptotic slope $\omega_q$ as shown in Eq.~(\ref{eq:CoMaccelRate}), we can obtain Eq.~(\ref{eq:accel}) for the prismatic inverted pendulum based control system.

\section{Robust Hybrid Automaton}
\label{sec:RHAutomaton}
Mathematical notations for the robust hybrid automaton of Section~\ref{subsec:RHLA}:

\begin{itemize}
 \item $\zeta$ is the phase-space progression variable;

 \item $\mathcal{Q}$ is the set of discrete states;

 \item $\mathcal{X}$ is the set of continuous states. The system state is augmented to $\boldsymbol{s} \coloneqq \zeta \times \mathcal{Q} \times \mathcal{X}$ in a hybrid state space;

 \item $\mathcal{U} \coloneqq \{\boldsymbol{u}_q, q \in \mathcal{Q}\}$, is the set of control inputs. $\mathcal{U} = \{\boldsymbol{u}_c\} \cup \{\boldsymbol{u}_d\}$ where $\boldsymbol{u}_c, \boldsymbol{u}_d$ are continuous and discrete control inputs, respectively;

 \item $\mathcal{W}$ is the set of disturbances;

 \item $\mathcal{F}$ is the vector field, with $\mathcal{F}: \zeta \times \mathcal{Q} \times \mathcal{X} \times \mathcal{U} \times \mathcal{W} \rightarrow T_{\mathcal{X}}$, where $T_{\mathcal{X}}$ is tangent bundle of $\mathcal{X}$;
 
\item $\mathcal{I} \coloneqq \zeta \times \mathcal{Q} \times \mathcal{X}$, is the initial condition;

\item $\mathcal{D}(q): \mathcal{Q} \rightarrow 2^\mathcal{X}, q \in \mathcal{Q}$, is the domain\footnote{$2^\mathcal{X}$ represents the power set (all the subsets) of $\mathcal{X}$.};

\item $\mathcal{R} \coloneqq \{\mathcal{R}_q, q \in \mathcal{Q}\}$, is the collection of recoverability bundles;

\item $\mathcal{B} \coloneqq \{\mathcal{B}_q, q \in \mathcal{Q}\}$, is the collection of invariant bundles;

\item $\mathcal{E}(q, q+1) \coloneqq \mathcal{Q}\times\mathcal{Q}$, is the edge;

\item $\mathcal{G}(q, q+1): \mathcal{Q}\times\mathcal{Q}\rightarrow2^{\mathcal{X}_q}$ is the \cools{guard}, which is abbreviated as $\mathcal{G}_{q \rightarrow q+1}$; $\mathcal{G}(q, q+1) = \cup_{\tau,\mu}^{~}\mathcal{G}^{[\tau]}_{\mu}$, where $\tau$ and $\mu$ denote transition types defined in the paragraph below Table~\ref{table:SwitchingGuard}\footnote{More details are provided in [\cite{branicky1998unified}].}.

\item $\mathcal{T}(q, q+1): \mathcal{Q}\times\mathcal{Q}\rightarrow2^{\mathcal{X}_{q+1}}$ is the transition termination set;

\item $\Delta^{[\tau]}_{\mu(q \rightarrow q+1)}(\boldsymbol{s}_q^{-}, \boldsymbol{u}_{q}^{-}, w_d^{-})$, is the transition map.

\end{itemize}
\noindent Based on the automaton above, the hybrid system can be represented by
\begin{equation}\label{eq:hybridSystemTransitions}
\Sigma_{q} :
\left\{ \begin{array}{ll} 
   \left\{ 
    \boldsymbol{\mathcal{F}}_{\boldsymbol{x}}^{+},\boldsymbol{u}_{q+1}^+, \boldsymbol{x}_{q+1}^+
   \right\} 
   \;\;\leftarrow\;\;
   \Delta^{[\tau]}_{\mu(q \rightarrow q+1)}
            (\boldsymbol{s}_q^{-}, \boldsymbol{u}_{q}^{-}, w_d^{-}),
   &{\rm if}\;\;(\boldsymbol{s}_q^{-}, \boldsymbol{u}_{q}^{-}, w_d^{-}) 
    \in\mathcal{G}^{[\tau]}_{\mu}(q,q+1)
   \\ & \\
  \dot{\boldsymbol{x}}_{q} = 
   \boldsymbol{\mathcal{F}}_{\boldsymbol{x}}(\zeta, q,\boldsymbol{x}_{q}, \boldsymbol{u}_{q}, w_d),
   &
   {\rm otherwise}
\end{array}\right.
\end{equation}
where $\boldsymbol{s}_q = (\zeta_q,q,\boldsymbol{x}_{q}^T)^T$ is the hybrid automaton state. This automaton has non-periodic orbits, since our planning focuses on irregular terrain locomotion. A directed diagram of this non-periodic automaton is shown in Fig.~\ref{fig:Automaton}.
\section{Multi-Contact Maneuvers}
\label{sec:multicontact}

The objective of this section is to incorporate multi-contact transitions into our gait planner to achieve more natural motions. To achieve this capability, we fit a polynomial function with a smooth transition behavior between single contact phase curves at the transition points. For this process, desired boundary values of position, velocity and acceleration are given by the gait designer. It is necessary to also take into account time constraints to guarantee the synchronization of the sagittal and lateral behaviors. Boundary and timing conditions allow us to calculate the coefficients of the polynomials. More mathematical details of this approach can be found in [\cite{zhao2012three}]. In this study, a multi-contact transition curve is created utilizing $25\%$ of the total time slot for a given step. This is consistent with the timing that we have observed in human walking. This percentage is a parameter adjustable by the designer as demanded. The results are shown in Fig.~\ref{fig:Multicontact}.

\begin{figure}[!h]
 \centering
   \includegraphics[width=0.9\linewidth]{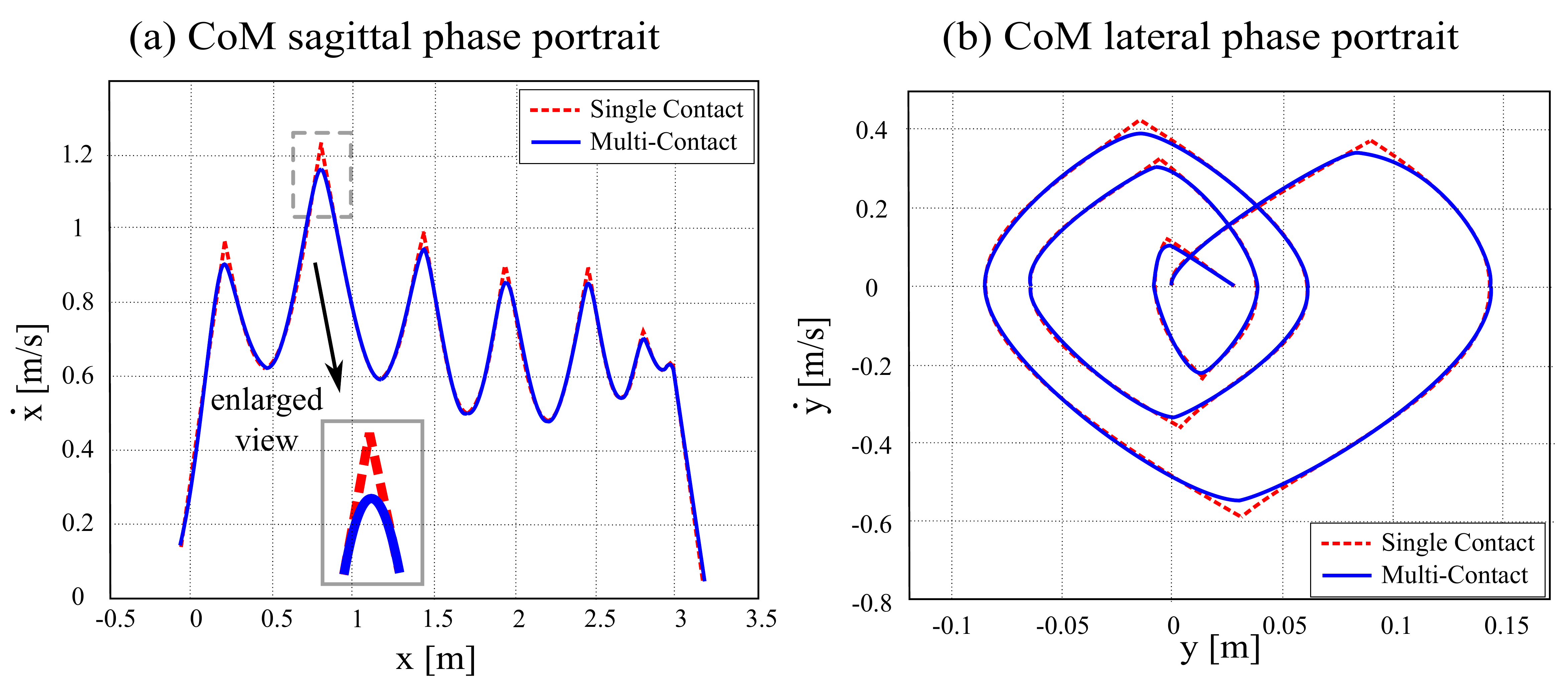}
 \caption{\captionsize Integration of multi-contact transition phases. The subfigures (a) and (b) are similar to their counterparts of Fig.~\ref{fig:3Dsinglecontact} but with an additional multi-contact phase. By using $5^{\rm th}$ order polynomials and guaranteeing continuity with the existing curves, we get the polynomial parameters for curve fitting.}
 \label{fig:Multicontact}
\end{figure}

\section{Proof of Phase-Space Tangent Manifold}
\label{sec:PSMDerivation}

In the nominal control case, the torques $\tau_y$ and $\tau_z$ of Eq.~(\ref{eq:dynx}) are zero. For this case, the sagittal inverted pendulum dynamics are simple, $\ddot{x} = \omega^2(x - x_{\rm foot})$, where $\omega$ is constant for a given step. Since the foot placement $x_{\rm foot}$ is also constant over the step, then $\ddot{x}_{\rm foot} = \dot{x}_{\rm foot} = 0$. Therefore the previous equation is equivalent to $\ddot{x} - \ddot{x}_{\rm foot} = \omega^2 (x - x_{\rm foot})$. Let us define a transformation $\tilde{x} = x - x_{\rm foot}$. We can then write $\ddot{\tilde{x}} = \omega^2 \tilde{x}$. Using Laplace transformations, we have $s^2 \tilde{x}(s) - \tilde{x}_0 - s \dot{\tilde{x}}_0 = \omega^2 \tilde{x}(s)$. Based on this, we get
\begin{equation}
\tilde{x}(t) = \mathscr{L}^{-1}\{\dfrac{\tilde{x}_0 + s\dot{\tilde{x}}_0}{s^2 - \omega^2}\}.
\end{equation}
Solving the equation above, we can derive an analytical solution
\begin{align}\label{eq:acceleration3}
\tilde{x}(t) = \dfrac{\tilde{x}_0 (e^{\omega t} + e^{-\omega t})}{2} + \dfrac{\dot{\tilde{x}}_0
(e^{\omega t} - e^{-\omega t})}{2 \omega} = \tilde{x}_0 {\rm cosh}(\omega t) + \dfrac{1}{\omega} \dot{\tilde{x}}_0 {\rm sinh}(\omega t),
\end{align}
and by taking its derivative, we get
\begin{align}\label{eq:acceleration3-0}
\dot{\tilde{x}}(t) = \omega \tilde{x}_0 {\rm sinh}(\omega t) + \dot{\tilde{x}}_0 {\rm cosh}(\omega t).
\end{align}
These two equations can be further expressed as
\begin{eqnarray}\label{eq: closedformposition}
x(t) &=& (x_0 - x_{\rm foot}) {\rm cosh}(\omega t) + \dfrac{1}{\omega} \dot{x}_0 {\rm sinh}(\omega t) + x_{\rm foot}
\\
\label{eq: closedformvelocity}
\dot{x}(t) &=& \omega (x_0 - x_{\rm foot}) {\rm sinh}(\omega t) + \dot{x}_0 {\rm cosh}(\omega t)
\end{eqnarray}
Now we have the following state space formulation
\begin{eqnarray}\nonumber
\begin{pmatrix}
x(t) - x_{\rm foot}\\
\dot{x}(t)\\
\end{pmatrix}
&=& \begin{pmatrix}
x_0 - x_{\rm foot} & \dot{x}_0 / \omega\\
\dot{x}_0 & \omega (x_0 - x_{\rm foot})\\
\end{pmatrix}
\begin{pmatrix}
{\rm cosh}(\omega t)\\
{\rm sinh}(\omega t)\\
\end{pmatrix} \qquad
\end{eqnarray}\nonumber
which implies
\begin{eqnarray}\label{eq: sinhcosh}
\begin{pmatrix}
{\rm cosh}(\omega t)\\
{\rm sinh}(\omega t)\\
\end{pmatrix}
&=& \dfrac{1}{\omega (x_0 - x_{\rm foot})^2 - \dot{x}^2_0/\omega}
\begin{pmatrix}
\omega (x_0 - x_{\rm foot}) & -\dot{x}_0 / \omega\\
-\dot{x}_0 & x_0 - x_{\rm foot}\\
\end{pmatrix}
\begin{pmatrix}
x - x_{\rm foot}\\
\dot{x}\\
\end{pmatrix}
\end{eqnarray}
since ${\rm cosh}^2(x) - {\rm sinh}^2(x) = 1$, we get 
\begin{align}\label{eq:surface0}
\big(\omega(x_0 - x_{\rm foot}) (x - x_{\rm foot}) - \dot{x}_0 \dot{x}/\omega\big)^2 - \big(- \dot{x}_0 (x - x_{\rm foot}) 
 + \dot{x} (x_0 - x_{\rm foot})\big)^2 = \big(\omega (x_0 - x_{\rm foot})^2 - \dot{x}^2_0/\omega\big)^2
\end{align}
After expanding the square terms and moving all terms to one side, we obtain 
\begin{align}
(x_0 - x_{\rm foot}) ^2\big(2\dot{x}^2_0 - \dot{x}^2 + \omega^2 (x - x_0) (x + x_0 - 2x_{\rm foot})\big) - \dot{x}^2_0 (x - x_{\rm foot})^2 + \dot{x}^2_0 (\dot{x}^2 - \dot{x}^2_0)/\omega^2 = 0
\end{align}
which is the phase-space tangent manifold $\sigma$ defined in Proposition~\ref{theorem:PSM}. $\qquad \qquad\qquad \qquad\qquad\qquad\qquad\qquad\qquad\qquad\qquad\qquad \qed$

\setcounter{algorithm}{2}
\begin{algorithm}[h!]
\begin{algorithmic}[1]
\setstretch{1.1}
\STATE Initialize walking step index $k \gets 1$, discrete state $q$, initial condition $\mathcal{I}_{q}$, $\epsilon$ for invariant bundle $\mathcal{B}_{q}(\epsilon)$, 
stage update indicator $b_{\rm update} \gets \FALSE$.
\WHILE{$\boldsymbol{x}_{q} \notin \mathcal{G}_a^{[\delta_j]}(q, q+1)$}
\IF{$\boldsymbol{x}_{q} \in \mathcal{G}_d^{[\delta_j]}(q, q_{\rm dist})$}
\STATE Execute $\Delta_{d(q \rightarrow q_{\rm dist})}^{[\delta_j]}(w_d)$ and quantize the disturbed state $\boldsymbol{x}_{q_{\rm dist}}$.
\STATE Generate optimal policies $(\zeta, \dot{x}, \ddot{x}, \tau, \omega, \mathcal{L}, \mathcal{V})_{\rm opt}$ by dynamic programming.
\STATE $b_{\rm update} \gets \TRUE$.
\STATE Compute the phase-space manifold $\sigma_{\rm trans}$ at transition instant by Eq.~(\ref{eq:simplifiedPSM}).
\IF{$\boldsymbol{x}_{q_{\rm trans}} \notin \mathcal{R}_{q}$}
\STATE Re-plan $x_{{\rm foot}_{q+1}}$ by Eq.~(\ref{eq:replanfoot}) and search $y_{{\rm foot}_{q+1}}$ by Algorithm~2.
\ENDIF
\ENDIF
\STATE Compute $\sigma_{i+1}$ over domain $\mathcal{D}_{q}$ by Eq.~(\ref{eq:simplifiedPSM}).
\IF{$b_{\rm update}$ is $\TRUE$}
\STATE Update stage index $i_{\rm stage}$ of recovery optimal control inputs.
\ENDIF
\IF{$\boldsymbol{x}_{q} \notin \mathcal{B}_{q}$}
\STATE Compute $\boldsymbol{u}_{c_{i+1}} =: (\tau_y, \omega)_{i_{\rm stage}}$ by Eq.~(\ref{eq:slidingcontrol_a}) and assign $\ddot{x}_{i+1} \gets \ddot{x}_{\rm opt}(i_{\rm stage})$.
\ELSE 
\STATE Compute $\boldsymbol{u}_{c_{i+1}} =: (\tau_y, \omega)_{i+1}$ by Eq.~(\ref{eq:slidingcontrol_b}) and assign $\ddot{x}_{i+1}$ by Eq.~(\ref{eq:accel}).
\ENDIF
\STATE Evolve $(x_{i+1}, \dot{x}_{i+1})$ over domain $\mathcal{D}_{q}$ by the analytical solution in Eq.~(\ref{eq:simplifiedPSM}).
\STATE $i \gets i + 1$.
\ENDWHILE
\STATE $q \gets q+1$, re-assign $\mathcal{I}_{q+1}$, $b_{\rm update} \gets \FALSE$ and jump to line 2 for next walking step.
\end{algorithmic}
\caption{Overall Robust Hybrid Locomotion Planning Structure}
\label{al:overall-planning}
\end{algorithm}

\section{Proof of Phase-Space Cotangent Manifold}
\label{sec:OrthogonalManifoldDerivation}
In this case, we use the tangent manifold in Eq.~(\ref{eq:simplifiedPSM}) to derive the cotangent manifold $\zeta$. By taking the derivative of Eq.~(\ref{eq:simplifiedPSM}), we have
\begin{align}\label{eq:d_sigma}
d \sigma & = \dfrac{\partial \sigma}{\partial x} dx 
           + \dfrac{\partial \sigma}{\partial \dot{x}} d \dot{x},
\end{align}           
where
\begin{align}
            \dfrac{\partial \sigma}{\partial x}
           = -2 \dot{x}_{\rm apex}^2 (x - x_{\rm foot}), 
           \quad
             \dfrac{\partial \sigma}{\partial \dot{x}}
           =  2 \dot{x}_{\rm apex}^2 \dot{x}/\omega^2.
\end{align}
The $\sigma$ manifold's normal vector is given by its gradient, 
 ${\bf e}_n=\big(-2 \dot{x}_{\rm apex}^2 (x-x_{\rm foot})\;,
                  \;2 \dot{x}_{\rm apex}^2\dot{x}/\omega^2 \big)^T$, 
and its tangent vector is orthogonal to ${\bf e}_n$, i.e.,  
 ${\bf e}_t=\big(2\dot{x}_{\rm apex}^2 \dot{x}/\omega^2\;,\; 
                 2\dot{x}_{\rm apex}^2 (x-x_{\rm foot})\big)^T$. 
 Since $\zeta$ is orthogonal to $\sigma$, the tangent vector of $\zeta$ is the normal vector of $\sigma$, i.e., 
\begin{align}\label{eq:d_zeta}
d \zeta & = \dfrac{\partial \zeta}{\partial x} dx 
          + \dfrac{\partial \zeta}{\partial \dot{x}} d \dot{x},
\end{align}
where
\begin{align}
            \dfrac{\partial \zeta}{\partial x} = 2 \dot{x}_{\rm apex}^2 \dot{x}/\omega^2, 
          \quad 
            \dfrac{\partial \zeta}{\partial \dot{x}}  = -2 \dot{x}_{\rm apex}^2 (x - x_{\rm foot})
\end{align}
Via the equations above, we can further obtain
\begin{align}\label{eq:dxdy}
\dfrac{d \dot{x}}{d x} = -\dfrac{\dot{x}}{\omega^2 (x - x_{\rm foot})} 
                 \qquad \Rightarrow \qquad 
                 \omega^2 \int_{\dot{x}_0}^{\dot{x}}\dfrac{d\dot{x}}{\dot{x}} 
                 = -\int_{x_0}^{x} \dfrac{dx}{x-x_{\rm foot}}
\end{align}
then we have
\begin{align}\label{eq:ln}
\ln(\dfrac{\dot{x}}{\dot{x}_0})^{\omega^2} + \ln \dfrac{x - x_{\rm foot}}{x_0 - x_{\rm foot}} = 0
\qquad \Rightarrow \qquad 
(\dfrac{\dot{x}}{\dot{x}_0})^{\omega^2} \dfrac{x - x_{\rm foot}}{x_0 - x_{\rm foot}} = 1
\end{align}
Thus, the cotangent manifold can be defined as 
\begin{align}\label{eq:cotangent-manifold}
\zeta  = \zeta_0(\dfrac{\dot{x}}{\dot{x}_0})^{\omega^2} \dfrac{x - x_{\rm foot}}{x_0 - x_{\rm foot}}
\end{align}
where the constant $\zeta_0$ is a nonnegative scaling factor, which is chosen as the phase progression value when contact switches occur. $(x_0, \dot{x}_0)$ is the initial condition at $\zeta=\zeta_0$. The equation above is the phase-space cotangent manifold $\zeta$ defined in Proposition~\ref{prop:PSCoM}. $\qquad\qquad\qquad\qquad\qquad\qquad\qquad\qquad\qquad\qquad\qquad\qquad\qquad\qquad\quad\qquad\qquad\qquad\qquad\qquad\qquad\qquad\qquad\qquad\qquad\qquad\qquad\qed$

\section{Dynamic Programming}
\label{sec:DP}
Dynamic programming divides a multi-period planning problem into simpler subproblems at different stages. In contrast to using a traditional time discretization, our study discretizes CoM sagittal position. Our objective is to generate recovery control policies offline for all admissible disturbance and store them as a policy table. Therefore, recovery can be achieved by looking up the table when disturbances are detected at runtime. We implement a grid-search backward DP. The cost function in Eq.~(\ref{eq:optimization-1}) can be defined as the value function $\mathcal{V}(q, \boldsymbol{x}_N)$. According to Bellman's equation, one step optimization at $n^{\rm th}$ stage is
\begin{align}\label{eq:onestepOptimization}
\mathcal{V}_n(q, \boldsymbol{x}_n) = \underset{\boldsymbol{u}_{\boldsymbol{x}}^c}{\text{min}}
\;\; & \mathcal{L}_n(q, \boldsymbol{x}_n, \boldsymbol{u}_{\boldsymbol{x}}^c) + \mathcal{V}_{n+1}(q, \boldsymbol{x}_{n+1}),
\end{align}
\noindent which is known as principle of optimality for discrete systems\footnote{Since the cost is computed iteratively in a backward way, the stage index $n$ decreases. In this value iteration algorithm, $\mathcal{V}_{n+1}(q, \boldsymbol{x}_{n+1})$ represents the total optimal cost from $(n+1)^{\rm th}$ stage to the terminal $N^{\rm th}$ stage for all feasible states $\dot{x}_{n+1}$. Accordingly, we solve the optimal control sequence from $(n+1)^{\rm th}$ to $N^{\rm th}$ stage for all $\dot{x}_{n+1}$. Then for $n^{\rm th}$ stage, we only need to solve the optimal cost from $n^{\rm th}$ to $(n+1)^{\rm th}$ stage.}.  Note that, the control input $\boldsymbol{u}_{\boldsymbol{x}}^c$ as well as CoM acceleration are assumed to be constant within one stage. For the integral of $\mathcal{L}_n(q, \boldsymbol{x}_n, \boldsymbol{u}_{\boldsymbol{x}}^c)$, linear interpolation is used to estimate CoM velocities. The velocities of two consecutive stages satisfy $\dot{x}_{n+1} = \dot{x}_{n} + \ddot{x}_{n} T_{n}$, where $T_{n}$ is the duration of one stage. Given the constant acceleration within one stage, we have
$\delta x_{n} = (\dot{x}_{n+1} + \dot{x}_{n})T_{n}/2$. 
Combining the two equations above, we can derive the constant acceleration

\begin{align}
\ddot{x}_{n} = \frac{\dot{x}_{n+1}^2-\dot{x}_{n}^2}{2 \delta x_{n}},
\end{align}
which is used to seed values to the equality constraint in the optimization problem of Eq.~(\ref{eq:optimization-1}), allowing us to solve for the continuous control input.

\begin{figure*}[!bht]
 \centering
   \includegraphics[width=0.95\linewidth]{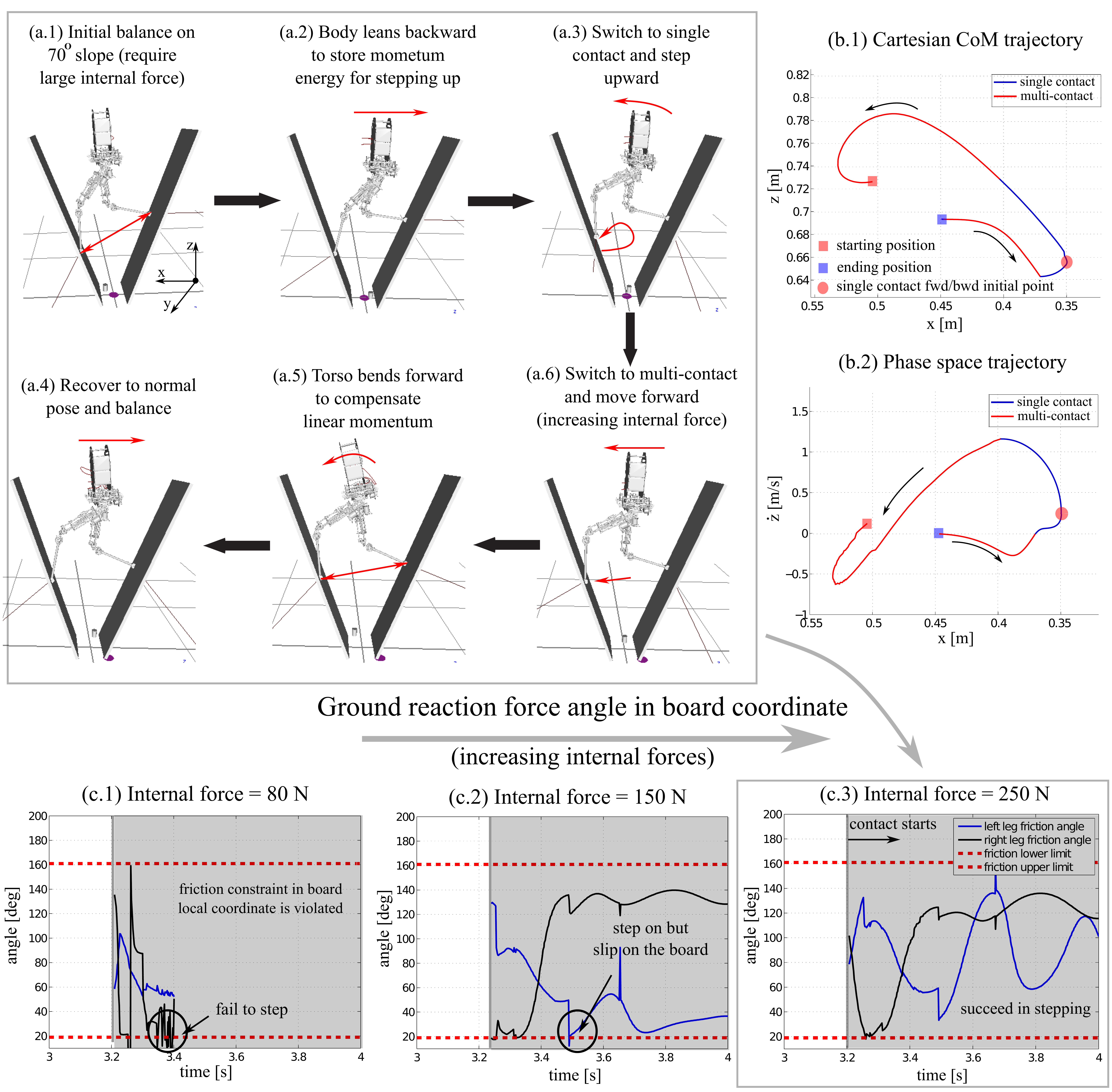}
 \caption{\captionsize Bouncing over a disjointed terrain. In this simulation, a biped balances on a steep disjointed terrain and dynamically bounces upwards. Internal force and torso pitch torque are controlled appropriately to achieve this motion. Subfigure (b) shows the 2D Cartesian CoM trajectory and the $x-\dot{z}$ phase portrait. (c) shows contact reactions for three different desired internal forces.}
 \label{fig:wedgejumping}
\end{figure*}

\section{Preliminary Results on Bouncing Over A Disjointed Terrain}
\label{sec:bouncingOverTerrain}
A more challenging locomotion scenario is preliminarily explored using a disjointed terrain. The slope of the surfaces is $70^\circ$. The goal is to step up over the surfaces by bouncing over the terrain. A physics based dynamic simulation called SrLib is used for validation and a whole body operational space controller [\cite{WBOSC15}] is implemented to follow the locomotion plans. The robot in the dynamic simulator has masses and inertias distributed across its body compared to the previous simulations. It possesses the same degrees of freedom, actuation joints and kinematic parameters as those in previous simulations. Another difference is that this simulation is planarized, meaning that the robot is not allowed to move laterally.
Snapshots of a one-step bouncing behavior are shown in Fig.~\ref{fig:wedgejumping} (a). To successfully bounce over the terrain, we design a CoM path manifold, shown in Fig.~\ref{fig:wedgejumping} (b.1), that mimics that of a pre-recorded human jumping motion [\cite{sentis2011humanoids}]. During the multi-contact phase, we apply a $250$ N internal tension force, shown in Fig.~\ref{fig:wedgejumping} (c), between the two surfaces to avoid sliding down due to the weight of the robot. The torso angular moment is also controlled immediately before and after the stepping-up motion. Our planner for this scenario operates in the $x\textnormal{--}\dot z$ phase-space as shown in Fig.~\ref{fig:wedgejumping} (b.2). This is more convenient as $\dot z$ captures the moment at which the center-of-mass starts falling down. More details on this strategy are discussed in [\cite{sentis2011humanoids}]. Note that the keyframe in this case also becomes defined as a state $(x, \dot{z})$, shown as a red circle in Fig.~\ref{fig:wedgejumping} (b.2). Even though the bouncing behavior on the disjointed terrain is intrinsically different from the previously-studied rough terrain walking, we still use the proposed single contact inverted pendulum model of Section~\ref{subsec:singlecontact} and the multi-contact dynamics of Section~\ref{subsec:dualContact}. The overall behavior is essentially different than the walking cases. The main reasons are that: (i) internal force control is needed to overcome gravity forces on the highly inclined surfaces and that (ii) the multi-contact contact phase is more dominant than the single contact phase.

\end{appendix_sec}

\bibliographystyle{chicago}
\bibliography{IJRR}

\end{document}